\documentclass{article}
\usepackage[utf8]{inputenc}
\usepackage{xcolor}
\usepackage{graphicx}
\usepackage{tikz}
\usepackage{hyperref}
\usepackage{csquotes}
\usepackage{amsfonts}
\usepackage[affil-it]{authblk}
\usepackage{amsmath}
\usepackage{amssymb}
\usepackage{enumitem}
\usepackage{placeins}
\usepackage{algorithm}
\usepackage{algpseudocode}
\usepackage[symbol]{footmisc} 
\usepackage{accents} 

\usepackage{amsthm}
\newtheorem{thm}{Theorem}
\newtheorem*{thm*}{Theorem}
\newtheorem{cor}[thm]{Corollary}

\newtheorem{rmk}{Remark}
\newtheorem{lem}[thm]{Lemma}
\newtheorem*{lem*}{Lemma}
\newtheorem{ex}{Example}

\usepackage{geometry}
\usepackage{mathtools}
\usepackage{multirow}
\usepackage{siunitx}
\usepackage{natbib}

\definecolor{myred}{HTML}{F4938C}
\definecolor{mygrey}{HTML}{E8E3E3}
\definecolor{mypurp}{HTML}{DE83E0}
\definecolor{myblue}{HTML}{a1cced}

\title{\textbf{Discovery and inference beyond linearity for epidemiological data by integrating Bayesian regression, tree ensembles and Shapley values}}
\author{Giorgio Spadaccini (Amsterdam UMC, Leiden University)\\
Marjolein Fokkema (Leiden University)\\
Mark A. van de Wiel (Amsterdam UMC)}
\affil{}
\date{}

\begin{document}
  \maketitle

\begin{abstract}
Machine Learning (ML) is gaining popularity in epidemiology and healthcare studies for hypothesis-free discovery of risk and protective factors. ML is strong at discovering nonlinearities and interactions, but this power is compromised by a lack of reliable inference. Although Shapley values provide local measures of features' effects, valid uncertainty quantification for these effects is typically lacking, thus precluding statistical inference. We propose RuleSHAP, a framework that addresses this limitation by combining a dedicated Bayesian sparse regression model with an improved tree-based rule generator and Shapley value attribution. RuleSHAP provides detection of nonlinear and interaction effects, with uncertainty quantification at the individual level as a key contribution. We derive an efficient formula for computing marginal Shapley values within this framework. We apply RuleSHAP to data from an epidemiological cohort to detect and infer several effects for high cholesterol and blood pressure, such as nonlinear interaction effects between features like age, sex, ethnicity, BMI and glucose level. To conclude, we demonstrate the validity of our framework on simulated data.
\end{abstract}

\noindent%
\vfill

\newpage
\setlength{\parskip}{12pt} 

\section{Introduction}
Risk and protective factors are at the center of many studies in healthcare. These studies aim to understand how policies, behaviors and treatments are associated to negative and positive health outcomes. In most cases, (generalized) linear regression is employed, because it provides readily available inference. This approach, however, is very rigid, as in many healthcare settings the effect of features on the outcome goes beyond linearity. Nonlinear and interaction terms may be included, but this requires prior knowledge or hypotheses about the functional form of features' effects. The expression \enquote{hypothesis-free discovery} is often used to refer to scientific investigation carried out through Machine Learning (ML) models in contexts where linearity is too restrictive  \citep{del2019machine,madakkatel2021combining}. Unlike parametric models such as linear regression, ML models allow to capture more complex patterns in the data, as they do not require any prior specification of the shape of the effects. This gives researchers an edge over the standard approach, as ML can capture trends that a linear model would miss.\par

However, combining reliable inference with interpretability, both of which are key for healthcare applications, remains a challenge for ML models. Feature importance measures such as ALE \citep{ALE}, LIME \citep{ribeiro2016should} and SHAP  \citep{shapley1953value,strumbelj2010efficient} have been developed to make black box models more transparent, but these measures are not suitable for inference, mainly due to a lack of uncertainty quantification \citep{molnar2020interpretable}. Some existing approaches aim to quantify the uncertainty of SHAP values \citep{watson2024explaining,slack2021reliable}, but they mostly focus on aleatoric uncertainty. Specifically, they quantify the uncertainty in the SHAP values of a fixed estimated model, thereby neglecting epistemic uncertainty, i.e. the uncertainty that stems from parameter instability.
While these feature importance measures may be combined with ML models that are naturally equipped with uncertainty quantification \cite[e.g., Bayesian Additive Regression Trees; ][ or Bayesian rule ensembles; \citeauthor{nalenz2018tree} \citeyear{nalenz2018tree}]{chipman2010bart}, we show that such inference is unreliable when used for local hypothesis testing. Careful semi-parametric modeling, e.g. through Generalized Additive Models \cite[GAMs; ][]{GAMlimits,hastie2017GAM} may also be used for this purpose but does not allow for hypothesis-free detection of interactions. Other researchers rely on a two-step procedure in which the features are filtered with an ML model and then tested with linear regression \citep{madakkatel2021combining,liu2023combining,madakkatel2023hypothesis}. However, this approach can yield overly optimistic uncertainty quantifications unless one uses two batches of data (see  Example~\ref{ex:TwoStepbad} in the Supplementary Material). Furthermore, its inference is limited to linear main effects, rendering it potentially inconsistent with the ML model.\par

We aim to obtain both flexibility and inference within one coherent framework and apply it to discover beyond-linear trends related to cardiovascular health. We propose RuleSHAP, a framework for using rule-based, hypothesis-free discovery that combines Bayesian regression, tree ensembles, and Shapley values to both detect and infer complex patterns. RuleSHAP critically improves over existing methods in three main ways. First, it generates rules in a manner that we argue to be better suited for interpretation and uncertainty quantification, compared to existing rule ensembles. For example, it specifically creates rules for beyond-linear effects. Second, it relies on the good coverage properties of Bayesian horseshoe regression \citep{van2017uncertainty} to obtain high-quality uncertainty quantification. Third, RuleSHAP efficiently produces marginal Shapley values \citep{shapley1953value,strumbelj2010efficient} with Bayesian uncertainty quantification, allowing for inference through credible intervals.\par

The work is structured as follows: in Section~\ref{sec:RuleFit}, we discuss the existing RuleFit family of models that RuleSHAP improves upon. In Section~\ref{sec:RuleSHAP}, we introduce RuleSHAP and derive its Shapley values.
In Section~\ref{sec:HELIUS} we apply our method to the data collected in the HELIUS epidemiological cohort study \citep{snijder2017cohort} to explore risk and protective factors for high cholesterol level and high systolic blood pressure, uncovering nonlinear interaction effects of covariates such as age, sex, ethnicity, BMI and glucose level. In Section~\ref{sec:EmpiricalEval} we use simulations to confirm that RuleSHAP correctly reconstructs the effects of signal features and flags noise features' local effects as not significant. In Section~\ref{sec:Discussion}, we summarize our findings and describe strengths, limitations and future directions.

\section{RuleFit family of models}
\label{sec:RuleFit}
In this section, we briefly summarize how the existing RuleFit models work. This includes the original RuleFit model \citep{friedman2008predictive,fokkema2017fitting} and its Bayesian variation HorseRule \citep{nalenz2018tree}.\par

A RuleFit model \cite[also known as Prediction Rule Ensemble; ][]{fokkema2017fitting}, is a linear regression model that is enhanced with dichotomous terms to account for nonlinearity and interactions. These dichotomous terms are defined as binary decision rules and are typically extracted from one or more tree ensembles, such as Random Forests, Gradient Boosting or XGboost. The resulting model takes the form:
\begin{equation}
    F(x) = \hat{a}_0 + \sum_{k=1}^q{\hat{a}_k r_k(x) + \sum_{j=1}^p{\hat{b}_j x_j}},
    \label{eq:RuleFitLinearEq}
\end{equation}
where $r_1,\ldots,r_q$ are the decision rules and $x_1,\ldots,x_p$ are the linear terms which may also undergo Winsorization to enforce stability\footnote{See \citep[Section 5]{friedman2008predictive}, we employ the standard of truncating to the 2.5-th and the 97.5-th percentiles.}. Examples of decision rules are $r_k(x)=I(x_2 < 0.7)$, which models a nonlinear trend, or $r_k(x)=I(x_1 > 0.5,\, x_3 > 0.8)$ which models an interaction. This enhanced linear model $F(x)$ is estimated by fitting a sparse linear regression (such as a LASSO fit) on the terms $x_1,\ldots,x_p,r_1,\ldots,r_q$. The underlying idea is that since all rules defining the tree ensemble are carried over to the linear model, RuleFit should retain (most of) the accuracy of the original tree ensemble that generated the rules. At the same time, RuleFit will simplify it, since the final model is still a sparse linear regression.\par

Heuristic measures for features' effects have been proposed that can easily be computed from the coefficients $\hat{a}_k$ and $\hat{b}_j$ \citep{friedman2008predictive}. However, these measures may be inaccurate when estimating compensating effects or interactions. We expand these considerations further with Example~\ref{ex:FriedImp} in the Supplementary Material.\par

HorseRule \citep{nalenz2018tree} is a Bayesian variation of RuleFit. Instead of the frequentist LASSO, it uses Bayesian regression with a Horseshoe prior \citep{carvalho2010HSselection} to estimate the model in Equation~\ref{eq:RuleFitLinearEq}. HorseRule also introduces a structured penalization: stronger shrinkage is applied to rules with higher depth (i.e. many features involved in its definition) or very extreme support (i.e., almost all or almost no observations follow the rule). More details are given in Subsection~\ref{subsec:HorseRulePrior} of the Supplementary Material. This prior favors the use of linear terms by construction, since it never shrinks linear terms more than a rule, but oftentimes less. Despite this definition, we will show in Section~\ref{sec:EmpiricalEval} that HorseRule still often fails to recognize linear trends.\par 

Even with the Bayesian uncertainty quantification provided by HorseRule, RuleFit's feature effect measure is not suitable for inference. This means that inference on the feature level, which is desirable since features are typically present in multiple terms, remains unattainable. More details are discussed in Subsection~\ref{subsec:NoInfer} of the Supplementary Material.

\section{RuleSHAP}
\label{sec:RuleSHAP}
The goal of RuleSHAP is to adapt the RuleFit family of models to an inferential framework. Currently, three critical issues limit that:
\begin{itemize}[label={-}]
    \item \textbf{Linear fit:} The linear fit used to estimate the coefficients in Equation~\ref{eq:RuleFitLinearEq} should also estimate their uncertainty well. For instance, HorseRule typically overshrinks linear effects (see Figure~\ref{fig:exBARTRS} on page \pageref{fig:exBARTRS} and Figures~\ref{fig:FriedCoeffsp10}-\ref{fig:LogiCoeffsp30} in the Supplementary Material). Our model, on the other hand, separates the shrinkage between rules and linear terms. This is crucial for inference on linear effects.
    \item \textbf{Rule generation:} The coverage of the coefficients in the linear fit is conditional on a fixed set of rules. However, the definition of the rules themselves is also random as it depends on the data at hand. Therefore, the uncertainty in the rule generation also needs to be quantified.
    \item \textbf{Feature effect measure:} The feature effect measures defined for RuleFit models do not provide reliable and interpretable metrics, even if satisfactory uncertainty quantification for the coefficients is available.
\end{itemize}
In this section we expand on these issues and describe how RuleSHAP handles each of them. The current paper is restricted to the setting of binary or continuous outcome, but RuleSHAP may be extended to any type of generalized linear regression where horseshoe prior regression can be implemented and working residuals can be computed.

\subsection{Linear fit}
In the linear fit, we aim to enforce model simplicity as much as possible while keeping the uncertainty quantification reliable. More specifically, we slightly re-define the horseshoe prior so that it shrinks linear terms less and thereby produces more accurate inference for linear effects. The advantages of using linear terms when they adequately approximate the feature-outcome relationship are threefold: first, it gives the model more simplicity. Second, linear terms are not defined conditionally on the data. This means that they are readily compatible with inference, unlike rule terms. Third, linear terms induce an effect across all observations. This allows the local effect estimates to better borrow information from each other, as they are all of the form $\hat{b}_j(x^{(i)}_j-\overline{x}_j)$ with the same $\hat{b}_j$ across all observations. This is in contrast with rules, which instead may target only a handful of observations.\par

When HorseRule is fitted and a horseshoe prior is used, linear terms tend to be shrunken excessively and are replaced by rules. As we show in Section \ref{sec:EmpiricalEval}, this happens even when the true data-generating feature effect is purely linear. We believe this to be caused by the definition of the horseshoe prior: local shrinkage parameters of each coefficient are combined with one overarching global shrinkage. When many rules are uninformative, they can substantially inflate the global shrinkage. Linear terms must then rely on their local shrinkage parameters to counterbalance this effect, and whenever this compensation fails, those terms become excessively shrunken.\par

To prevent this, we change the structured horseshoe prior used by HorseRule and give it a hierarchical shape, similarly to grouped LASSO regularization \citep{lim2015learning}. More specifically, we propose to use separate global shrinkages $\tau_R$ and $\tau_L$ for rules and linear terms respectively:
\begin{gather*}
 y|X_R,X_L,a,b,\sigma^2 \sim \mathcal{N}(X_R \cdot a+X_L \cdot b,\sigma^2I_n),\\
a_k|\lambda,\tau,\tau_R,\sigma^2 \sim \mathcal{N}(0,\lambda_k^2\tau_R^2\tau^2\sigma^2), \qquad \qquad \qquad 
b_j|\gamma,\tau,\tau_L,\sigma^2 \sim \mathcal{N}(0,\gamma_j^2\tau_L^2\tau^2\sigma^2),\\
\lambda_k \sim \mathcal{C}^+(0,A_k),\qquad \qquad \qquad
\gamma_j \sim \mathcal{C}^+(0,1),\,\,\,\\
\tau,\tau_L,\tau_R \sim \mathcal{C}^+(0,1),\qquad\qquad \qquad \sigma^2 \sim \sigma^{-2}d\sigma^2, \qquad \,\,
\end{gather*}
where $y \in \mathbb{R}^n$ are the observed outcomes, $X_L \in \mathbb{R}^{n \times p}$ is the design matrix whose $p$ columns are the linear terms and $X_R \in \mathbb{R}^{n\times q}$ is the design matrix whose $q$ columns are the rule terms appearing in Equation~\ref{eq:RuleFitLinearEq}. The notation $\mathcal{C}^+(0,A)$ denotes the half-Cauchy distribution of location 0 and scale $A$. The scalar $A_k$ is defined as:
$$A_k=\frac{(2\min(\overline{r}_k,1-\overline{r}_k))^{\mu-0.5}}{m_k^\eta\sqrt{2\max(\overline{r}_k,1-\overline{r}_k)}}$$
and is equivalent to the structured shrinkage developed for HorseRule \citep{nalenz2018tree}, rescaled after accounting for standardization. Remark \ref{rmk:Rescaling} in the Supplementary Material clarifies the definition of $A_k$. The hyperparameters $\mu,\eta >0$ determine the level of shrinkage applied to rules with extreme support and high depth, respectively.\par

This prior differs from HorseRule's horseshoe prior, which only relies on $\tau$ instead of $\tau,\tau_L,\tau_R$. The use of rule- and linear-specific global shrinkages means that uninformative rules only inflate $\tau_R$ instead of $\tau$, preventing the shrinkage of linear terms from being affected.\par

To estimate the model, we adapted the Gibbs sampling scheme for plain Horseshoe regression described by \cite{makalic2015simple} and implemented in the work of \cite{horseshoeR,horseshoenlmR,Villani2017HorseRuleR}. More details are provided in the Supplementary Material.

\subsection{Rule generation}
RuleSHAP also alters rule generation, both to accommodate inference and to ensure that rules only model beyond-linear trends.\par

First, we focus on the inference-hindering problem that rule generation is also random and depends on the dataset at hand. This problem is worsened by the fact that the dataset used for the rule generation is the same as for the linear fit. More specifically, in a Random Forest, each tree uses a bootstrapped sample which, up to multiple copies, contains around $63\%$ of the total number of observations. This means that each rule already fits roughly two thirds of the whole data very well, and these observations are used to determine the coefficients of the model and their uncertainty. This might not only lead to overfitting, but also to overly optimistic uncertainty quantification: through the lens of the linear model, the rules \enquote{happen} to fit the data very well, when in fact they were constructed by design to do so. As shown in Section~\ref{sec:EmpiricalEval}, such underestimated uncertainty inflates the significance rates of the noise features' effects and should therefore be avoided.\par

To prevent this, we propose to robustify the case-resampling bootstrapping that defines the Random Forest by augmenting it with a parametric bootstrap. More specifically, we apply the following three-step strategy to generate the rules:

\begin{itemize}[label={}]
    \item \textbf{Step 1:} Fit a Random Forest on the data, using $y_1,\ldots,y_n$ as outcome.
    \item \textbf{Step 2:} Given any probability distribution $\mathcal{F}(\theta)$ with parameters $\theta$, assume:
    $$y_i \sim \mathcal{F}(\theta_i) \qquad \forall i=1,\ldots,n,$$
    and use the Random Forest's out-of-bag prediction to estimate the parameters of the assumed distributions:
    $$y_i \sim \mathcal{F}(\hat{\theta}_i) \qquad \forall i=1,\ldots,n.$$
    For example, assuming homoscedastic, normally distributed outcomes, the distributions and estimated parameters would be:
    $$y_i \sim \mathcal{N}(\hat{\mu}_i,\hat{\sigma}^2), \qquad \hat{\mu}_i = \hat{y}_i, \qquad \hat{\sigma}^2=\frac{1}{n}\sum_{i=1}^n(y_i-\widehat{y}_i)^2 \qquad \qquad \forall i=1,\ldots,n,$$
    where $\hat{y}_1,\ldots,\hat{y}_n$ are the out-of-bag predictions of the Random Forest fitted in Step 1.
    \item \textbf{Step 3:} Generate the rules by fitting a Random Forest which replaces the observed outcomes $y_1,\ldots,y_n$ with synthetic copies sampled from the estimated distributions $\mathcal{F}(\hat{\theta}_1),\ldots,\mathcal{F}(\hat{\theta}_n)$.
    So, if the bootstrap sample that is used to fit the $k$-th tree contains observations indexed as $i_{k,1},\ldots,i_{k,n}$, the synthetic outcomes to use for this bootstrap sample are not $y_{i_{k,1}},\ldots,y_{i_{k,n}}$, but instead:
    $$\accentset{\sim}{y}^{(k)}_1 \sim \mathcal{F}(\hat{\theta}_{i_{k,1}}),\ldots,\accentset{\sim}{y}^{(k)}_n \sim \mathcal{F}(\hat{\theta}_{i_{k,n}}).$$
For example, assuming homoscedastic, normally distributed outcomes, the outcomes used for the $k$-th bootstrap sample are drawn as:
    $$\accentset{\sim}{y}^{(k)}_1 \sim \mathcal{N}(\hat{y}_{i_{k,1}},\hat{\sigma}^2),\ldots,\accentset{\sim}{y}^{(k)}_n \sim \mathcal{N}(\hat{y}_{i_{k,n}},\hat{\sigma}^2).$$
\end{itemize}

The algorithm describes a general, distribution-agnostic mechanism. Because the coefficients in the rule ensemble model are estimated with a Bayesian linear fit, which assumes homoscedastic, normally-distributed outcomes, we invoke the same assumption during rule generation.\par

Our algorithm re-samples the synthetic outcome every time it is used, thus introducing uncertainty in the rule generation. Concretely, this injection of noise regularizes the rule generation, thereby increasing its stability and making it less prone to overfitting (see for instance \cite{bishop1995training}). On the other hand, predecessor methods developed by \cite{friedman2008predictive, fokkema2017fitting, nalenz2018tree}, repeatedly use the original outcome for generating rules as well as estimating the model in Equation \ref{eq:RuleFitLinearEq}. A comparison is shown in Figure~\ref{fig:RFvPRF} in the Supplementary Material and confirms a much more erratic behavior in the case where rules are generated from a non-modified Random Forest. In this example, the non-modified Random Forest induces false significance of the effect of some values of the feature $x_9$. Our approach, on the other hand, results in much smoother estimated feature effects and prevents such false significance. We therefore refer to our approach as \enquote{Smoothing Random Forest}. Although combinations of tree ensembling methods can be employed for rule generation as described by \cite{nalenz2018tree}, our experiments in Section~\ref{sec:EmpiricalEval} indicate that sole use of our Smoothing Random Forest lead to better distinction between noise and signal features.\par

A further measure RuleSHAP takes to encourage interpretability and counter the overfitting of rules is to fully disaggregate the rules. That is, the conditions that appear in a tree path are recombined in all possible ways. As an example, consider the tree in Figure~\ref{fig:TreeEx}. The rules that are generated from the path leading to the third terminal node $I(x_1 \geq 3,x_3<2,x_2\geq 4)$ are:
\begin{gather*}
    I(x_1 \geq 3), \qquad
    I(x_3 < 2), \qquad
    I(x_2 \geq 4), \\
    I(x_1 \geq 3, x_3 < 2), \qquad
    I(x_3 < 2, x_2 \geq 4), \qquad
    I(x_1 \geq 3, x_2 \geq 4), \\
    I(x_1 \geq 3,x_3<2,x_2\geq 4).
\end{gather*}
\begin{figure}
    \centering
    \includegraphics[width=0.6\linewidth]{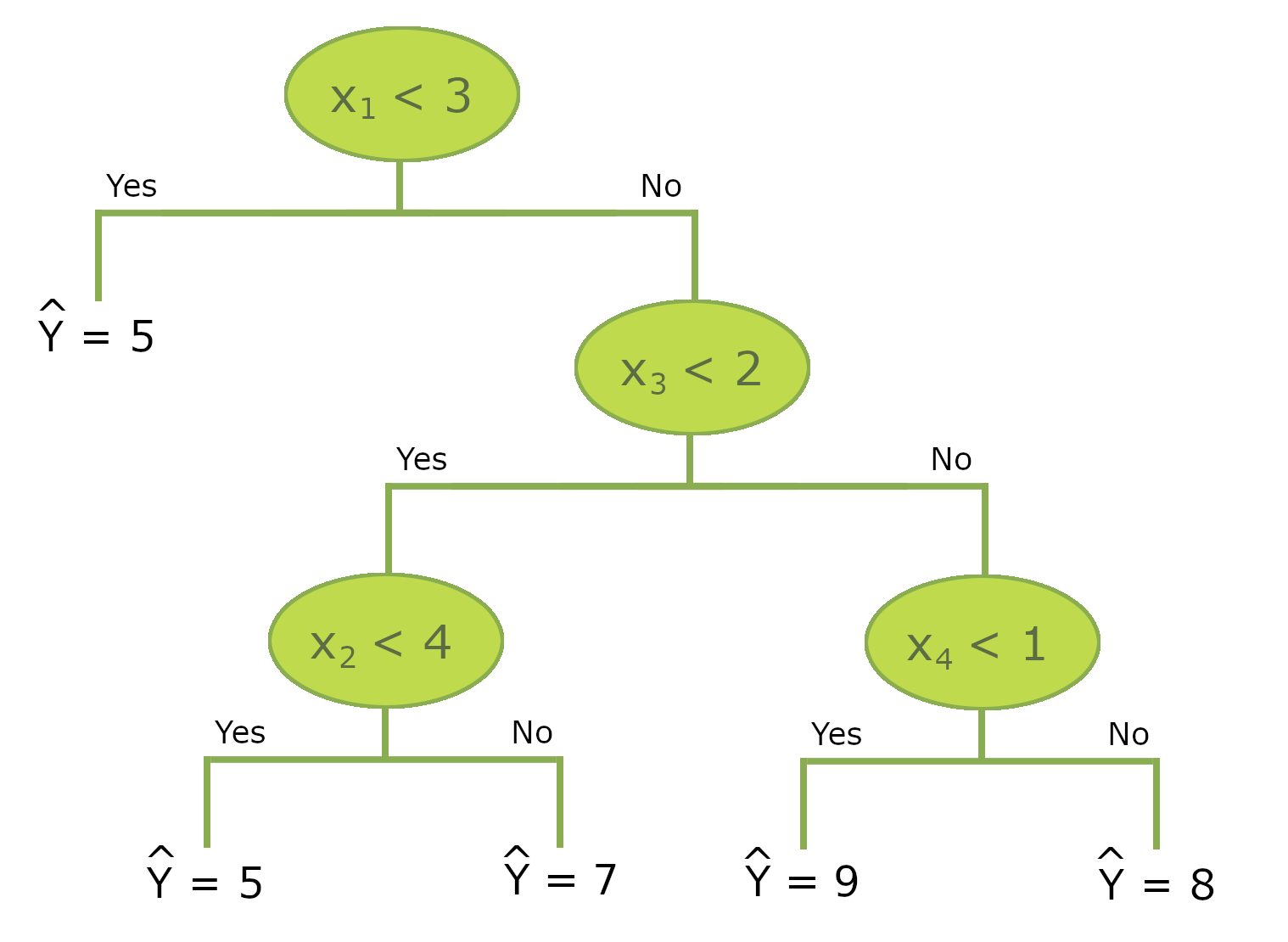}
    \caption{Example of one of the many trees generated from a tree ensemble.}
    \label{fig:TreeEx}
\end{figure}
This expands the set of conditions produced by RuleFit for this specific tree path, which would only be:
\begin{equation*}
I(x_1 \geq 3), \qquad    I(x_1 \geq 3, x_3 < 2), \qquad    I(x_1 \geq 3,x_3<2,x_2\geq 4).
\end{equation*}
By combining this disaggregation with a prior that discourages deeper rules, interactions are freely generated but will only be used if actually necessary to the model. In this example, disaggregation guarantees that the model can rely on e.g. condition $I(x_3 < 2)$. Without disaggregation, this condition would not be available to the model, which therefore might be forced to use the unnecesarily more complex condition $I(x_1 \geq 3,x_3 < 2)$.\par

Finally, to further prevent rules from (partially) modeling linear trends, we fitted the random forest on residuals from a linear model. This approach has previously been applied for fitting single trees \citep{chaudhuri1995generalized, loh2002regression, dusseldorp2010combining}. In line with these works, RuleSHAP does not use the observed outcome to generate rules, but instead uses the residuals from a Bayesian Horseshoe regression on the linear terms. For logistic regression, where residuals are not as straightforwardly defined as in the linear case, we used working residuals \citep{hardin2007generalized}. This is because these estimated residuals are on the scale of the features, which is the same scale that the rules are added to.

\subsection{Feature effect measures}
Using the RuleFit family of models for inference requires a suitable metric for features' effects. As further discussed in the Supplementary Material, RuleFit's importance measures as defined by \cite{friedman2008predictive} do not allow for testing as they induce a boundary testing problem. We therefore propose the use of marginal Shapley values to retrieve the local effect of each feature in the model.\par

Shapley values may be viewed as a weighted average of centered partial dependence functions. More specifically, the marginal Shapley value of the $j$-th feature for a datapoint $x^*$ is computed as:
\begin{equation}
    \label{eq:shapleyformula}
    \phi_j(x^*)=\sum_{S \subseteq \{1,\ldots,p\}\setminus\{j\}}\frac{1}{p\binom{p-1}{|S|}}\Big(\mathbb{E}[F(x)|\text{do}(x_S=x_S^*,x_j=x_j^*)]-\mathbb{E}[F(x)|\text{do}(x_S=x_S^*)]\Big),
\end{equation}
where the notation $x_S=x_S^*$ means $x_k=x_k^* \, \forall k \in S$ and $\mathbb{E}[F(x)|\text{do}(\cdots)]$ is the interventional expectation. In this formula, the weights $\frac{1}{p\binom{p-1}{|S|}}$ take into account the fact that there are fewer subsets $S$ where $|S|$ is very small or very large. Marginal Shapley values have also been proposed as a way to quantify the predictive effect of interactions \citep{lundberg2018consistent}. Their definition is analogous to that of Equation~\ref{eq:shapleyformula}, but contrasts are replaced with ANOVA-like interaction effects:
\begin{align}
\begin{split}
    \label{eq:shapleyformulaInt}
    \phi_{j,k}(x^*)&=\sum_{S \subseteq \{1,\ldots,p\}\setminus\{j,k\}}\frac{1}{(p-1)\binom{p-2}{|S|}}\Big(\mathbb{E}[F(x)|\text{do}(x_S=x_S^*,x_j=x_j^*,x_k=x_k^*)]\\
    &-\mathbb{E}[F(x)|\text{do}(x_S=x_S^*,x_j=x_j^*)]-\mathbb{E}[F(x)|\text{do}(x_S=x_S^*,x_k=x_k^*)]\\
    &+\mathbb{E}[F(x)|\text{do}(x_S=x_S^*)]\Big).
\end{split}
\end{align}
What is left after subtracting all the interaction Shapley values may then be viewed as main effects.\par

The additivity property of Shapley values, combined with the additive nature of RuleFit, allows us to focus on Shapley values of individual terms. The overall model effects are then obtained by summing up the effects across all terms. This process may be repeated for every posterior sample of the model to estimate the full distribution of Shapley values and perform inference.\par

Aside from the many model-agnostic approaches that approximate Shapley values \citep{chen2022algorithms}, there are also tree-specific computations for the models described in Section~\ref{sec:RuleFit}. More specifically, upon viewing an individual rule as a specific 0-1-valued tree, TreeSHAP \citep{treeSHAPpackage,lundberg2020local} may be used to compute the Shapley values of a dichotomous rule. In the next section, we introduce an explicit formula for computing them directly.

\subsection{Direct computation of Shapley values}
\label{subsec:DirectShapleys}
In this section we produce an explicit formula for the computation of marginal Shapley values of a single rule $r(x)$. The formula estimates each expectation $\mathbb{E}[r(x)|\text{do}(x_S=x_S^*)]$ with its respective sample mean and is thus an \enquote{exact} computation in this sense. This formula may be used to efficiently compute marginal Shapley values for any of the RuleFit-based models.

\begin{thm}
\label{thm:ourFormula}
Assume to have a dataset $\mathcal{T}$ of size $n$. Consider a 0-1 coded rule decomposed as the product of single conditions and thus of the form $r(x_1,\ldots,x_p)=\prod_{k=1}^pR_k(x_k)$, with ${R_k:\mathbb{R}\rightarrow\{0,1\}}$. Given $j \in \{1,\ldots,p\}$ and a datapoint $x^*$, the contribution of the $j$-th feature to the prediction $\hat{a} \cdot r(x^*)$ as defined by marginal Shapley values is unbiasedly estimated by:
$$\widehat{\phi}_j(x^*)=\hat{a} \cdot \Bigg(\frac{1}{n(p-q(x^*)+R_j(x^*_j))}\sum_{\substack{t \in \mathcal{T} \text{s.t.}\\ R_k(t_k)=1 \vee R_k(x^*_k)=1 \,\forall k}}\frac{R_j(x_j^*)-R_j(t_j)}{\binom{2p-q(x^*)-q(t)-1+R_j(x^*_j)+R_j(t_j)}{p-q(x^*)+R_j(x^*_j)}}\Bigg),$$
where $q: \mathbb{R}^p \rightarrow \mathbb{N}$ is defined as $q(x)=\sum_{k=1}^pR_k(x_k)$ and $\vee$ is the logical \enquote{or} operator.
\end{thm}
\begin{proof}
See Corollary~\ref{cor:ourFormulaCor} in the Supplementary Material.
\end{proof}

This formula suggests computing Shapley values as sample means, for instance as described in Algorithm~\ref{alg:ourShapley} in the Supplementary Material\footnote{In the Supplementary Material, the results are shown for a broader set of tree-based learners. For instance, it may be used to also estimate (interaction) Shapley values for (terms extracted from) GLM trees \citep{zeileis2008model} and PILOT trees \citep{raymaekers2024fast}. This presents an alternative, direct computation to the work of \cite{zern2023interventional}, where the TreeSHAP algorithm is adjusted to accommodate GLM trees.}. For rules that involve only some of the features, one may may disregard all features not involved in the rule (see Lemma~\ref{lem:uselessfeatures} in the Supplementary Material). Algorithm~\ref{alg:ourShapley} may also be used to compute marginal Shapley values of a regular decision tree, since a decision tree can be decomposed as a linear combination of the rule terms at its terminal nodes. In this case, the computational cost is the same as path-independent TreeSHAP \citep{lundberg2020local}, which also coincides with Algorithm~\ref{alg:ourShapley} in terms of estimates. However, Algorithm~\ref{alg:ourShapley} allows us to look at individual rules rather than whole trees for the same computational cost. The Shapley values for an individual rule involving $D$ features can only be computed with TreeSHAP by converting the rule back into a full 0-1-valued tree of depth $D$, with $D+1$ terminal nodes. Because of this, TreeSHAP is $D+1$ times slower than Algorithm 1 when the purpose is to obtain the Shapley values of each rule separately, as in our setting. \footnote{When $D$ is capped at a low value, the difference in efficiency may be small. Note that the relative computational advantage of our algorithm is larger in contexts where trees are very deep but only use few covariates per tree, since in these settings the inner \textit{for} loop in Algorithm~\ref{alg:ourShapley} only needs to be run for the features involved in the rule. Examples of such setting are Planted Trees and Planted Tree Forests \citep{hiabu2020random}.}.\par

Similar arguments support the computation of interaction Shapley values (see also Algorithm~\ref{alg:ourShapleyInt} in the Supplementary Material):
\begin{thm}
\label{thm:ourFormulaInter}
Assume a dataset $\mathcal{T}$ of size $n$. Consider a 0-1 coded rule decomposed as the product of single conditions and thus of the form $r(x_1,\ldots,x_p)=\prod_{k=1}^pR_k(x_k)$, with $R_k:\mathbb{R} \rightarrow \{0,1\}$. Given two different indices $j,j' \in \{1,\ldots,p\}$ and a datapoint $x^*$, the interaction of the $j$-th and the $j'$-th features within the prediction $\hat{a}\cdot r(x^*)$ as defined by marginal interaction Shapley values is unbiasedly estimated by:
\begin{align*}
\widehat{\phi}_{j,j'}(x^*) \!\begin{multlined}[t]
    =\frac{\hat{a}}{n(p-1-q(x^*)+R_j(x^*_j)+R_{j'}(x^*_{j'}))}\\
    \cdot \sum_{\substack{t \in \mathcal{T} \text{s.t.}\\ R_k(t_k)=1 \vee R_k(x^*_k)=1 \,\forall k}}\frac{R_{j'}(x_{j'}^*)R_j(x_j^*)-R_{j'}(x_{j'}^*)R_j(t_j)-R_{j'}(t_{j'})R_j(x_j^*)+R_{j'}(t_{j'})R_j(t_j)}{\binom{2p-q(x^*)-q(t)+R_{j'}(x^*_{j'})+R_j(x^*_j)+R_{j'}(t_{j'})+R_j(t_j)-3}{p-q(x^*)+R_{j'}(x^*_{j'})+R_j(x^*_j)}},
    \end{multlined}
\end{align*}
where $q: \mathbb{R}^p \rightarrow \mathbb{N}$ is defined as $q(x)=\sum_{k=1}^pR_k(x_k)$ and $\vee$ is the logical \enquote{or} operator.
\end{thm}
\begin{proof}
See Corollary~\ref{cor:ourFormulaIntCor} in the Supplementary Material.
\end{proof}

Theorems~\ref{thm:ourFormula}~and~\ref{thm:ourFormulaInter} write Shapley values as sample means over the dataset, which allows us to asses the instability of the Shapley value estimation itself if sample sizes are very small. However, since this instability vanishes with the sample sizes required for a RuleSHAP model, we consider this instability to be negligible for our setting. Furthermore, this formulation can be convenient for contexts where re-weighting of the observations may be needed, such as with stratified sampling.\par

\section{Risk and protective factors in Cardiovascular Health}
\label{sec:HELIUS}
We fit RuleSHAP to data from the large-scale epidemiological HELIUS study \citep{snijder2017cohort}, which collects information about physical and mental health across different ethnic groups in the Netherlands. The very large sample size of the HELIUS study ($n = 21'570$) allows us to further investigate the quality of our model for different sample sizes. Following \citep{van2024linked}, we focus on two different outcomes: systolic blood pressure and cholesterol level, both standardized to have unit variance. These are predicted from BMI, socio-demographic features (age, ethnicity, gender) and habits-related features (smoker Y/N, yearly cigarettes consumption, coffee drinker Y/N). As a form of negative control, four standard normally distributed noise features are added to the datasets.\par

In order to fit the RuleSHAP model, categorical features are dummy-coded. In the case of ethnicity, which presents five categories, we follow the same logic as contrast coding and code each dummy variable with the values -0.5,2. By doing so, we ensure that the dummy variables are standardized in the setting where the five categories are balanced. First, we assess the validity of our model across different settings: for each of the two outcomes, we fit different models (Ordinary Least Squares, LASSO regression, RuleSHAP, RuleFit, HorseRule, Random Forest) on subsamples of different sizes ($n=100,300,500,1'000,3'000,5'000$) and record their test predictive performance. More specifically, 15'000 observations are retained for training, and the remaining 6'570 observations are used for testing. For sample sizes $n=100,300,500,1'000,3'000$, five non-overlapping datasets are drawn from the 15'000 training observations, without replacement. For the sample size of $n=5'000$, where 5 non-overlapping datasets would require 25'000 training observations, we allow the datasets to overlap with each other, while still ensuring that no observation appears multiple times in the same dataset. To conclude, we perform a single fit on $n=10'000$ of the training observations and discuss its significant estimated effects.\par

Predictive performance is presented in Figure~\ref{fig:HeliusMSEs} for a representative subset of sample sizes and is fully shown in Figures~\ref{fig:HeliusMSEchol}~and~\ref{fig:HeliusMSEsbp} of the Supplementary Material. For the prediction of systolic blood pressure, which is already quite accurately described by a linear model, RuleSHAP performs similarly to a linear fit as it relies mostly on linear terms. For the prediction of cholesterol levels, on the other hand, a simple linear fit is too rigid, and this is reflected in a stronger predictive discrepancy between OLS and RuleSHAP.\par

\begin{figure}[h!]
    \centering
    \includegraphics[width=\linewidth]{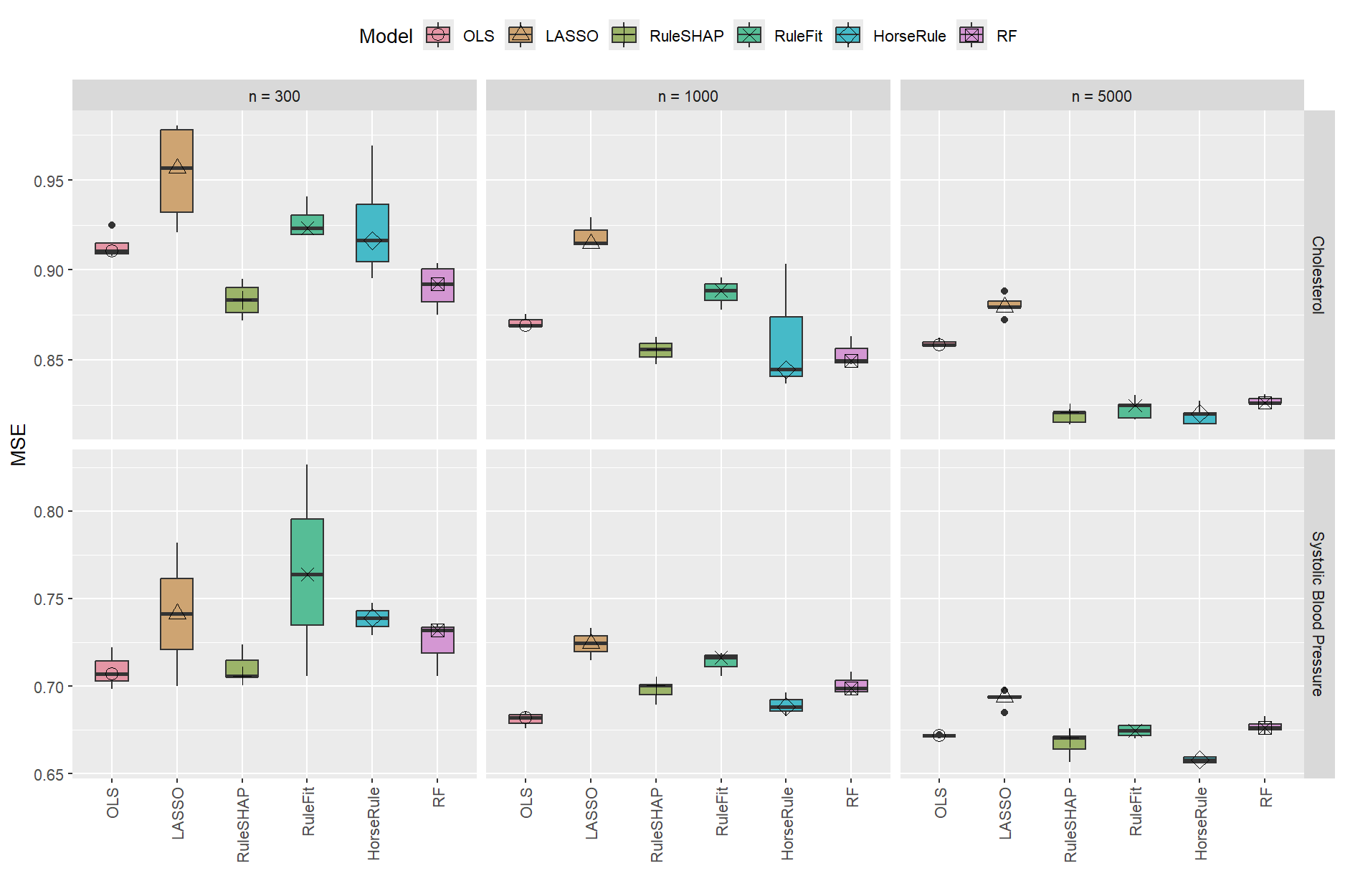}
    \caption{Test MSE across five fits of OLS linear regression, LASSO regression, RuleFit regression, HorseRule regression, RuleSHAP regression and Random Forest. These models are fitted on subsamples of different sizes $n$ taken from the HELIUS study data, with predicted outcome of cholesterol (top row) and systolic blood pressure (bottom row). Both outcomes have unit variance.}
    \label{fig:HeliusMSEs}
\end{figure}

Figures~\ref{fig:ShapleysSBP}~and~\ref{fig:ShapleysChol} show the Shapley values from the RuleSHAP fit on $n=10'000$ observations, for both outcomes. For both outcomes, all noise features obtained near zero, non-significant Shapley values. Significant marginal effects are observed for all features except being a smoker or coffee drinker. We see heavy wiggliness in the marginal effects of some features. This does not denote instability, but rather the presence of meaningful interactions. For example, the effect of a fixed value of age strongly oscillates across different ethnicities. The heatmaps in Figures~\ref{fig:HeatmapSBP}~and~\ref{fig:HeatmapChol} in the Supplementary Material use interaction Shapley values to produce a quick, global insight into the significant interaction trends estimated by RuleSHAP for both outcomes. The figures show interaction effects between age and other features (ethnicity and sex for prediction of systolic blood pressure; glucose, BMI, ethnicity and sex for prediction of cholesterol). Using marginal Shapley values, these can be further explored. Figure~\ref{fig:SexAgeInter}, for instance, specifically focuses on the nonlinear interaction between age and sex that is detected when predicting cholesterol level, showing that the difference between sexes in the effect of age on cholesterol is largest above 52 years of age. Figure~\ref{fig:EthnAgeInter} in the Supplementary Material explores the interaction between age and ethnicity for the prediction of systolic blood pressure.

\begin{figure}[h!]
    \centering
    \includegraphics[width=\linewidth]{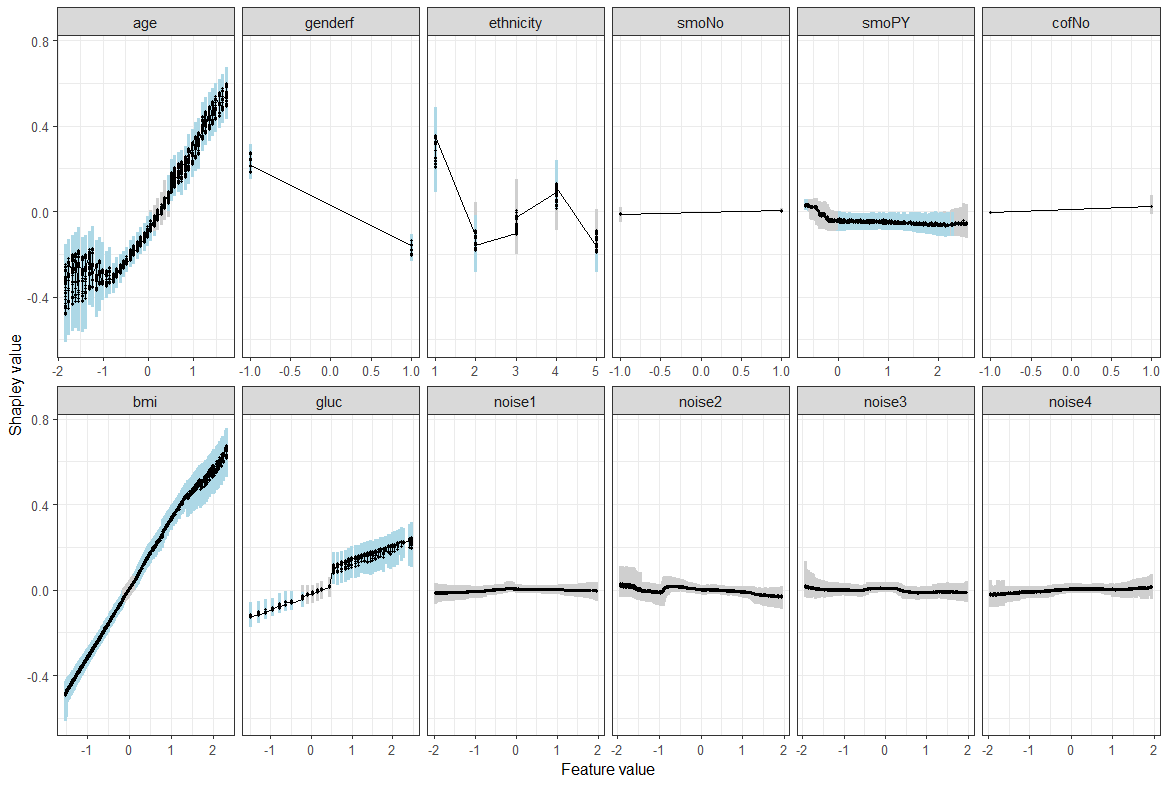}
    \caption{Shapley values computed from a RuleSHAP model fitted on $n=10'000$ observations from the HELIUS study, to predict systolic blood pressure. Point-wise 95\% credible intervals are shown and color-coded based on whether they contain zero (grey) or not (light blue).}
    \label{fig:ShapleysSBP}
\end{figure}

\begin{figure}[h!]
    \centering
    \includegraphics[width=\linewidth]{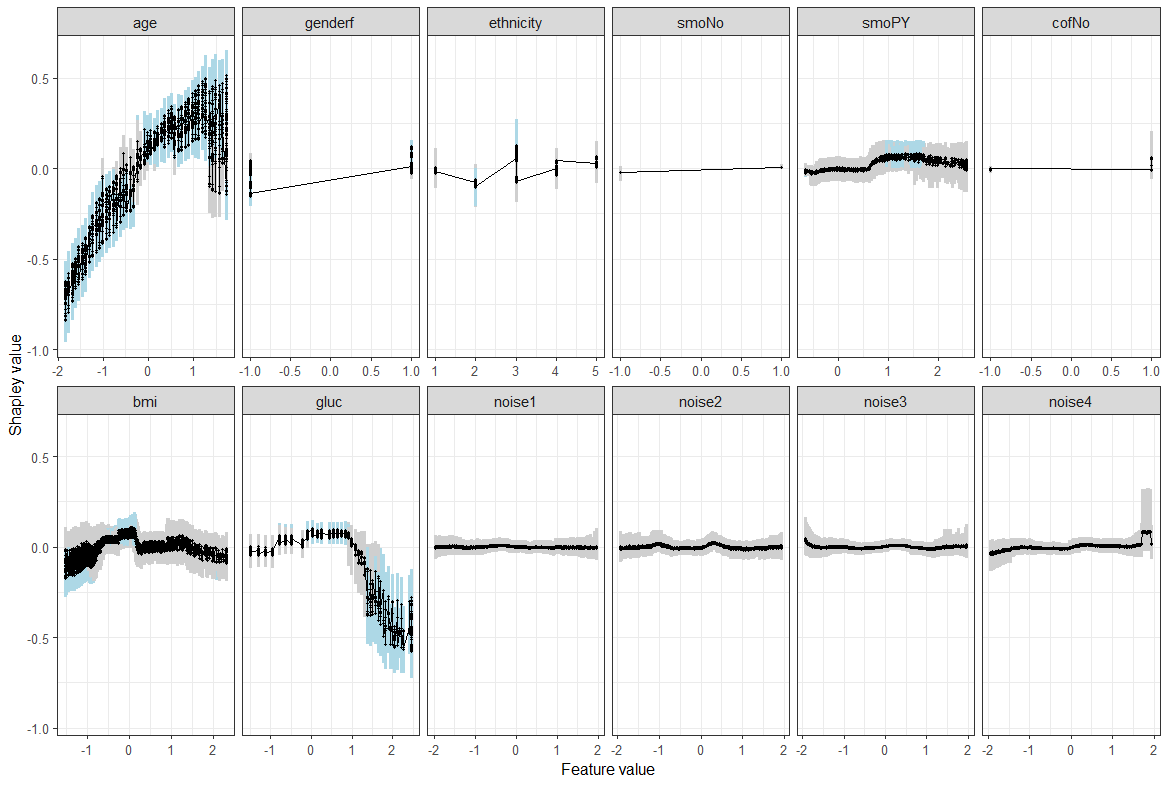}
    \caption{Shapley values computed from a RuleSHAP model fitted on $n=10'000$ observations from the HELIUS study, to predict cholesterol. Point-wise 95\% credible intervals are shown and color-coded based on whether they contain zero (grey) or not (light blue).}
    \label{fig:ShapleysChol}
\end{figure}

\begin{figure}[h!]
    \centering
    \includegraphics[width=\linewidth]{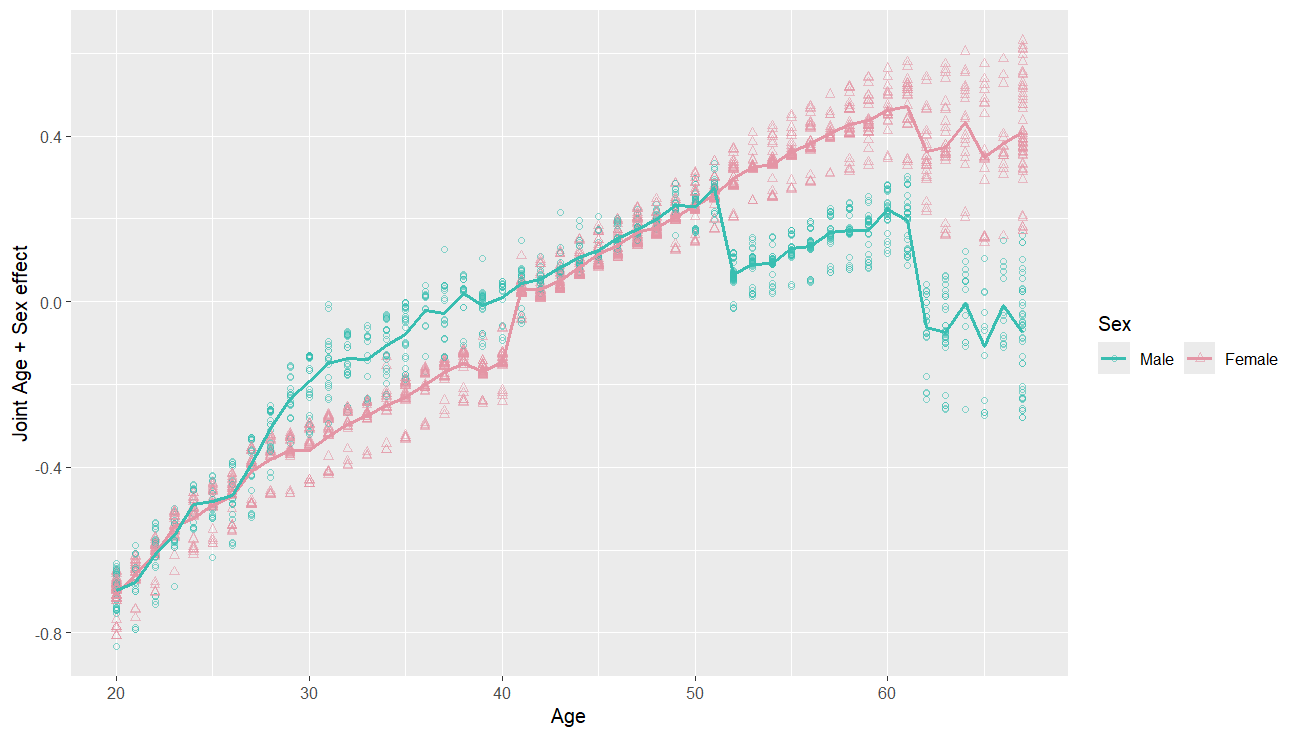}
    \caption{RuleSHAP's marginal Shapley values fitted on $n=10'000$ observations from the HELIUS study, to predict cholesterol, stratified by sex. The marginal Shapley values of sex and age are summed up together to visualize their joint effect. Lines show the overall trend of the mean joint contribution. The individual joint contributions, which oscillate between observations due to the other interaction effects of age, are represented by dots and triangles.}
    \label{fig:SexAgeInter}
\end{figure}

\section{Empirical evaluation}
\label{sec:EmpiricalEval}
In this section we describe the experiments that we conduct on simulated datasets to confirm the validity of the results observed in the previous section. Our aims are twofold: first, we compare RuleSHAP to the analogous approaches RuleFit, HorseRule and Random Forest. This part focuses on signal reconstruction and shows the limitations motivating us to develop RuleSHAP. Second, we compare RuleSHAP with BART and HorseRule, since these may also be combined with Shapley values to provide uncertainty estimates for local importance measures. This part focuses on the task of distinguishing signal from noise, as this is a typical use of uncertainty quantification in research.

\subsection{Evaluation of feature effect reconstruction}
\subsubsection{Method}
We first set up a simulation study to compare RuleSHAP with other methods typically used for prediction and interpretation purposes: OLS regression, Lasso regression \cite[from R package \textit{glmnet};][]{hastie2021introduction}, RuleFit \cite[from R package \textit{pre};][]{fokkema2017fitting}, HorseRule \cite[from R package \textit{horserule};][]{Villani2017HorseRuleR} and Random Forest \cite[from R package \textit{randomForest};][]{RFpackage}. The comparison is carried out in terms of predictive performance, recovery of individual effects and ability to prioritize linear terms over (many) rules. The methods are tested on the following benchmark data-generating procedure, initially proposed by \cite{friedman1991multivariate}:
\begin{align}
    \begin{aligned}
    \label{eq:FriedmanGenFunct}
    x_1,\ldots,x_{5} &\sim \mathcal{U}(0,1),\\
    \epsilon &\sim \mathcal{N}(0,\sigma^2),\\
    y&=10\sin(\pi x_1x_2)+20(x_3-0.5)^2+10x_4+5x_5+\epsilon.
\end{aligned}
\end{align}
The feature set is extended with noise features also drawn uniformly at random between 0 and 1. A correlation of 0.3 is ensured across all features. By default, the error variance in the Friedman procedure is set to $\sigma^2=1$, which produces an irreducible error of less than 5\% of the total variance. We instead choose $\sigma^2=100$, which produces about 80\% of irreducible error, for a more realistic setting.\par

Similarly to the experiments conducted in Section~\ref{sec:HELIUS}, we vary sample size $n=100,300,500,1'000,3'000,5'000$ and dimensionality $p=10,30$. For each combination of $n$ and $p$, the fitting is repeated on five simulated training datasets and predictive performance is assessed on a test dataset of size 10'000, common across all replicates. The experiment is analogously repeated for the case of logistic regression, where the outcome is defined by:
\begin{align}
\begin{aligned}
    \eta&=10\sin(\pi x_1x_2)+20(x_3-0.5)^2+10x_4+5x_5,\\
p&=\frac{\exp(\eta-\overline{\eta})}{\exp(\eta-\overline{\eta})+1},\\
y &\sim \mathcal{B}\textbf{e}(p),
\end{aligned}
\end{align}
with $\overline{\eta} \approx 14.4$ computed as the approximate mean of $\eta$ under this data-generation scheme, in the case of independent features.\par

A Bayesian regression typically requires a check on the convergence of the estimates. However, as this simulation is repeated for many settings in multiple iterations, we simply use a very large number of MCMC samples to ensure convergence without the need for manual checks (22'000 samples drawn per model, out of which 2'000 were burn-in). The results are coherent between replicates, which we interpret as a sign that the chains have indeed converged.

\subsubsection{Results: Estimation and prediction}
We first focus on the recovery of linear signal. 

Boxplots showing the estimated linear coefficients of RuleSHAP and the other methods are depicted in Figures~\ref{fig:FriedCoeffsp10} through~\ref{fig:LogiCoeffsp30} in the Supplementary Material. As expected, all models are able to correctly identify $x_6$ as noise, but RuleFit completely penalizes the linear terms out of the model. HorseRule does not induce a satisfactory improvement over RuleFit in this sense. RuleSHAP, on the other hand, is very likely to correctly identify $x_4$ and $x_5$ as linear terms.\par

Next, we focus on how well RuleSHAP reconstructs the effect of each feature. Figure~\ref{fig:exBARTRS} shows an example of how Shapley values may be plotted for such a purpose. We focus on the continuous outcome, where predictions of Random Forest, HorseRule and RuleSHAP are all on the same scale. Shapley values are computed for each of these models, to evaluate the quality of the feature effect reconstruction. For each replicate, the average squared distance between the true local effect and the estimated local effect across all datapoints is computed. To compensate for the fact that each feature's effect has a different scale, the squared distance of each signal feature is rescaled by the variance of that effect. The quantity thus indicates the \enquote{fraction of unexplained effect variance}. For example, the mean squared distance for the linear feature $x_4$ is defined as $\frac{1}{n}\sum_{i=1}^n\big(\phi_4(x^{(i)})-10 \cdot (x^{(i)}_4-0.5)\big)^2$, where $\phi_4(x^{(i)})$ is the Shapley value of $x_4$ computed for observation $x^{(i)}$, and 10 and 0.5 are the true linear effect and mean of $x_4$. To rescale, we divide by $\frac{1}{n}\sum_{i=1}^n\big(10 \cdot (x^{(i)}_4-0.5)\big)^2$. \par

\begin{figure}[h!]
    \centering
    \includegraphics[width=\linewidth]{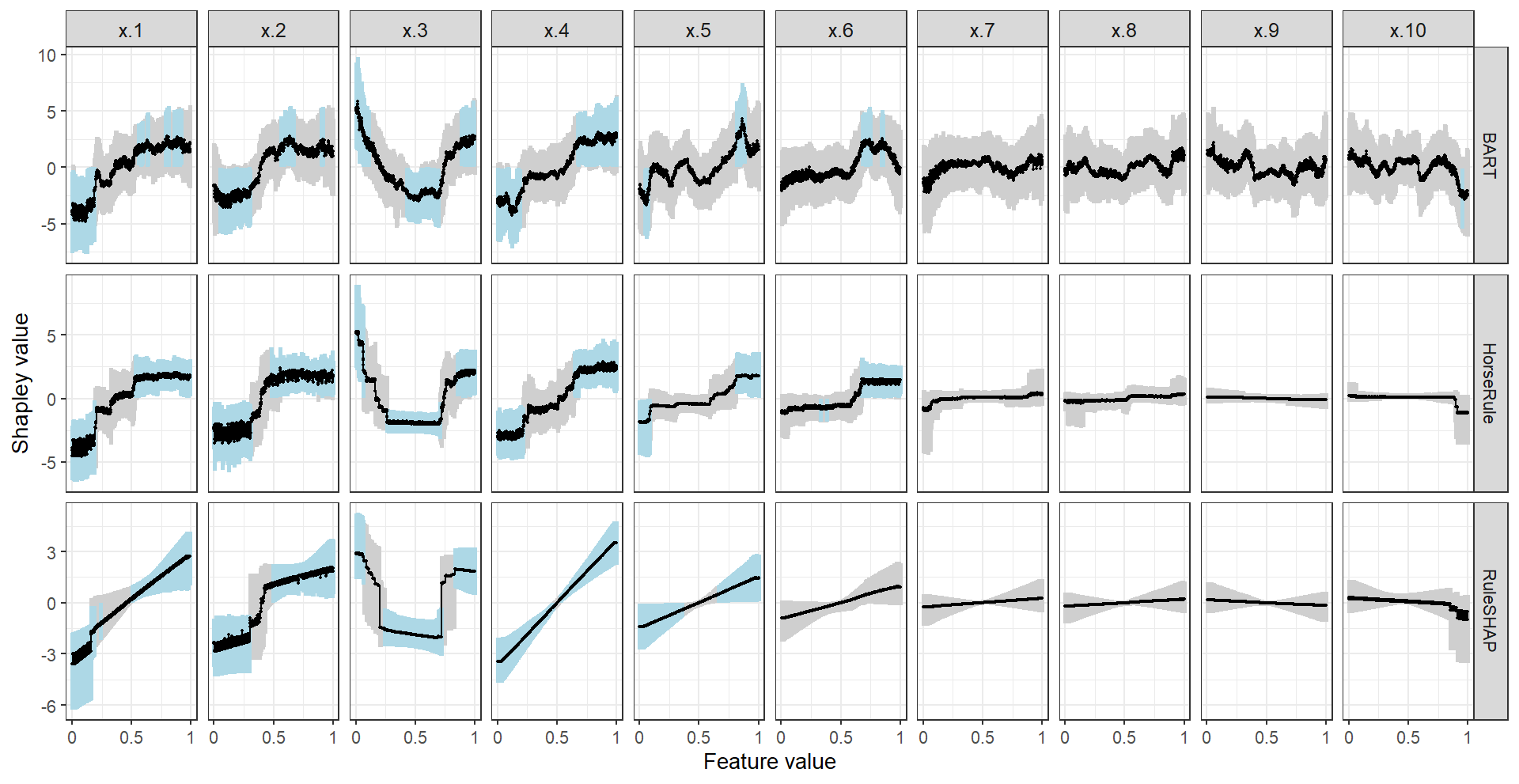}
    \caption{Example of Shapley values estimated from BART (first row), HorseRule (second row) and RuleSHAP (last row) fitted on the same $n=1'000$ Friedman-generated observations. Point-wise 95\% credible intervals are color-coded based on whether they contain zero (grey) or not (light blue).}
    \label{fig:exBARTRS}
\end{figure}

The boxplots in Figure~\ref{fig:LocImpDistsPreview} show the rescaled mean squared distances across replicates, for $p=10$ features and the three largest sample sizes. Full results for for $p=10$ and $p=30$ features are shown in the Supplementary Material in Figures~\ref{fig:LocImpDistsp10}~and~Figure~\ref{fig:LocImpDistsp30}, respectively. From these visualizations we observe that signal recovery for all methods improves with more observations. For low sample sizes, RuleSHAP performs better than all other models at identifying noise and purely linear effects. It performs comparable to the other methods in recovering the interaction but does not detect the purely nonlinear effect of feature $x_3$. This reflects an important property of RuleSHAP: the effects of $x_1,x_2,x_4,x_5$ are (partially) linear and reconstructed well across all sample sizes. The effect of feature $x_3$, on the other hand, can only be reconstructed by rules, and low sample sizes do not give enough certainty in the rule generation scheme for $x_3$ to emerge. HorseRule and Random Forest, on the other hand, employ a more uncertainty-insensitive rule-generation scheme, providing better recovery of the nonlinear effect of $x_3$ but also producing spurious effects of the noise features. At larger sample sizes, RuleSHAP's estimates of the feature effect are comparable to or better than the other methods', including for $x_3$ and especially for $p=30$. Random Forests remain unsatisfactory at effect reconstruction, even for settings with large sample size. The greedier nature of HorseRule and Random Forest imply a worse estimation of the (null) effect of noise, which in turn negatively affects the discrimination between noise and signal features, as we show with our next experiment.\par

\begin{figure}[h!]
    \centering
    \includegraphics[width=\linewidth]{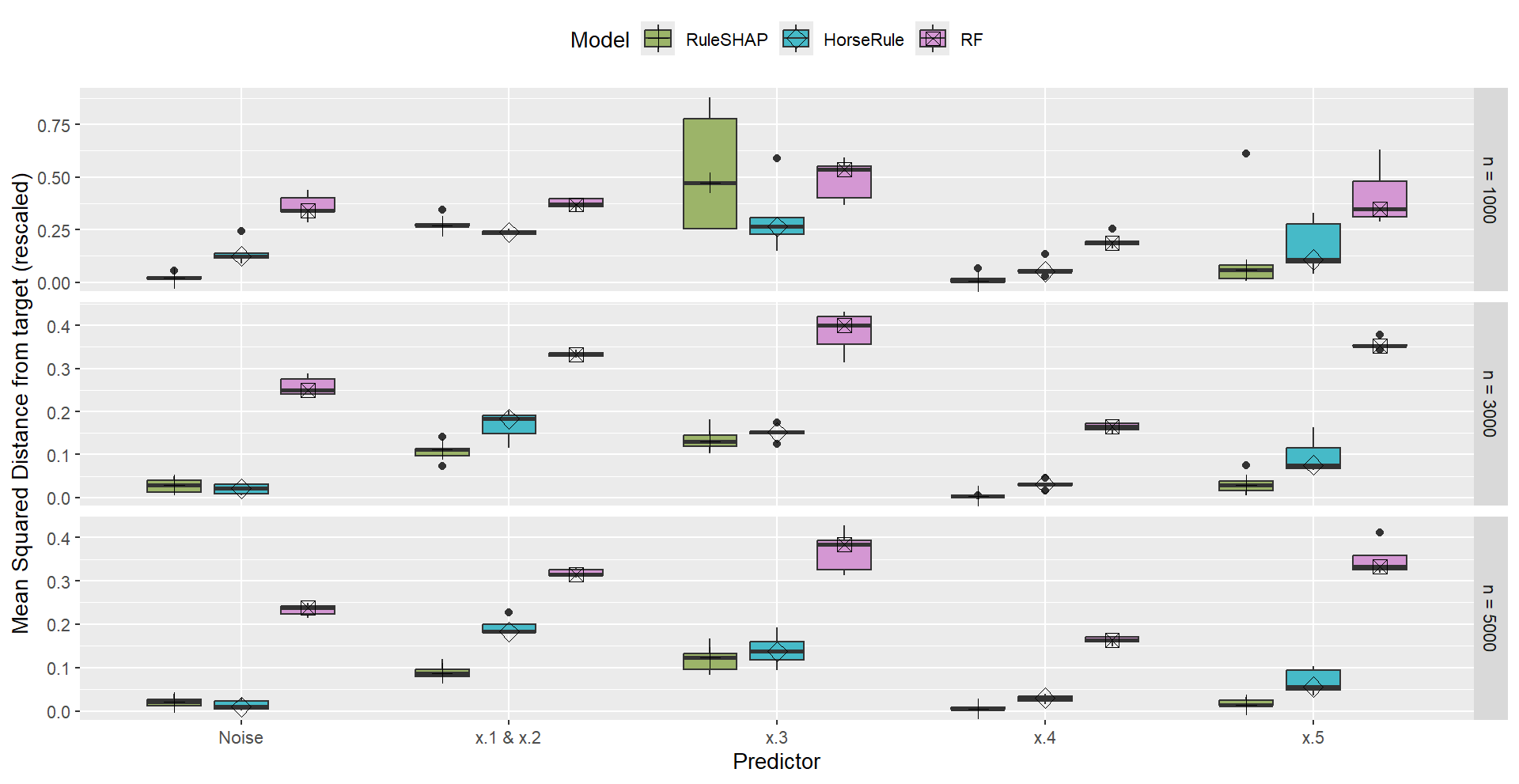}
    \caption{Rescaled mean squared distance between feature effects estimated by the model and the target effects for five replicates of different models fitted on Friedman-generated data, for $p=10$ features and some of the different sample sizes $n$. Full figure is shown in the Supplementary Material as Figure~\ref{fig:LocImpDistsp10}.}
    \label{fig:LocImpDistsPreview}
\end{figure}

In terms of predictive performance, RuleSHAP remains strongly competitive for all sample sizes. The test MSEs for the continuous outcome (Figure~\ref{fig:FriedMSE}) and the test area under the ROC curve for the binary outcome (Figure~\ref{fig:FriedAUC}) are shown in the Supplementary Material.

\subsection{Evaluation of uncertainty quantification}
\subsubsection{Method}
To evaluate uncertainty quantification, we compare RuleSHAP to BART and HorseRule, since all three may be combined with Shapley values to provide uncertainty estimates for local feature effects. In this experiment, we fit all three methods to datasets with three different $n/p$ ratios. More specifically, we use the Friedman generating function to generate $n=500,1'000,3'000$ observations of $p=10$ features. The error variance and correlation across features are fixed to $\sigma^2=100$ and $\rho=.3$ respectively, as in the previous experiments. BART, HorseRule and RuleSHAP are fitted on the data and Shapley values with 95\% credible intervals are computed for all three. Shapley values are computed using the \textit{treeshap} package \citep{treeSHAPpackage} for BART models, and using the expressions derived in Subsection \ref{subsec:DirectShapleys} for the RuleSHAP and HorseRule models. The experiment is replicated 100 times for each value of $n$.\par

We evaluate the credible intervals at the local level for both noise and signal features. For all features, we estimate the proportion of Shapley values with 95\% credible interval not containing zero. These proportions of `significant' values are averaged across noise features and signal features separately, rendering rejection rates, which are analyzed across replicates of the experiment. When estimating rejection rates, we further discriminate between small significant local effects (lowest 10\% in absolute value) and large significant local effects (anything else). These rates should be low for noise features and relatively high for the signal features, as the latter induce a true non-zero effect on all observations by design.

\subsubsection{Results: Local inference}

The density plots in Figure~\ref{fig:LocInfDens} show the estimated rejection rates, with average rejection rates indicated by a dot. We observe that the rejection rates of noise and signal features are the closest to each other for BART models. Although both RuleFit-based models better distinguish between signal and noise, RuleSHAP shows the lowest rejection rates for noise features and the highest ones for signal features.\par

\begin{figure}[h!]
    \centering
    \includegraphics[width=\linewidth]{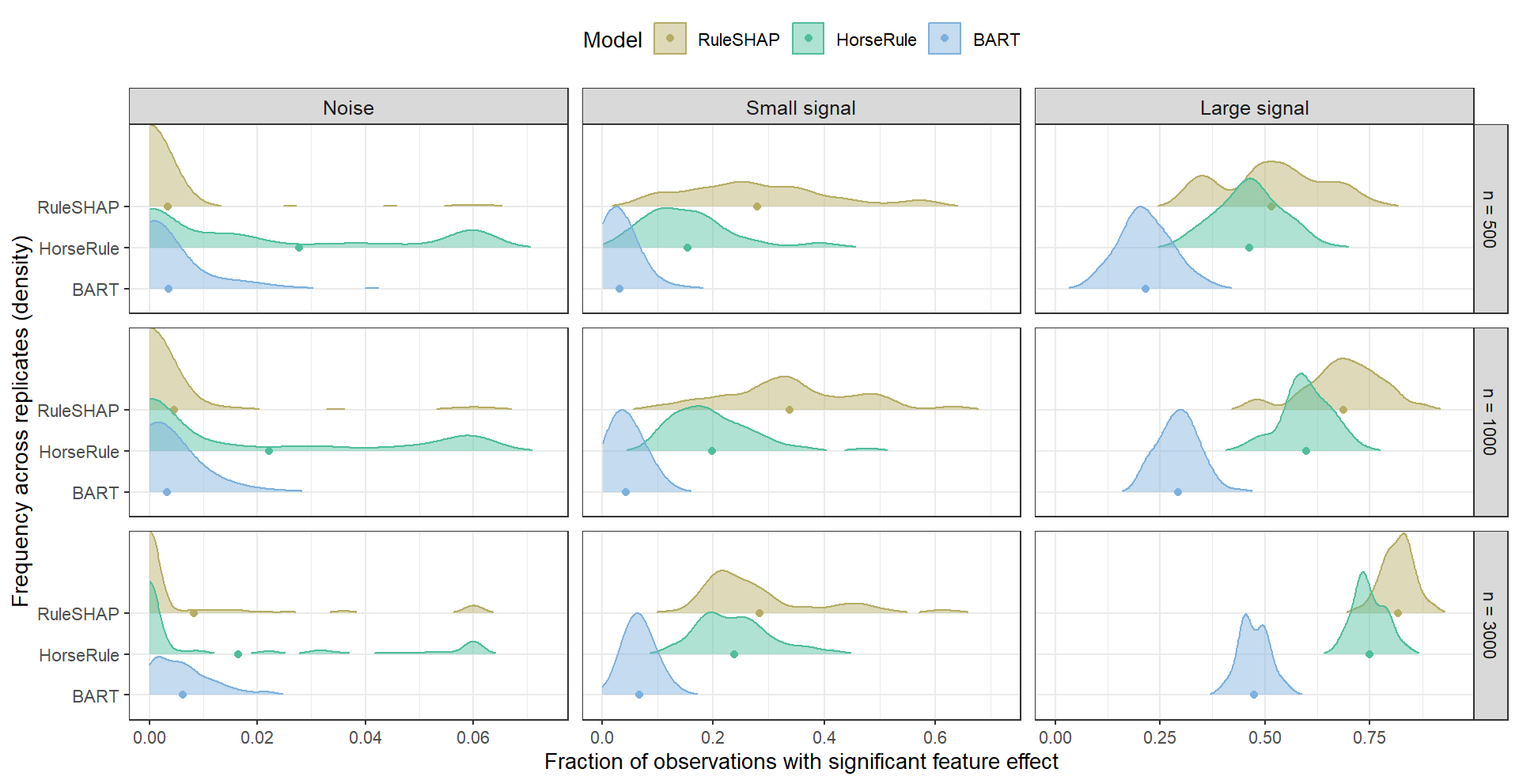}
    \caption{Density plot for rejection rates: the average proportion of significant Shapley values. Averages are computed separately for null effects from noise features (left column), small effects from signal features (middle column) and large effects from signal features (right column). The plot shows the variability across the 100 replicates of the experiment and compares BART, HorseRule and RuleSHAP. For noise features, estimated rejection rates are capped at .06.}
    \label{fig:LocInfDens}
\end{figure}

To illustrate the reason behind the good discriminative properties of RuleSHAP, Figure~\ref{fig:exBARTRS} shows the Shapley values produced by one of the BART, HorseRule and RuleSHAP fits. BART's (and to a lesser extent HorseRule's) feature effect estimations are less stable and underestimate uncertainty, yielding more frequently significant local effects for noise features. The RuleSHAP effect estimates appear less erratic and none of the Shapley values for noise features obtained significant credible intervals. This is likely due to RuleSHAP's more stable rule generation. \par

\section{Discussion}
\label{sec:Discussion}
We introduced RuleSHAP, which combines Bayesian sparse regression with tree ensembles and Shapley values to provide local inference for feature importance measures. We developed a rule-generation procedure that counters overfitting as well as overoptimistic uncertainty quantification. We developed a structured horseshoe prior that separates shrinkage of rule and linear effects. Finally, we derived an explicit formula for efficient computation of Shapley values. The results on simulated data indicate that RuleSHAP provides competitive predictive accuracy and superior local inference. More specifically, RuleSHAP strikes a better balance between linear terms and rules than existing methods, especially for larger sample sizes. This sets RuleSHAP apart from its predecessors, since both RuleFit and HorseRule seem mostly unable to accurately recognize linear signal, even when the underlying generating function is purely linear. This pattern confirms the findings by \cite{nalenz2018tree}, where the presence of linear terms in the HorseRule model was not relevant for predictive performance. \par

Further, RuleSHAP provides reliable detection and inference of beyond-linear trends and the latter can be performed on a local level. Many other ML alternatives are also typically considered for this context, but to our knowledge none of them can produce satisfactory results at the level of local effects. Generalized Additive Models (GAMs), for instance, can detect nonlinear effects and perform inference with them. However, such inference is only valid on a global level, since credible intervals for local effects are available but still suboptimal \citep{marra2012GAMcoverage}. Most importantly, GAMs (like GLMs) require manual specification of any interaction terms. These two elements impede discovery and reliable statistical inference for novel local effects. Tree ensembles, on the other hand, capture both nonlinearity and interaction effects without prior specification. In the context of Random Forests, they even include importance measures with uncertainty quantification \citep{ishwaran2019standard}. These measures, however, are also limited to a global setting and combine the marginal and interaction effects into one single feature importance. This makes any interpretation of what is being tested far from straightforward. The coverage of these intervals also requires further study, as the Out-Of-Bag estimation that they rely upon is similar to a nesting of bootstrapping and cross-validation. The bias of cross-validation can be non-negligible \citep{bates2024cross} and this may be propagated into the uncertainty quantification. Our simulations (Example~\ref{ex:TwoStepbad} in the Supplementary Material) showed very high type I error rates of these intervals.\par

\subsection{Results on cardiovascular health}
When applied to the data from the HELIUS study, RuleSHAP uncovers several significant trends related to systolic blood pressure (SBP). Our model suggests a significant, positive association between SBP and age. This association is well documented: for instance, \cite{sbp_age1} systematically found older participants to have higher SBP across different types of hypertension, while \cite{sbp_age2} recorded higher hypertention rates as age increased. Through capturing interaction terms, our model also reveals this assocation to be heterogeneous between ethnicities and between sexes, which has also been observed before \citep{sbp_age_ethn,sbp_age_sex}. RuleSHAP's significant, positive association between SBP and BMI is also well established in the literature \citep{sbp_bmi1,sbp_bmi2}, and described by \cite{sbp_bmi3} as \enquote{nearly linear}, matching the association shape depicted in Figure~\ref{fig:ShapleysSBP}. Lastly, RuleSHAP identifies a positive, significant association between SBP and glucose level, which for instance is also observed by \citep{sbp_gluc1} in middle-aged men and by \cite{sbp_gluc2} in both men and women. RuleSHAP estimates smoking to have a paradoxically negative association with SBP. These results have also been observed before \citep{sbp_smoke1,sbp_smoke_a}. \cite{sbp_smoke2} suggest that this kind of paradox might be due to systematic error in the measurement of SBP that is specific to smokers. Fitting a linear model on the same training data as RuleSHAP indeed produces the same significant, negative association, as shown in Table~\ref{tab:sbpLM} of the Supplementary Material.\par

The RuleSHAP model fitted to predict total cholesterol (TC) levels detects a positive association with smoking, a negative association with glucose level and a
complex association with age. The positive association between TC and smoking habits is well documented \citep{chol_smoke1,chol_smoke2,chol_smoke3}. However, the relationship between glucose and TC varies across papers. Previous work by \cite{chol_gluc2}, where TC was used to predict glucose level, specifically found non-linearity in the form of a negative association between the two, but only for low values of TC. This matches the shape we observe in Figure~\ref{fig:ShapleysChol}, albeit with reversed roles of predictor and outcome. However, work such as by \cite{chol_gluc3} suggests the opposite effect, as they estimated significantly higher odds for high TC among those with higher glucose level. A possible explanation for the discrepancy in results is that neither our analysis nor the one by \cite{chol_gluc2} accounted for diagnoses or treatments of diabetes, while the analyses of \cite{chol_gluc3} did. Our results may thus be affected by not accounting for this possible confounder. When a linear model is fitted on the same training data, a negative association is also identified (see Table~\ref{tab:cholLM} in the Supplementary Material). The association between TC and age is positive overall, with a slight decline for older ages. This matches what was for instance reported by \cite{chol_age}. The trend is also strongly hetereogeneous, stratified by BMI, glucose level, ethnicity, sex. In this motivating example, we focused on the strongest interaction. This is the interaction between sex and age, where female participants show notably higher cholesterol levels only for higher values of age, as also found by \cite{chol_age_sex}. The plot in Figure~\ref{fig:SexAgeInter} estimates that the strong discrepancy in the sex-specific age trajectories starts at 52 years of age, which is compatible with the average menopausal age of 50-51 in Western countries \citep{menopauseWest} and the broader average menopausal age range of 45-55 years \citep{menopauseAll}. This agrees with the literature suggesting that women have higher cholesterol levels after menopause \citep{chol_age_sex}.

\subsection{Related methods}
Other approaches than RuleSHAP also aim to quantify the uncertainty of local importance measures computed from an ML method. In the past few years, some of the literature has focused on estimating uncertainty at the level of the importance measures themselves \citep{watson2024explaining,slack2021reliable,williamson2020efficient}. However, this work is either only applicable to population-wide goodness-of-fit measures \citep{williamson2020efficient} or, more often, only focuses on ensuring that the \enquote{true} Shapley values of a fixed, specific model are covered by the confidence/credible intervals. This means that the variance of the model itself is not taken into account, thus underestimating the overall uncertainty of these measures. To overcome this issue and account for uncertainty in the model fit, a Bayesian Additive Regression Tree (BART) has also been used as a way to propagate model uncertainty to the Shapley values \citep{horiguchi2023estimating}. However, due to their complex nature, BART models tend to require more observations in order to work well. Furthermore, the intricacy of the model makes it difficult to specify a prior which can ensure valid coverage and inferential properties: earlier work shows its coverage to be suboptimal \citep{hahn2020bayesian} and to induce false positives, especially with smaller sample sizes \citep{sparapani2021nonparametric}. Valid coverage for spike-and-tree variations of BART have been proven asymptotically \citep{rovckova2020posterior}, but RuleSHAP's horseshoe prior allows for more realistic sample sizes instead \citep{van2017uncertainty}. Our simulations confirm the unsatisfactory uncertainty quantifications of BART at the level of the feature's effect. \par

\subsection{Limitations}
The current main limitation of RuleSHAP is its computational cost: the Bayesian linear fit used to estimate the coefficient requires $O(nP\min(n,P))$, where $P=p+q \approx 10'000$ is the number of terms in a RuleSHAP model. This means that computational time grows quadratically with sample size for $n \lessapprox 10'000$, and linearly thereafter. We relied on the use of a high performance cluster, where the computational time of fitting a single RuleSHAP model spanned from about 40 minutes for $n=300$ observations to short of five days for $n=5'000$ observations. A dataset of $n=1'000$ observations took approximately six hours to analyze. However, all these models comprised 20,000 MCMC samples to ensure convergence of the Markov chain without the need to assess it manually. For most practical settings, a Markov Chain would converge with fewer samples. Breaking or parallelizing the Bayesian regression Markov Chain into multiple smaller chains allows for computation on larger sample sizes. Importantly, computation time is relatively insensitive to $p$, as the number of terms in the model mostly depends on the number of rules that are generated. Nevertheless, as for the other methods, a larger value of $p$ would reduce the power to correctly identify the effect of signal features.\par

A number of common caveats apply for interpretation of a RuleSHAP model: without appropriate modifications, the results cannot be interpreted in a causal sense. Furthermore, the interpretation depends on the data at hand. This includes our analysis of the data from the HELIUS study: for example, self-sampling might also explain the interactions that we observed, since diabetes incidence varies with age and individuals with a diabetes diagnosis might be less inclined to participate to the study. These associations and interactions are thus not necessarily meaningful in an etiological sense and might be specific to the dataset that RuleSHAP is fitted on.\par

\subsection{Outlook}
Overall, these findings suggest that RuleSHAP offers a valid framework for hypothesis-free discovery and local inference. It offers a promising tool for both healthcare research, where inference options are strongly restrictive, and for transparent Machine Learning, where concerns about forms of bias or lack of transparency require optimal interpretability. These fields of application naturally inspire the need for further research and development. A thorough analysis of coverage could further support the use of RuleSHAP for hypothesis testing in scientific research. Global coverage may also be enabled by the use of tolerance intervals \citep{wallis1951tolerance}. Furthermore, the natural decomposition into feature effects offered by Shapley values might make it more accessible to test for the significance of groups of features or interactions. Faster implementations of RuleSHAP could be explored, for instance through mini-batch Gibbs sampling \citep{smolyakov2018adaptive} or with fast approximations of the Horseshoe regression \citep{johndrow2020scalable}. Finally, further research could focus on the direct formula that we produced for the computation of marginal Shapley values: using importance sampling in the estimation of the sample means could extend the results of Theorems~\ref{thm:ourFormula}~and~\ref{thm:ourFormulaInter} to estimate non-marginal Shapley values.\par

\section{Acknowledgements}
\label{sec:Acknowledgements}
We would like to thank Benny Markovitch for suggesting disaggregation of rules to improve separation of main and interaction effects. Marjolein Fokkema and Giorgio Spadaccini are supported by an NWO Vidi grant (number VI.Vidi.231G.065). The Amsterdam University Medical Centres and the Public Health Service of Amsterdam (GGD Amsterdam) provided core financial support for HELIUS. The HELIUS study is also funded by research
grants of the Dutch Heart Foundation (Hartstichting; grant no. 2010T084), the Netherlands Organization for Health Research and Development (ZonMw; grant no. 200500003), the European Integration Fund (EIF; grant no. 2013EIF013) and the European Union (Seventh Framework Programme, FP-7; grant no. 278901). This work was performed using the computational resources from the Academic Leiden Interdisciplinary Cluster Environment (ALICE) provided by Leiden University.

\section{Supplementary code}
\label{sec:Supplements}
The R code used to run the experiments and analyses is available on GitHub at the link \url{https://github.com/GiorgioSpadaccini/ruleSHAP}. The code to fit a RuleSHAP model has been converted into an R package which may be installed, for instance, using the \textit{ install\_github} function from the \textit{remotes} package. More details are available on the webpage linked above. The analyses were run under the R version 4.4.0 (2024-04-24).

\section{Data Availability Statement}
The synthetic data that support the findings of this study may be generated using the R code available in the GitHub repository linked in Section~\ref{sec:Supplements} using the R version 4.4.0 (2024-04-24). The data from the HELIUS study~\citep{snijder2017cohort} are available upon reasonable request to the HELIUS Scientific Coordinator (e-mail address: \href{mailto:heliuscoordinator@amsterdamumc.nl}{heliuscoordinator@amsterdamumc.nl}). Details about availability of the data may be found at \url{https://heliusstudy.nl/en/researchers}. The analysis in this study was approved by the HELIUS Executive Board (approval number: 240717) and deemed to not conflict with ethical approvals and informed consent forms of the HELIUS study.

\bibliographystyle{agsm}
\bibliography{reference}
\addcontentsline{toc}{chapter}{References}

\newpage

\section{Supplementary Material}

\renewcommand{\thefigure}{S\arabic{figure}}
\setcounter{figure}{0}
\renewcommand{\thetable}{S\arabic{table}}
\setcounter{table}{0}
\renewcommand{\thealgorithm}{S\arabic{algorithm}}
\setcounter{algorithm}{0}

\subsection{RuleFit's feature effect measures}
\label{subsec:RuleFitImp}
RuleFit's definitions of feature importance are based on the fact that RuleFit is a linear model. More specifically, the local importance of the $j$-th feature for the prediction of the datapoint $x$ is defined as:
\begin{equation}
    \label{eq:LocImpDef}
    \mathcal{I}_j(x):=|\hat{b}_j|\cdot |x_j-\overline{x}_j|+\sum_{\substack{k \text{ s.t.}\\ x_j \in r_k}}\frac{|\hat{a}_k|\cdot |r_k(x)-\overline{r}_k|}{m_k},
\end{equation}
where $\overline{x}_j$ and $\overline{r}_k$ are the observed means of $x_j$ and $r_k$ respectively, $m_k$ denotes the number of features involved in the $k$-th rule and, with an abuse of notation, $x_j \in r_k$ is meant as \enquote{the $j$-th feature is involved in rule $r_k$}. Concretely, this means that the local effect of the linear term $|\hat{b}_j|\cdot |x_j-\overline{x}_j|$ is added to the local effects of all rules that involve the $j$-th feature, and the effect of a rule $|\hat{a}_k| \cdot |r_k(x)-\overline{r}_k|$ is divided equally across all of the $m_k$ features it involves.\\

An analogous definition is used for the global importance of the $j$-th feature across all points:
\begin{equation}
    \label{eq:GlobImpDef}
    \mathcal{I}_j:=|\hat{b}_j|\cdot sd(x_j)+\sum_{\substack{k \text{ s.t.}\\ x_j \in r_k}}\frac{|\hat{a}_k|\cdot \sqrt{\overline{r}_k\cdot (1-\overline{r}_k)}}{m_k}.
\end{equation}
The same rationale is used here.\\

While these definitions come naturally and allow for easy computation, these measures are based on the intuition that the effects of the rules should be equally split across the features involved, which does not necessarily have to be the case. Furthermore, in Equation~\ref{eq:GlobImpDef}, the importances of the different terms are added up, even though they could possibly correspond to conflicting or complementary effects that partially compensate each other. We expand these considerations further with the following example:

\begin{ex}
\label{ex:FriedImp}
The feature importance measures $\mathcal{I}_j(x)$ and $\mathcal{I}_j$ from Equations~\ref{eq:LocImpDef}~and~\ref{eq:GlobImpDef} may be misleading:
\end{ex}
For a setting where $X_1,X_2,X_3 \sim \mathcal{N}(0,1)$, consider the two models:
$$F(x_1,x_2,x_3)=x_1+I(x_1 < -1)+I(x_2 < 3 , x_3 \geq 0) = x_1+r_1(x)+r_2(x)$$
$$G(x_1,x_2,x_3)=x_1+I(x_3 \geq 0)=x_1+r_3(x),$$
$$r_1(x):=I(x_1 < -1), \qquad r_2(x):=I(x_2 < 3 , x_3 \geq 0), \qquad r_3(x):=I(x_3 \geq 0).$$

In this context, RuleFit measures the importance of feature $X_1$ to be higher in model $F(x)$ than it is in model $G(x)$, both locally and globally:
$$\mathcal{I}^{(F)}_1=sd(x_1)+\sqrt{\overline{r}_1\cdot (1-\overline{r}_1)} > sd(x_1) = \mathcal{I}^{(G)}_1,$$
$$\mathcal{I}^{(F)}_1(x^*)=|x^*_1-\overline{x}_1|+|r_1(x^*)-\overline{r}_1| > |x^*_1-\overline{x}_1| = \mathcal{I}^{(G)}_1(x^*).$$
For this feature, however, the rule $r_1(x)$ has a dampening role on the effect of the linear term $x_1$, so the importance of $X_1$ should be lower in model $F(x)$ than in model $G(x)$. Although this issue could be addressed by removing the absolute values in the definition of the importance measures, that the concern of unfairly splitting the effect of interaction effects remains. With probability higher than $99.5\%$, rules $r_2(x)$ and $r_3(x)$ coincide, since the condition $X_2 < 3$ is very likely to be true. This means that $X_2$ technically appears in the model $F(x)$, but is essentially not relevant in it, similar to $G(x)$. Furthermore, the importance of $X_3$ in models $F(x)$ and $G(x)$ should be approximately the same. Yet, because RuleFit's feature importance measure splits the effect of an interaction equally between the involved features, the followingconditions are not satisfied either:
$$\mathcal{I}^{(F)}_2=\frac{\sqrt{\overline{r}_2\cdot (1-\overline{r}_2)}}{2} > 0 = \mathcal{I}^{(G)}_2,$$
$$\mathcal{I}^{(F)}_2(x)=\frac{|r_2(x)-\overline{r}_2|}{2} > 0 = \mathcal{I}^{(G)}_1(x),$$
$$\mathcal{I}^{(F)}_3=\frac{\sqrt{\overline{r}_2\cdot (1-\overline{r}_2)}}{2} \approx \frac{\sqrt{\overline{r}_3\cdot (1-\overline{r}_3)}}{2} < \sqrt{\overline{r}_3\cdot (1-\overline{r}_3)} = \mathcal{I}^{(G)}_3,$$
$$\mathcal{I}^{(F)}_3(x)=\frac{|r_2(x)-\overline{r}_2|}{2} \approx \frac{|r_3(x)-\overline{r}_3|}{2} < |r_3(x)-\overline{r}_3| = \mathcal{I}^{(G)}_1(x).$$

\subsection{HorseRule's Bayesian prior}
\label{subsec:HorseRulePrior}
HorseRule is a Bayesian variation of RuleFit where, instead of a LASSO fit, a Bayesian horseshoe regression is used:
$$y \in \mathbb{R}^n, \,\, X_L \in \mathbb{R}^{n \times p}, \,\, X_R \in \mathbb{R}^{n \times q}$$
$$y|X_R,X_L,a,b,\sigma^2 \sim \mathcal{N}(X_R \cdot a+X_L \cdot b,\sigma^2I_n),$$
$$a_k|\lambda,\tau,\sigma^2 \sim \mathcal{N}(0,\lambda_k^2\tau^2\sigma^2),$$
$$b_j|\gamma,\tau,\sigma^2 \sim \mathcal{N}(0,\gamma_j^2\tau^2\sigma^2),$$
$$\lambda_k \sim \mathcal{C}^+(0,A_k),$$
$$\gamma_j \sim \mathcal{C}^+(0,1),$$
$$\tau \sim \mathcal{C}^+(0,1),$$
$$\sigma^2 \sim \sigma^{-2}d\sigma^2,$$
where $X_L$ is the design matrix whose $p$ columns are the linear terms and $X_R$ is the design matrix whose $q$ columns are the rule terms appearing in Equation~\ref{eq:RuleFitLinearEq}. The notation $\mathcal{C}^+(0,A)$ denotes the half-Cauchy distribution of location 0 and scale $A$. The value $A_k$ that determines the shrinkage of the coefficient of a rule $r_k$ is defined as:
\begin{equation}
    \label{eq:NVHorseShoept2}
    A_k=\frac{\big(2\cdot \min(1-\overline{r}_k,\overline{r}_k)\big)^\mu}{(m_k)^\eta},
\end{equation}
where $\overline{r}_k$ is the observed frequency/mean of the rule and $m_k$ denotes the number of features involved in the definition of the $k$-th rule. The hyperparameters $\mu,\eta >0$ determine the level of shrinkage applied to specific, complex rules and the authors suggest default values of $\mu=1,\eta=2$ \citep{nalenz2018tree}. This prior favors the use of linear terms by construction, since linear terms are never shrunken more than a rule and are shrunken as much as a rule of support $\overline{r}_k=0.5$ and depth $m_k=1$.\\

RuleFit also discourages rules with unbalanced support, but does so implicitly. More specifically, it avoids standardization of the rule terms. This means that unbalanced rules are purposefully not rescaled by their small standard deviation and are therefore shrunken more. The two approaches are connected, since rescaling a feature before fitting a horseshoe model is equivalent to leaving the feature unchanged and re-scaling the local parameter $A_j$ instead:
\begin{rmk}
\label{rmk:Rescaling}
Standardizing or rescaling a feature is directly connected to its shrinkage in a horseshoe prior:
\end{rmk}
For simplicity, we refer to the horseshoe prior with structured shrinkage defined without distinction between linear terms and rules:
$$y \in \mathbb{R}^n, \,\, X \in \mathbb{R}^{n \times p},$$
$$y|X,\beta,\sigma^2 \sim \mathcal{N}(X \cdot \beta,\sigma^2I_n),$$
$$\beta_k|\lambda,\tau,\sigma^2 \sim \mathcal{N}(0,\lambda_k^2\tau^2\sigma^2),$$
$$\lambda_k \sim \mathcal{C}^+(0,A_k),$$
$$\tau \sim \mathcal{C}^+(0,1),$$
$$\sigma^2 \sim \sigma^{-2}d\sigma^2.$$
Given the feature vectors $x_1,\ldots,x_p$, which are columns of the design matrix $X$, one can rescale the features and turn them into $x_1'=x_1/\alpha_1,\ldots,x'_p=x_p/\alpha_p$, columns of the design matrix $X'$. If one fits a Bayesian regression on such rescaled features and produces a coefficient vector estimate $\beta'$, then returning on the original scale means computing the vector $\beta=\begin{pmatrix} \beta'_1/\alpha_1 \cdots \beta'_p/\alpha_p \end{pmatrix}^t$. In this case, the probability density is:
\begin{align*}
    \pi(\beta=b|X,y) &= \pi(\beta'=\alpha \odot b|X',y)\\
    &\propto e^{-\frac{1}{2\sigma^2}||y-\sum_{k=1}^p\frac{x_k}{\alpha_k} \cdot \alpha_kb_k||^2}\cdot e^{-\frac{1}{2\tau^2}\big(\sum_{k=1}^p\frac{b_k^2\alpha_k^2}{\lambda_k^2}\big)} \cdot \frac{1}{1+\tau^2} \cdot \prod_{k=1}^p \frac{1}{A_k^2+\lambda_k^2}\\
    &\propto e^{-\frac{1}{2\sigma^2}||y-\sum_{k=1}^px_kb_k||^2}\cdot e^{-\frac{1}{2\tau^2}\big(\sum_{k=1}^p\frac{b_k^2}{(\lambda_k/\alpha_k)^2}\big)} \cdot \frac{1}{1+\tau^2} \cdot \prod_{k=1}^p \frac{1}{(A_k/\alpha_k)^2+(\lambda_k/\alpha_k)^2}
\end{align*}
which, up to re-parametrizing $\lambda_k \leftarrow \lambda_k/\alpha_k$, is equivalent to the same prior as in HorseRule, but with the structured shrinkage parameter $A_k$ being replaced by $A_k/\alpha_k$. If we consider the specific setting of HorseRule where a rule term $r_k(x)$ would be rescaled by its standard deviation, this means replacing $A_k$ with:

\begin{align*}
    A_k'&=\frac{A_k}{sd(r_k)}\\
    &=\frac{A_k}{\sqrt{\overline{r}_k(1-\overline{r}_k)}}\\
    &=\frac{(2\min(\overline{r}_k,1-\overline{r}_k))^\mu}{m_k^\eta \sqrt{\overline{r}_k(1-\overline{r}_k)}}\\
    &=\frac{(2\min(\overline{r}_k,1-\overline{r}_k))^{\mu-0.5}}{m_k^\eta}\cdot\sqrt{\frac{2\min(\overline{r}_k,1-\overline{r}_k)}{\overline{r}_k(1-\overline{r}_k)}}\\
    &=\frac{(2\min(\overline{r}_k,1-\overline{r}_k))^{\mu-0.5}}{m_k^\eta}\cdot\sqrt{\frac{2}{\max(\overline{r}_k,1-\overline{r}_k)}}.
\end{align*}

In particular, standardization is equivalent to reducing the penalization parameter $\mu$ by 0.5 and introducing a (much more stable) multiplicative factor of $\sqrt{\frac{2}{\max(\overline{r}_k,1-\overline{r}_k)}}$. This rescaling prevents instability in the standardization process while having the same effect on shrinkage. This means that standardization may be performed implicitly by defining the following local shrinkage scale:
\begin{align*}
    A_k&=\frac{(2\min(\overline{r}_k,1-\overline{r}_k))^{\mu-0.5}}{m_k^\eta}\cdot\sqrt{\frac{2}{\max(\overline{r}_k,1-\overline{r}_k)}}\\
\end{align*}
If one wishes to ensure that the scale parameter $A_k$ for a rule is never higher than that of a linear term, $A_k$ may be multiplied by $\frac{1}{2}$. RuleSHAP does so, thus using the following local shrinkage scale:\\
\begin{align*}
    A_k&=\frac{(2\min(\overline{r}_k,1-\overline{r}_k))^{\mu-0.5}}{m_k^\eta\sqrt{2\max(\overline{r}_k,1-\overline{r}_k)}}\\
\end{align*}

\subsection{Lack of inference for Machine Learning models}
\label{subsec:NoInfer}
It is not straightforward to use HorseRule's uncertainty of the coefficients to perform inference on the importance measures defined for RuleFit, since the natural null for this measure, which would be $\mathcal{I}_j(x)=0$, is at the edge of the space where the measure is defined, i.e. $\mathcal{I}_j(x) \in \mathbb{R}^+$. This means that even for a noise feature, the confidence interval of $\mathcal{I}_j(x)$ would always exclude the zero and thus would always reject the null hypothesis. For example, a HorseRule model of the form $F(x)=a_1 I(x_1 > 0)+b_1 x_1$ would only show a non-significant importance if several posterior samples had $a_1=b_1=0$. However, the Gibbs sampling scheme almost surely precludes any coefficient from being exactly zero, since the Gibbs sampling scheme requires to draw every posterior sample of the coefficients by fitting a generalized Ridge regression \citep{makalic2015simple}.\\

In the literature one can also find other approaches used to combine ML with inference. For example, Random Forests may also rely on uncertainty quantifications for their variable importance measure (RF-VImp, \citep{ishwaran2019standard}). Other researchers rely on a two-step approach that retains the most significant features as detected from a ML model and then performs a linear regression to establish significance \citep{madakkatel2021combining,liu2023combining,madakkatel2023hypothesis}. However, this approach is suboptimal when using a single dataset for both steps: the noise features that are retained by the ML model are precisely the ones that, by chance, fit the current sample best, and the regression stage does not take this into account. The following example shows how both approaches produce high type I error rates.

\begin{ex}
\label{ex:TwoStepbad}
The two-step procedure that selects features with an ML model and subsequently performs inference with an OLS fit on the same data could lead to high type I error rates. Performing inference using RF-VImp with a Random Forest might also lead to high type I error rates:
\end{ex}

We set up an exemplary simulation to show the problems of this approach: more specifically, we simulated data under the Friedman generating function described in Equation~\ref{eq:FriedmanGenFunct}. The dataset comprised five signal and 25 noise features, for a total of  $p=30$ independent features and $n=1'000$ observations. The experiment was repeated 100 times. In each replicate, we fitted the two-step approach on the data. The approach involved fitting an XGBoost tree ensemble (R package \textit{xgboost}; \cite{XGBcite}). The eight features with highest average absolute Shapley value were selected and used to fit a linear model on the same dataset (R function \textit{lm}). In each replicate, the linear model was used to test for significance on these eight features. Each replicate also included fitting a Random Forest, and using RF-VImp to estimate uncertainty and thus perform inference. As shown in Table~\ref{tab:RFVImp}, the type I error rate estimated across replicates is much higher than the nominal level of $\alpha=0.05$, for both approaches.

\begin{table}[]
    \centering
    \begin{tabular}{l|cc}
      \multirow{2}{*}{Quantity} & \multicolumn{2}{c}{Model} \\
       & \text{RF-VImp} & \text{Two-step} \\\hline
       \# Null hypothesis rejections for noise features & 2'438 & 68 \\
       \# Tested noise features & 2'500 & 309 \\
       Estimated type I error rate & .9752 & .2201 \\
    \end{tabular}
    \caption{Observed frequencies of rejection of the null hypothesis for inference performed with Random Forest combined with RF-VImp and for inference performed with a two-step approach combining XGBoost and linear regression. Frequencies are summed across all 100 replicates of the experiment and type I error rate is estimated as the ratio between the times that the null hypothesis was rejected for any noise feature and the number of times that any noise feature was included in the testing model.}
    \label{tab:RFVImp}
\end{table}

\subsection{RuleSHAP implementation details}
Some details about the rule generation and the practical implementation of the RuleSHAP function are given in this subsection.\\

To prevent too many almost identical rules from being generated, some minor measures were taken on top of the ones discussed in the main text. For instance, winsorization was taken further and we enforced that the terminal nodes of the trees also contain at least 2.5\% of the observations, and in any case at least 10 observations. Furthermore, the thresholds involved in the definition of the rules were rounded to the third digit, if the feature was continuous. Lastly, we enforced a maximum rule depth of three features for optimal interpretability.\\

The practical implementation of the model followed the work of \citep{makalic2015simple,horseshoeR,horseshoenlmR,Villani2017HorseRuleR}, but the Gibbs sampling scheme was adapted to accommodate rule-specific and linear-specific shrinkages. The algorithm uses the following conditional distributions:

\begin{gather*}
\begin{bmatrix}
    a\\
    b
\end{bmatrix} \Big |\cdot \sim \mathcal{N}(\Sigma X^ty,\sigma^2\Sigma),\\
\lambda_k^2|\cdot \sim \mathcal{IG}(1,\frac{1}{\eta_k}+\frac{a_k^2}{2A_k^2\tau^2\tau_R^2\sigma^2}),\\
\eta_k | \cdot \sim \mathcal{IG}(1,1+\frac{1}{\lambda_k^2}),\\
\gamma_j^2|\cdot \sim \mathcal{IG}(1,\frac{1}{\nu_j}+\frac{b_j^2}{2\tau^2\tau_L^2\sigma^2}),\\
\nu_j | \cdot \sim \mathcal{IG}(1,1+\frac{1}{\gamma_j^2}),\\
\tau_R^2|\cdot \sim \mathcal{IG}(\frac{q+1}{2},\frac{1}{\xi_R}+\sum_{k=1}^q\frac{a_k^2}{2A_k^2\tau^2\lambda_k^2\sigma^2}),\\
\xi_R | \cdot \sim \mathcal{IG}(1,1+\frac{1}{\tau_R^2}),\\
\tau_L^2|\cdot \sim \mathcal{IG}(\frac{p+1}{2},\frac{1}{\xi_L}+\sum_{j=1}^p\frac{b_j^2}{2\tau^2\gamma_j^2\sigma^2}),\\
\xi_L | \cdot \sim \mathcal{IG}(1,1+\frac{1}{\tau_L^2}),\\
\tau^2|\cdot \sim \mathcal{IG}(\frac{p+q+1}{2},\frac{1}{\xi}+\sum_{k=1}^q\frac{a_k^2}{2A_k^2\tau_R^2\lambda_k^2\sigma^2}+\sum_{j=1}^p\frac{a_j^2}{2\tau_L^2\gamma_j^2\sigma^2}),\\
\xi | \cdot \sim \mathcal{IG}(1,1+\frac{1}{\tau^2}),\\
\sigma^2 | \cdot \sim \mathcal{IG}\Big(\frac{n+p}{2},\frac{||y-X_R \cdot  - X_L \cdot b||^2}{2}+\sum_{k=1}^q\frac{a_k^2}{2A_k^2\tau^2\tau_R^2\lambda_k^2}+\sum_{j=1}^p\frac{a_j^2}{2\tau^2\tau_L^2\gamma_j^2}\Big),\\
\end{gather*}
with $X=(X_R|X_L),\Sigma=(X^tX+\Lambda^{-1})^{-1}$ and
$$\Lambda=\begin{pmatrix}
    \tau^2\tau_R^2\lambda_1^2A_1^2 & 0 & \cdots & 0 & 0 & \cdots & 0 \\
    0& \tau^2\tau_R^2\lambda_2^2A_2^2 & \cdots & 0 & 0 & \cdots & 0\\
    \vdots & \vdots & \ddots & \vdots & \vdots & \ddots & \vdots\\
    0 & 0 & \cdots & \tau^2\tau_R^2\lambda_k^2A_k^2 & 0 & 0 & 0\\
    0 & 0 & \cdots & 0 & \tau^2\tau_L^2\gamma_1^2 & 0 & 0 \\
    \vdots & \vdots & \ddots & \vdots & \vdots & \ddots & \vdots\\
    0 & 0 & \cdots & 0 & 0 & \cdots & \tau^2\tau_L^2\gamma_p^2\\
\end{pmatrix}$$

To estimate features' effects, marginal Shapley values are computed with Algorithms~\ref{alg:ourShapley}~and~\ref{alg:ourShapleyInt} to estimate overall effects and interaction effects respectively.

\begin{algorithm}[hbt!]
\caption{Estimating Shapley values $\phi_1(x^*),\ldots,\phi_p(x^*)$ for $r(x)=\prod_{k=1}^pR_k(x_k)$}
\label{alg:ourShapley}
\begin{algorithmic}
\Require 
\State Dataset $x^{(1)},\ldots,x^{(n)} \in \mathbb{R}^p$
\State Observation $x^{*} \in \mathbb{R}^p$
\State $R_1(x_1^{(1)}),\ldots,R_p(x_p^{(1)}),\ldots,R_1(x_1^{(n)}),\ldots,R_p(x_p^{(n)}) \in \{0,1\}$
\State $R_1(x_1^*),\ldots,R_p(x_p^*) \in \{0,1\}$
\Ensure
\State $q^* \gets \sum_{k=1}^p R_k(x^{*}_k)$
\State $phiSum_1 \gets \cdots \gets phiSum_p \gets 0$ \Comment{Preallocate summations as in formula}
\For{$i=1,\ldots,n$} \Comment{Loop over points used for estimates}
\If{$\sum_{k=1}^p(1-R_k(x_k^{(i)}))(1-R_k(x_k^{*}))=0$} \Comment{Check if point contributes to estimates}
    \State $q \gets \sum_{k=1}^p R_k(x^{(i)}_k)$
    \For{$j=1,\ldots,p$}
    \State $phiSum_j \gets phiSum_j + \frac{R_j(x_j^*)-R_j(x_j)}{\binom{2p-q^*-q-1+R_j(x^*_j)+R_j(x^{(i)}_j)}{p-q^*+R_j(x^*_j)}}$ \Comment{Update summations}
    \EndFor
\EndIf
\EndFor
\For{$j=1,\ldots,p$}
    \State $\phi_j \gets \frac{phiSum_j}{n(p-q^*+R_j(x^*_j))} $ \Comment{Rescale summations to obtain Shapley values}
\EndFor
\end{algorithmic}
\end{algorithm}

\begin{algorithm}[hbt!]
\caption{Estimating interaction Shapley values $\{\phi_{j,j'}(x^*)\}_{j,j'=1,\ldots,p}$ for $r(x)=\prod_{k=1}^pR_k(x_k)$}
\label{alg:ourShapleyInt}
\begin{algorithmic}
\Require 
\State $R_1(x_1^{(1)}),\ldots,R_p(x_p^{(1)}),\ldots,R_1(x_1^{(n)}),\ldots,R_p(x_p^{(n)}) \in \{0,1\}$
\State $R_1(x_1^*),\ldots,R_p(x_p^*) \in \{0,1\}$
\Ensure
\State $q^* \gets \sum_{k=1}^p R_k(x^{*}_k)$
\State $phiSum_{1,1} \gets \cdots \gets phiSum_{1,p} \cdots \gets phiSum_{p,p} \gets 0$ \Comment{Preallocate summations as in formula}
\For{$i=1,\ldots,n$} \Comment{Loop over points used for estimates}
\If{$\sum_{k=1}^p(1-R_k(x_k^{(i)}))(1-R_k(x_k^{*}))=0$} \Comment{Check if point contributes to estimates}
    \State $q \gets \sum_{k=1}^p R_k(x^{(i)}_k)$
    \For{$j=2,\ldots,p$}
    \For{$j'=1,\ldots,j-1$}
    \State $phiSum_{j,j'} \gets phiSum_{j,j'} + \frac{R_j(x_j^*)R_{j'}(x_{j'}^*)+R_j(x^{(i)}_j)R_{j'}(x^{(i)}_{j'})-R_j(x_j^*)R_{j'}(x^{(i)}_{j'})-R_j(x^{(i)}_j)R_{j'}(x_{j'}^*)}{\binom{2p-q^*-q-3+R_j(x^*_j)+R_{j'}(x^*_{j'})+R_j(x^{(i)}_j)+R_{j'}(x^{(i)}_{j'})}{p-q^*+R_j(x^*_j)+R_{j'}(x^*_{j'})}}$
    \EndFor
    \EndFor
\EndIf
\EndFor
\For{$j=2,\ldots,p$}
\For{$j'=1,\ldots,j-1$}
    \State $\phi_{j,j'} \gets \frac{phiSum_{j,j'}}{n(p-1-q^*+R_j(x^*_j)+R_k(x^*_{j'}))} $ \Comment{Rescale summations to obtain Shapley values}
    \State $\phi_{j',j} \gets \phi_{j,j'}$
\EndFor
\EndFor
\end{algorithmic}
\end{algorithm}

\FloatBarrier

\newpage
\subsection{Figures and results}
\FloatBarrier
\begin{figure}[h]
    \centering
    \includegraphics[width=\linewidth]{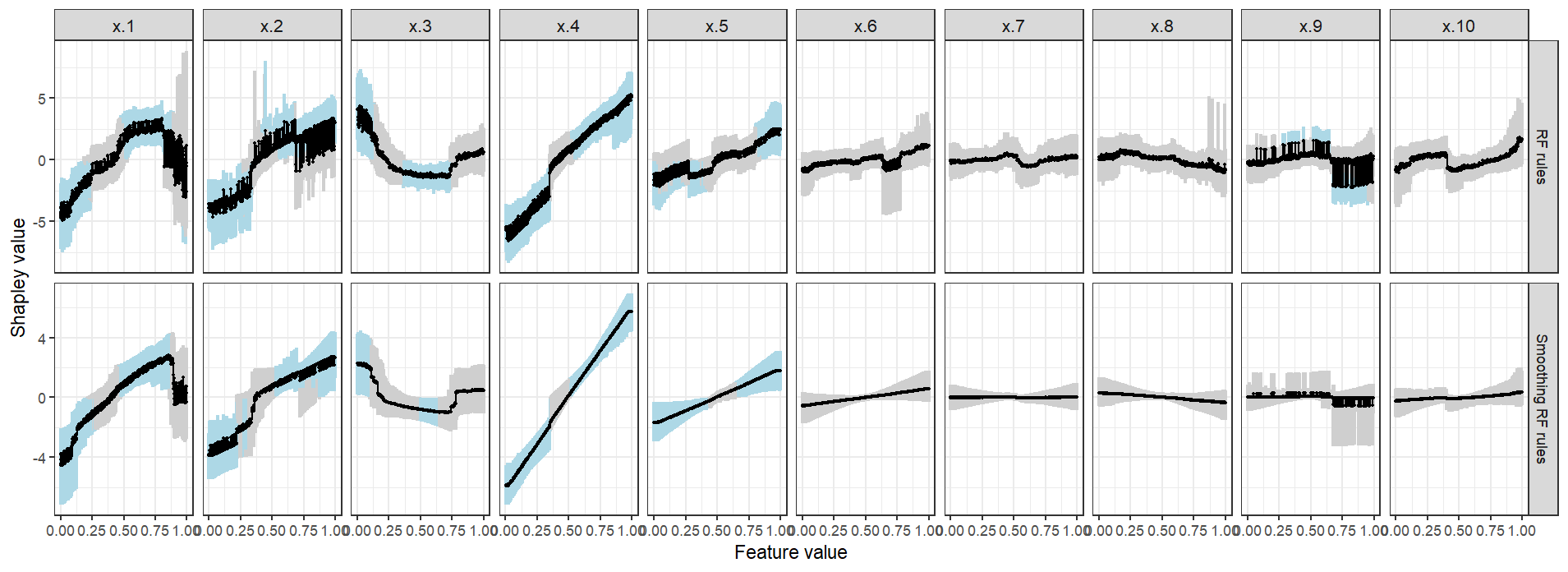}
    \caption{Feature effects as estimated with marginal Shapley values computed on a RuleSHAP model. In the top row, rules are generated using a non-modified Random Forest. In the bottom row, rules are generated using a Smoothing Random Forest. 95\% credible intervals for the Shapley values are gray if they overlap with zero and light blue otherwise. Fits were performed on $n=1'000$ observations for $p=10$ features with a pairwise correlation of $\rho=0.3$. The data was generated under the Friedman generation scheme, with an error variance $\sigma^2=100$.}
    \label{fig:RFvPRF}
\end{figure}

\begin{figure}[h]
    \centering
    \includegraphics[width=\linewidth]{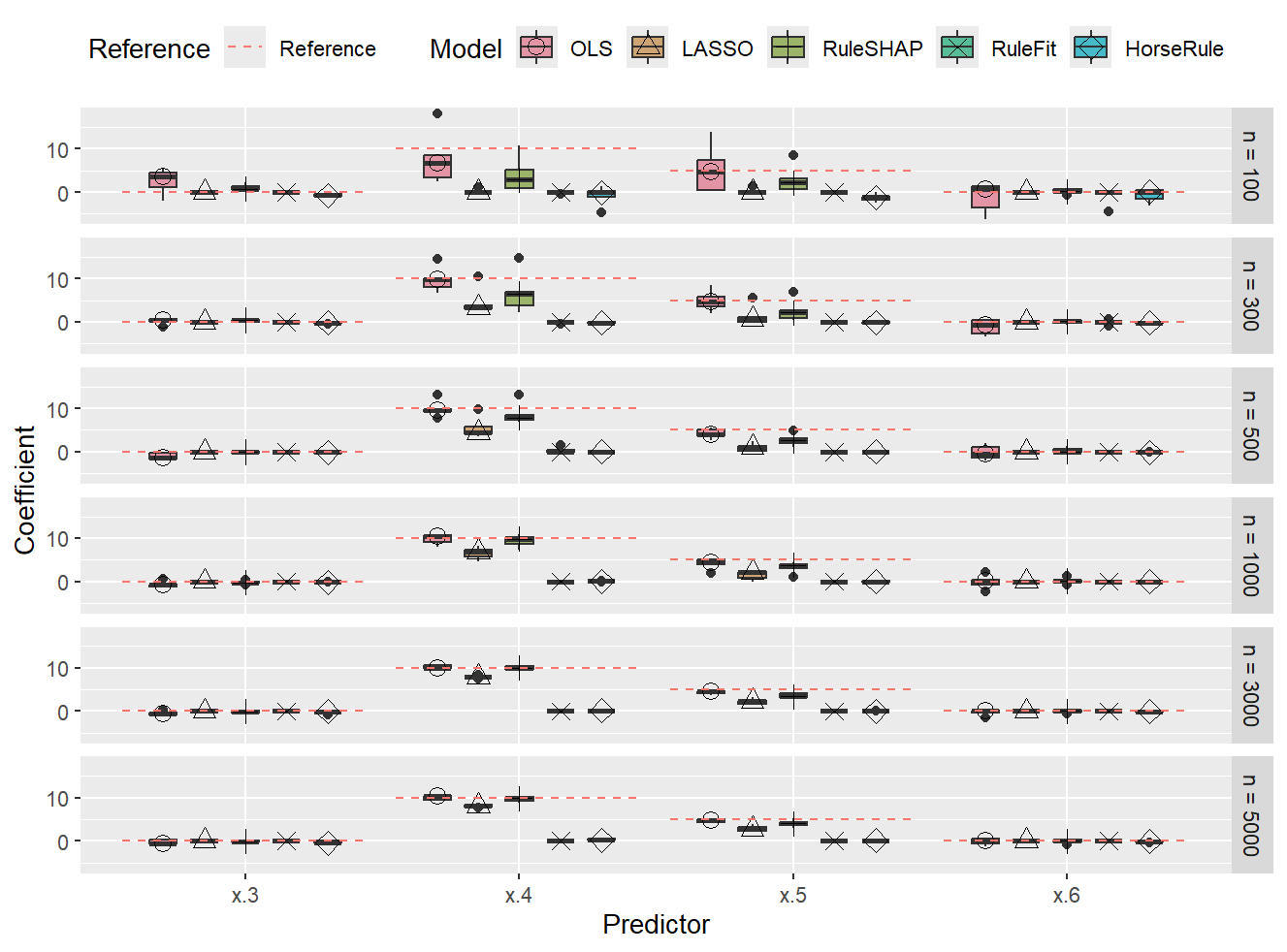}
    \caption{Coefficients of linear terms across five fits of OLS linear regression, LASSO regression, RuleFit regression, HorseRule regression \iffalse(HR1 for default settings; HR2 for custom settings)\fi  and RuleSHAP regression. These models are fitted on the Friedman-generated data with continuous outcome and $p=10$, across different values of $n$. Target coefficients are presented with dashed red lines. Note that the boxplots do not represent uncertainty, but rather variation across the five replicates of the experiment.}
    \label{fig:FriedCoeffsp10}
\end{figure}

\begin{figure}[h]
    \centering
    \includegraphics[width=\linewidth]{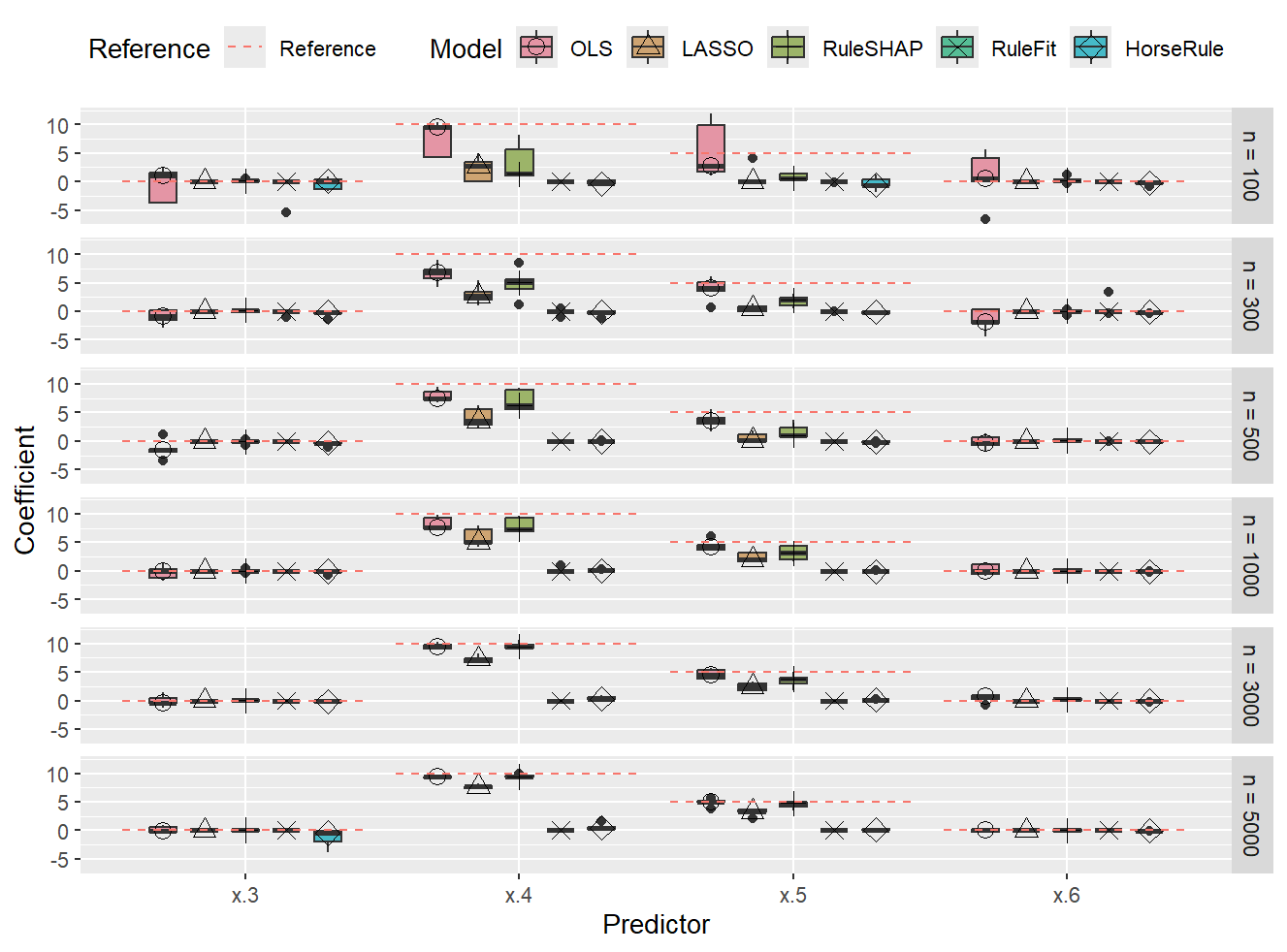}
    \caption{Coefficients of linear terms across five fits of OLS linear regression, LASSO regression, RuleFit regression, HorseRule regression \iffalse(HR1 for default settings; HR2 for custom settings)\fi  and RuleSHAP regression. These models are fitted on the Friedman-generated data with continuous outcome and $p=30$, across different values of $n$. Target coefficients are presented with dashed red lines. Note that the boxplots do not represent uncertainty, but rather variation across the five replicates of the experiment.}
    \label{fig:FriedCoeffsp30}
\end{figure}

\begin{figure}[h]
    \centering
    \includegraphics[width=\linewidth]{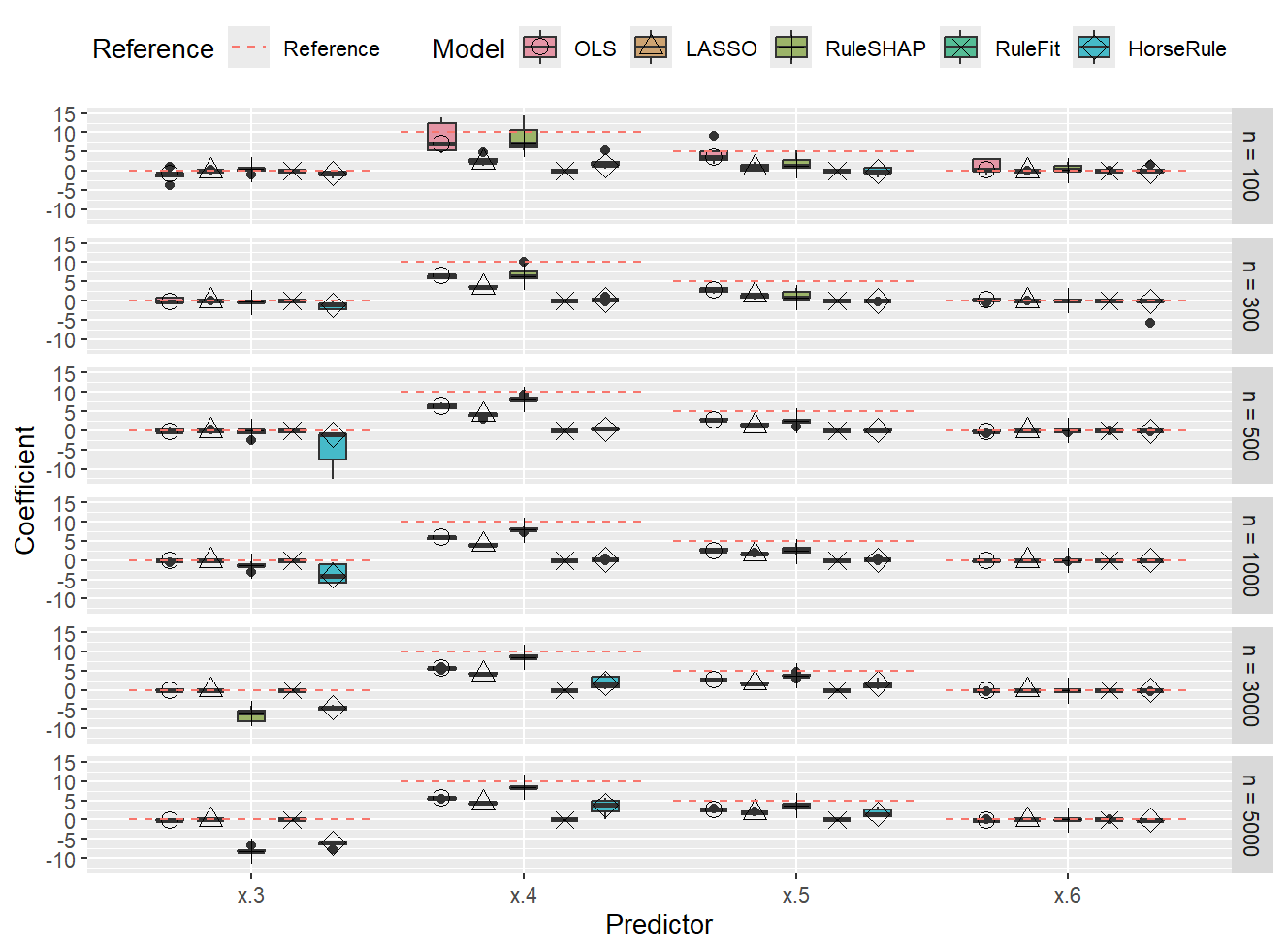}
    \caption{Coefficients of linear terms across five fits of logistic linear regression, LASSO logistic regression, RuleFit logistic regression, HorseRule logistic regression \iffalse(HR1 for default settings; HR2 for custom settings)\fi  and RuleSHAP logistic regression. These models are fitted on the Friedman-generated data with binary outcome and $p=10$, across different values of $n$. Target level is presented with dashed red lines. Note that the boxplot does not represent uncertainty quantifications, but rather the variation across the five replicates of the experiment.}
    \label{fig:LogiCoeffsp10}
\end{figure}

\begin{figure}[h]
    \centering
    \includegraphics[width=\linewidth]{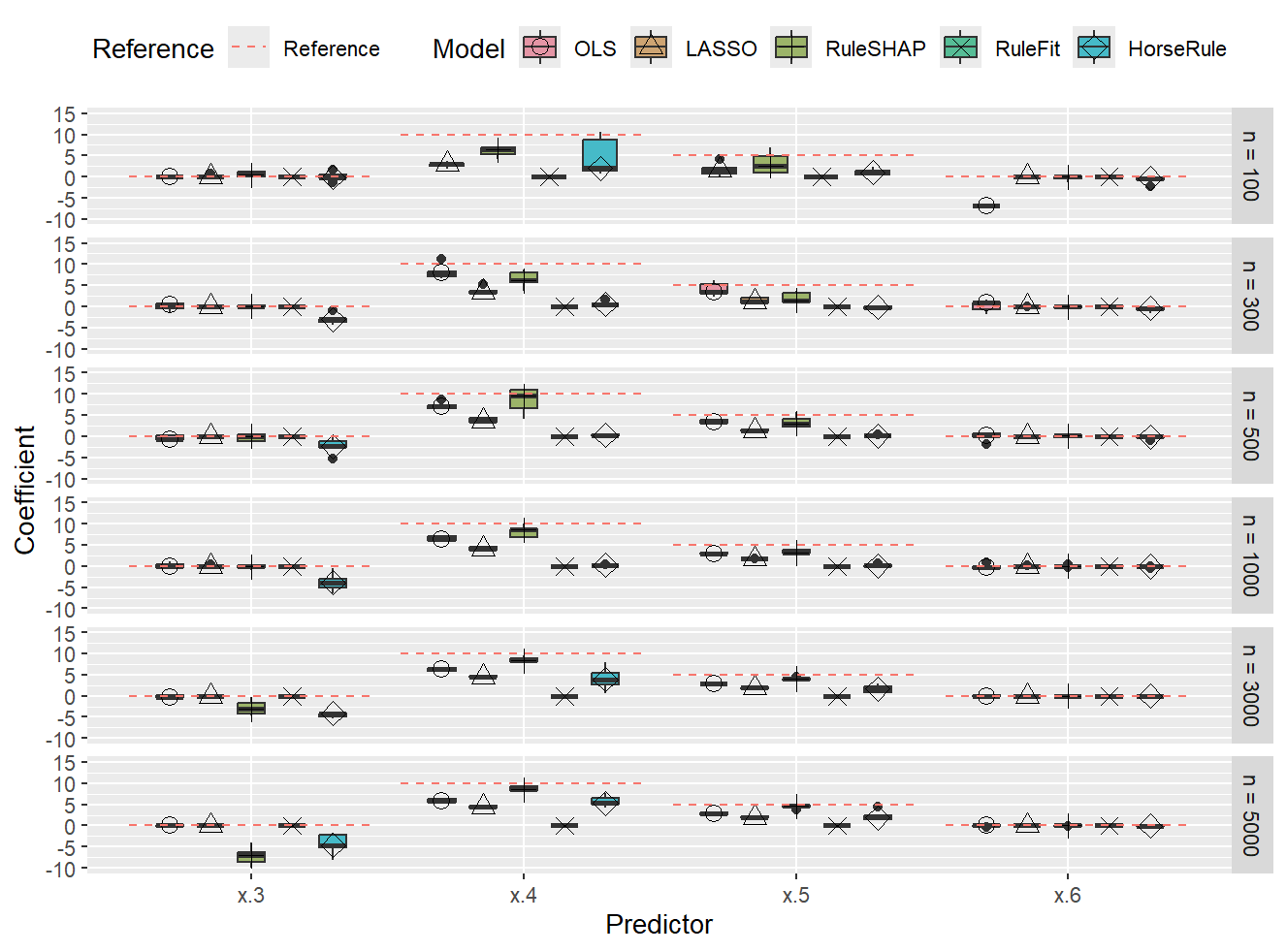}
    \caption{Coefficients of linear terms across five fits of logistic linear regression, LASSO logistic regression, RuleFit logistic regression, HorseRule logistic regression \iffalse(HR1 for default settings; HR2 for custom settings)\fi  and RuleSHAP logistic regression. These models are fitted on the Friedman-generated data with binary outcome and $p=30$, across different values of $n$. Target level is presented with dashed red lines. Note that the boxplot does not represent uncertainty quantifications, but rather the variation across the five replicates of the experiment.}
    \label{fig:LogiCoeffsp30}
\end{figure}

\begin{figure}[h]
    \centering
    \includegraphics[width=\linewidth]{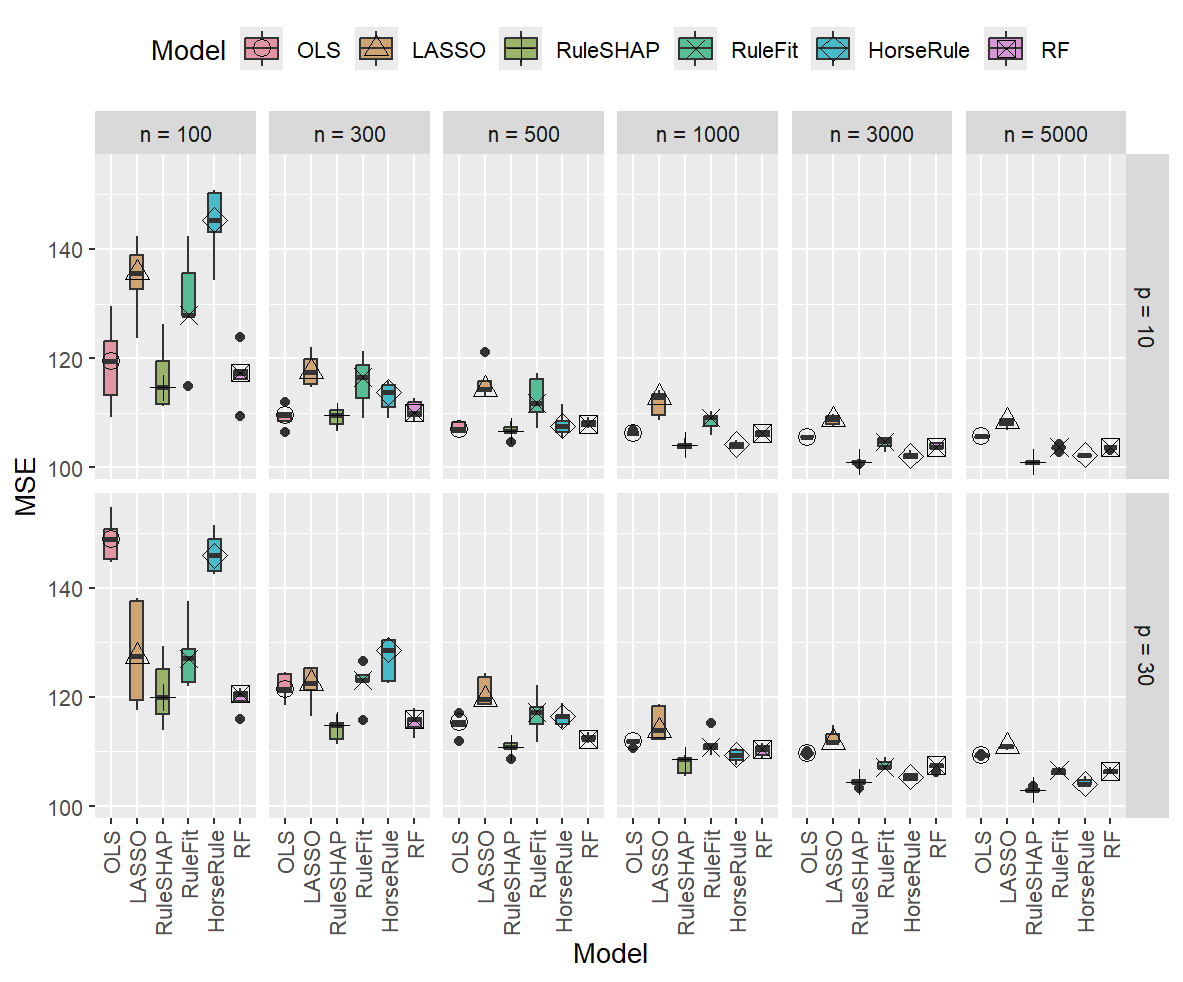}
    \caption{Test MSE across five fits of OLS linear regression, LASSO regression, RuleFit regression, HorseRule regression \iffalse(HR1 for default settings; HR2 for custom settings)\fi , RuleSHAP regression and Random Forest. These models are fitted on the Friedman-generated data with continuous outcome, across different values of $n$ and $p$.}
    \label{fig:FriedMSE}
\end{figure}

\begin{figure}[h]
    \centering
    \includegraphics[width=\linewidth]{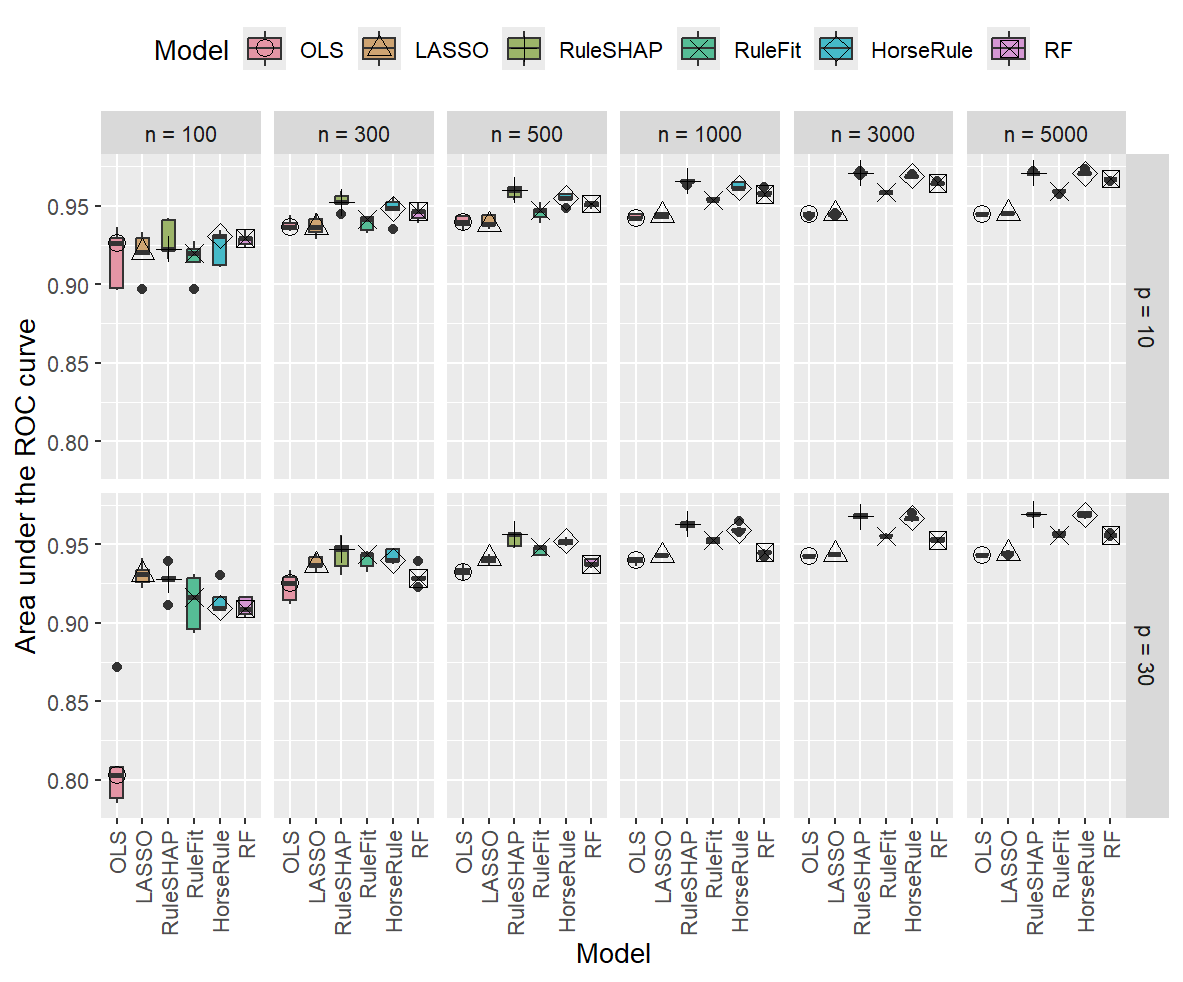}
    \caption{Test Area Under the ROC Curve (AUC) across five fits of logistic linear regression, LASSO logistic regression, RuleFit logistic regression, HorseRule logistic regression \iffalse(HR1 for default settings; HR2 for custom settings)\fi , RuleSHAP logistic regression and Random Forest. These models are fitted on the Friedman-generated data with binary outcome, across different values of $n$ and $p$.}
    \label{fig:FriedAUC}
\end{figure}

\begin{figure}[h]
    \centering
    \includegraphics[width=\linewidth]{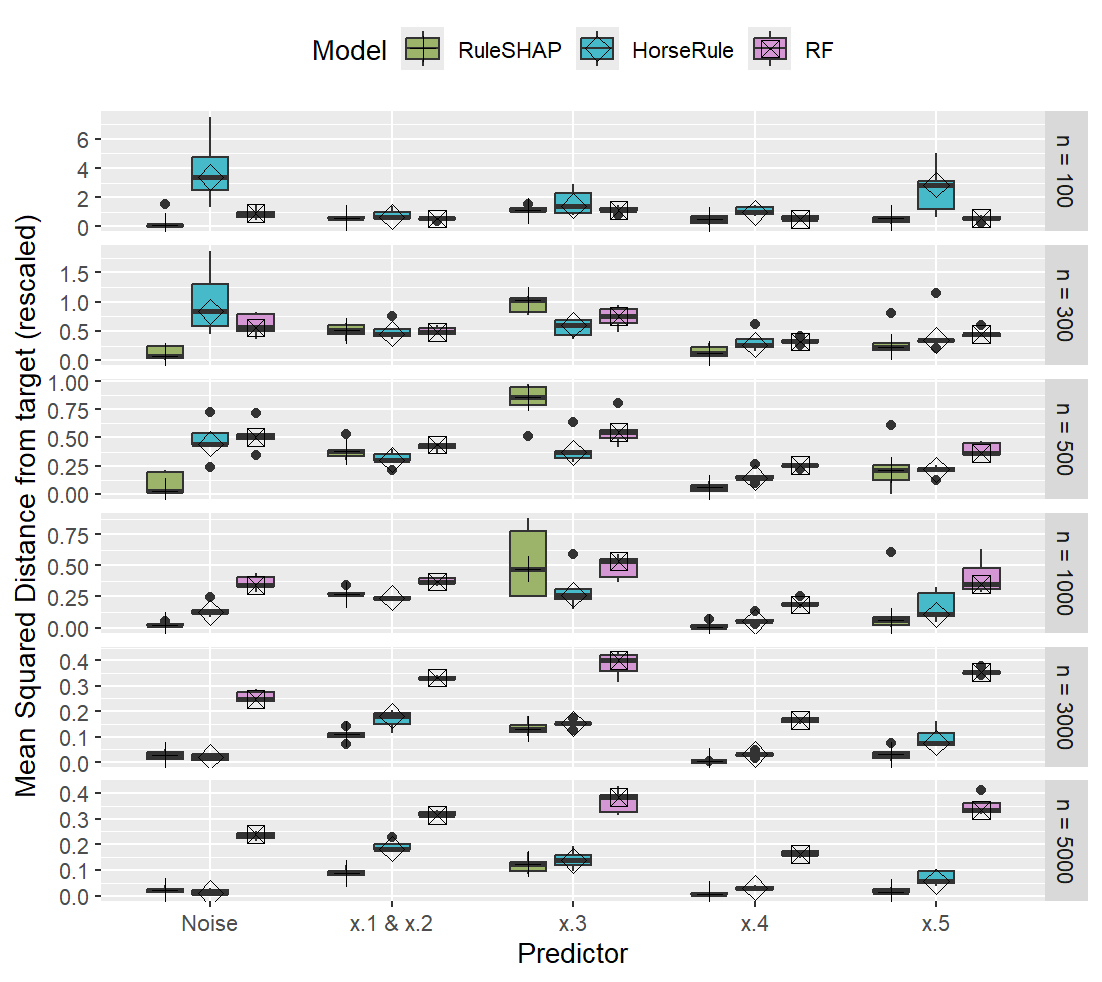}
    \caption{Re-scaled mean squared distance between feature effects estimated by the model and the target effects for five replicates of different models fitted on Friedman-generated data with continuous outcomes, for $p=10$ features and different sample sizes $n$. Estimated feature effects are computed as marginal Shapley values. Since $x_1$ and $x_2$ induce an interacting effect, their feature effects are analyzed jointly, as the sum of the Shapley values of the two features. The distance is averaged across all replicates of the experiment, across all points of the dataset. All noise features are averaged out together. To ease interpretation across different effect magnitudes, the squared distance of each signal feature was rescaled by the overall variance of that effect. The scale on the y axis varies across plots.}
    \label{fig:LocImpDistsp10}
\end{figure}

\begin{figure}[h]
    \centering
    \includegraphics[width=\linewidth]{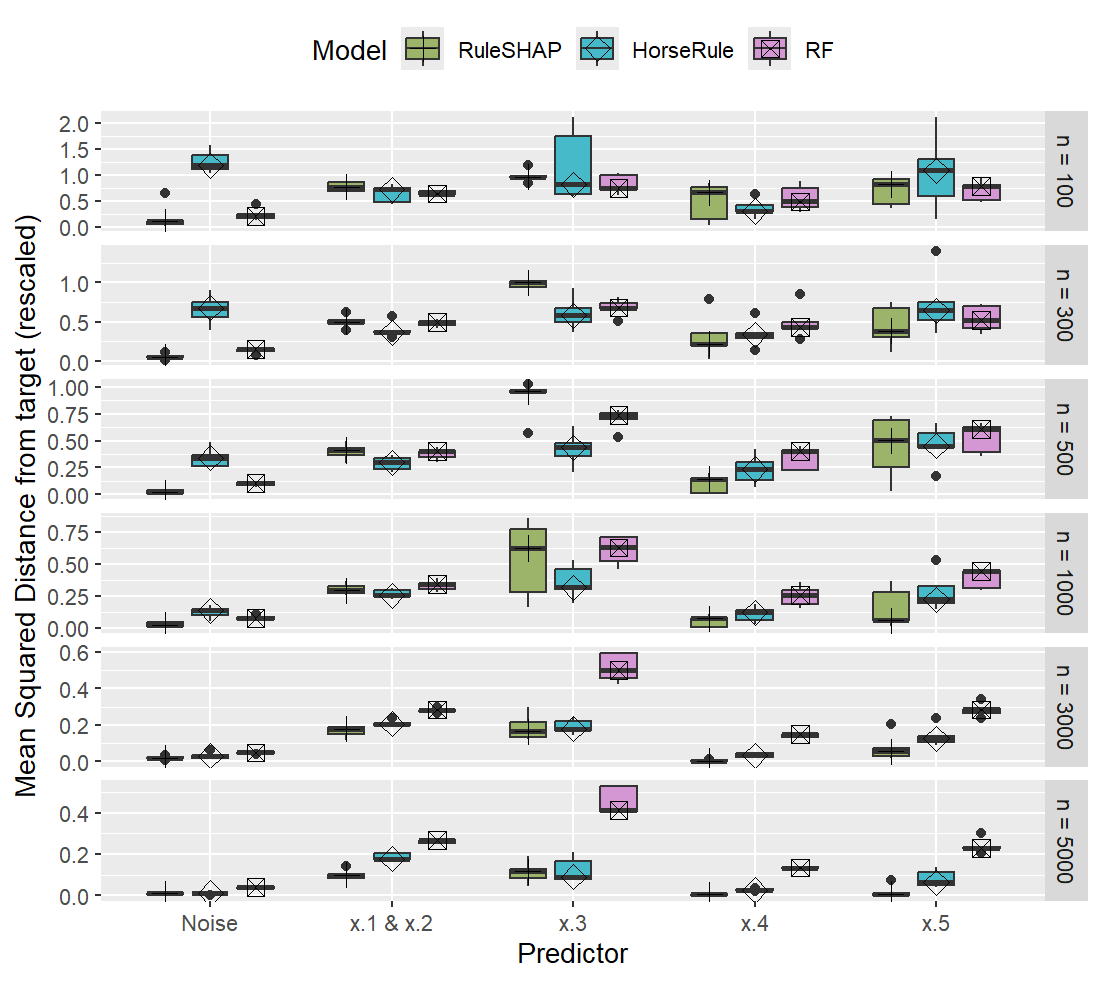}
    \caption{Re-scaled mean squared distance between feature effects estimated by the model and the target effects for five replicates of different models fitted on Friedman-generated data with continuous outcome, for $p=30$ features and different sample sizes $n$. Estimated feature effects are computed as marginal Shapley values. Since $x_1$ and $x_2$ induce an interacting effect, their feature effects are analyzed jointly, as the sum of the Shapley values of the two features. The distance is averaged across all replicates of the experiment, across all points of the dataset. All noise features are averaged out together. To ease interpretation across different effect magnitudes, the squared distance of each signal feature was rescaled by the overall variance of that effect. The scale on the y axis varies across plots.}
    \label{fig:LocImpDistsp30}
\end{figure}

\begin{figure}[h]
    \centering
    \includegraphics[width=\linewidth]{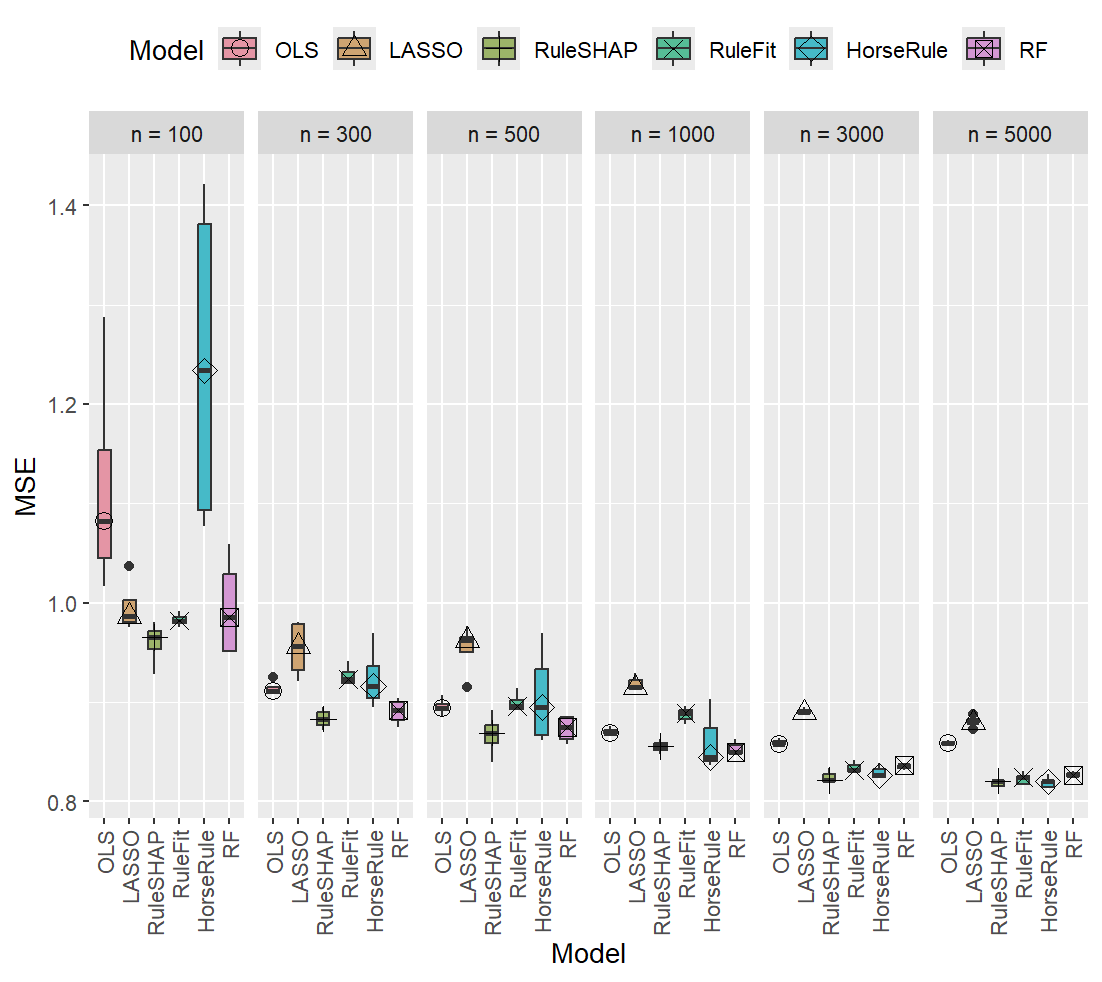}
    \caption{Test MSE across five fits of OLS linear regression, LASSO regression, RuleFit regression, HorseRule regression \iffalse(HR1 for default settings; HR2 for custom settings)\fi , RuleSHAP regression and Random Forest. These models are fitted on subsamples of different sizes $n$ taken from the HELIUS study data, with \textbf{cholesterol} as predicted outcome.}
    \label{fig:HeliusMSEchol}
\end{figure}

\begin{figure}[h]
    \centering
    \includegraphics[width=\linewidth]{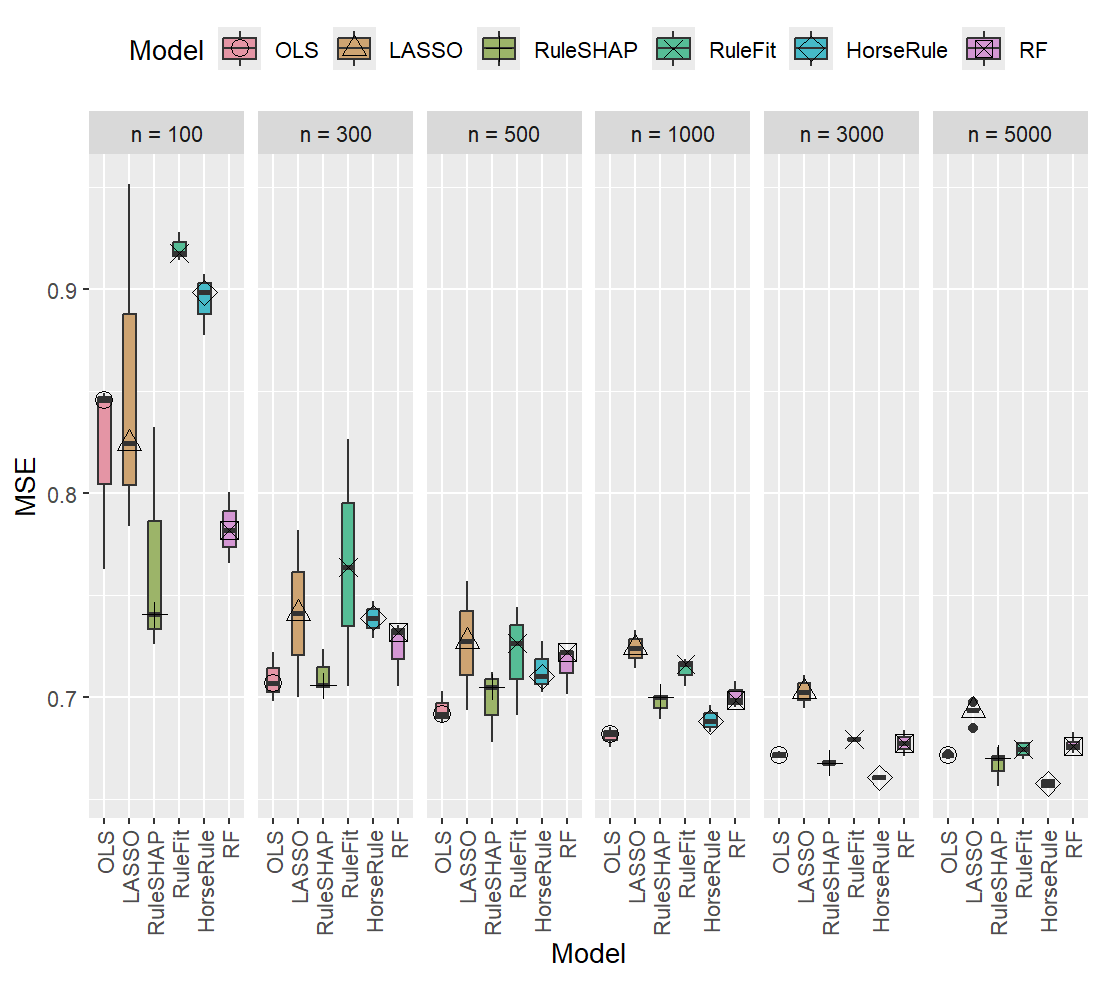}
    \caption{Test MSE across five fits of OLS linear regression, LASSO regression, RuleFit regression, HorseRule regression \iffalse(HR1 for default settings; HR2 for custom settings)\fi , RuleSHAP regression and Random Forest. These models are fitted on subsamples of different sizes $n$ taken from the HELIUS study data, with \textbf{systolic blood pressure} as predicted outcome.}
    \label{fig:HeliusMSEsbp}
\end{figure}

\begin{figure}[h]
    \centering
    \includegraphics[width=\linewidth]{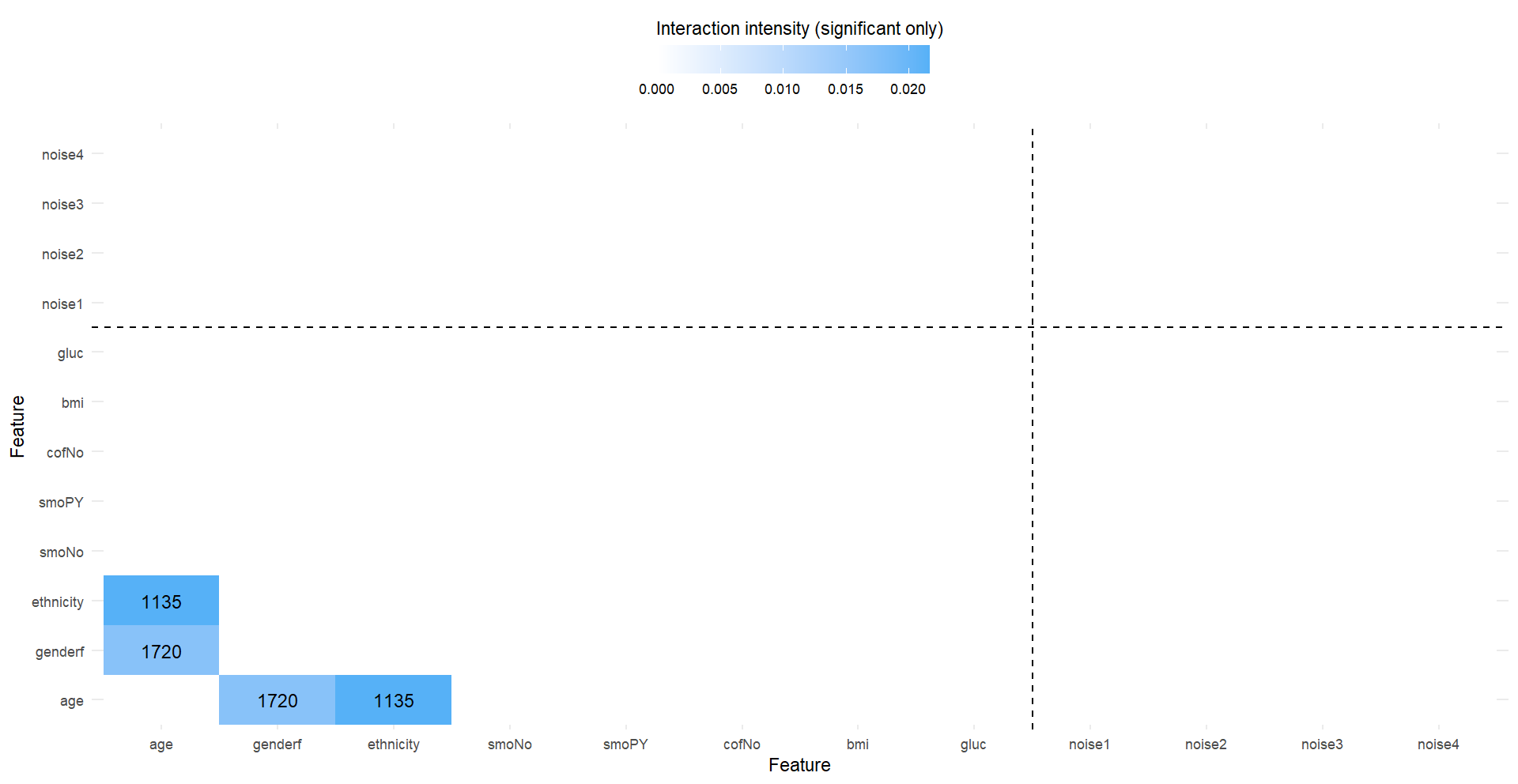}
    \caption{Interaction Shapley values computed from a RuleSHAP model fitted on $n=10'000$ observations from the HELIUS study, to predict \textbf{systolic blood pressure}. Point-wise 95\% credible intervals are used to determine significance of interactions, and only significant interactions are shown in the form of mean absolute values. Dashed line separates artificial noise features from the rest. Numbers in the tiles count for how many observations the interaction Shapley value is significant.}
    \label{fig:HeatmapSBP}
\end{figure}

\begin{figure}[h]
    \centering
    \includegraphics[width=\linewidth]{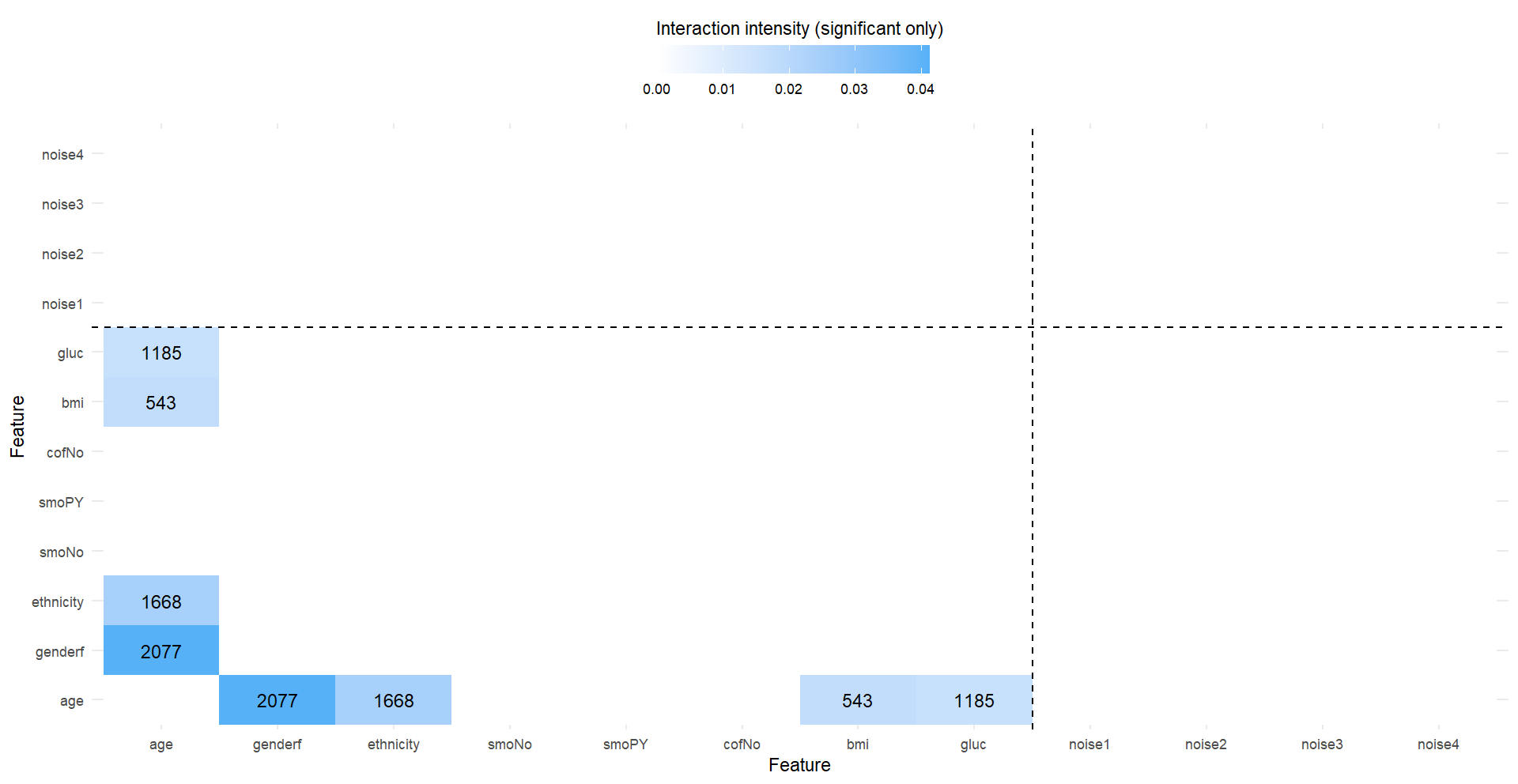}
    \caption{Interaction Shapley values computed from a RuleSHAP model fitted on $n=10'000$ observations from the HELIUS study, to predict \textbf{cholesterol}. Point-wise 95\% credible intervals are used to determine significance of interactions, and only significant interactions are shown in the form of mean absolute values. Dashed line separates artificial noise features from the rest. Numbers in the tiles count for how many observations the interaction Shapley value is significant.}
    \label{fig:HeatmapChol}
\end{figure}

\begin{figure}[h]
    \centering
    \includegraphics[width=\linewidth]{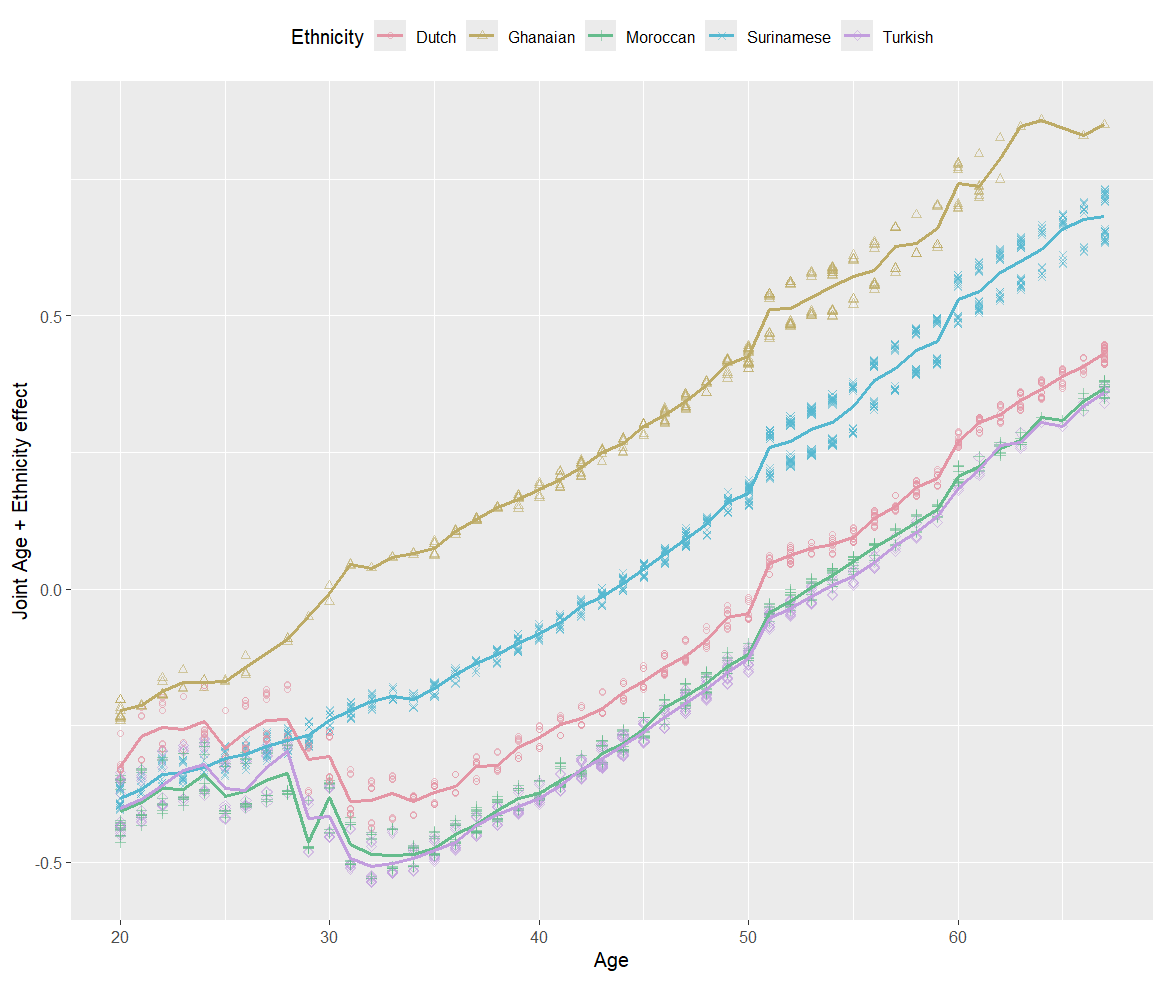}
    \caption{Marginal Shapley values computed from a RuleSHAP model fitted on $n=10'000$ observations from the HELIUS study, to predict systolic blood pressure, stratified by ethnicity. The marginal Shapley values of ethnicity and age are summed up together for each observation, so that their joint effect may be visualized. Lines show the overall trend of the mean joint contribution. The individual joint contributions, which (weakly) oscillate between observations due to the other interaction effects of age, are represented by dots.}
    \label{fig:EthnAgeInter}
\end{figure}

\begin{table}[h]
    \centering
    \begin{tabular}{c|cccc}
        Feature & Estimate & Std. Error & $t$ value & p-value \\\hline
(Intercept) & 0.442581 & 0.026095 & 16.960 & $< 0.001$ \\
age & 0.273671 & 0.009756 & 28.053 & $< 0.001$ \\
sex: female & -0.199646 & 0.008875 & -22.495 & $< 0.001$ \\
ethnicity: Moroccan & -0.605203 & 0.032872 & -18.411 & $< 0.001$ \\
ethnicity: Dutch & -0.442634 & 0.032807 & -13.492 & $< 0.001$ \\
ethnicity: Surinamese & -0.292533 & 0.029498 & -9.917 & $< 0.001$ \\
ethnicity: Turkish & -0.607234 & 0.033842 & -17.943 & $< 0.001$ \\
smoPY & -0.047704 & 0.008897 & -5.362 & $<0.001$ \\
coffee: No & 0.030347 & 0.012410 & 2.445 & 0.0145 \\
bmi & 0.281364 & 0.009284 & 30.308 & $< 0.001$ \\
gluc & 0.099396 & 0.009627 & 10.325 & $< 0.001$ \\
noise1 & 0.003980 & 0.008151 & 0.488 & 0.6254 \\
noise2 & -0.012675 & 0.008227 & -1.541 & 0.1234 \\ 
noise3 & -0.007220 & 0.008129 & -0.888 & 0.3744 \\ 
noise4 & 0.010030 & 0.008149 & 1.231 & 0.2184 
    \end{tabular}
    \caption{Coefficients and p-values from fitting linear regression to predict systolic blood pressure on the dataset from the HELIUS study.}
    \label{tab:sbpLM}
\end{table}

\begin{table}[h]
    \centering
    \begin{tabular}{c|cccc}
        Feature & Estimate & Std. Error & $t$ value & p-value \\\hline
(Intercept) &  -0.034152 &  0.030688 & -1.113 &  0.2658 \\
age & 0.289759 &  0.011473 & 25.257 & $< 0.001$ \\
sex: female & 0.003713 &  0.010437 &  0.356 &  0.7221 \\
ethnicity: Moroccan &  -0.190261 &  0.038657 & -4.922 & $< 0.001$ \\
ethnicity: Dutch &  0.181669 &  0.038581 &  4.709 & $< 0.001$ \\
ethnicity: Surinamese & 0.005175 &  0.034689 &  0.149 &  0.8814 \\
ethnicity: Turkish &  0.056655 &  0.039798 &  1.424 &  0.1546 \\
smoPY & 0.021367 &  0.010463 &  2.042 &  0.0412 \\
cofNo &  -0.026024 &  0.014594 & -1.783 &  0.0746 \\
bmi & 0.042939 &  0.010917 &  3.933 & $< 0.001$ \\
gluc & -0.091602 &  0.011321 & -8.091 & $< 0.001$ \\
noise1 & -0.001924 &  0.009586 & -0.201 &  0.8410 \\
noise2 &  0.002711 &  0.009675 &  0.280 &  0.7793 \\
noise3 & -0.002317 &  0.009559 & -0.242 &  0.8085 \\
noise4 &  0.015465 &  0.009583 &  1.614 &  0.1066
    \end{tabular}
    \caption{Coefficients and p-values from fitting linear regression to predict cholesterol level on the dataset from the HELIUS study.}
    \label{tab:cholLM}
\end{table}

\FloatBarrier

\subsection{Theoretical results}
\begin{lem}
    \textbf{(Variation of Vandermonde's identity)}For any $a,b,c \in \mathbb{N}$ such that $c \leq b$, the following equality holds:
    $$\sum_{u=0}^c\binom{a+u}{u}\binom{b-u}{c-u}=\binom{a+b+1}{c}.$$
    \label{lem:binomialsum}
\end{lem}
\begin{proof}    
We firstly produce the following intermediate formula with a combinatorial argument:
$$\binom{a+b}{c}=\sum_{u=0}^{c} \binom{a+u}{u}\binom{b-u}{c-u}-\sum_{u=0}^{c-1} \binom{a+u}{u}\binom{b-u-1}{c-u-1}.$$
Notice that the result holds straightforwardly for values of $a=0$ or $b=0$ or $b=c$. For the remaining cases, we rely on two ways to compute the number of subsets. Consider the set $G=\{1,\ldots,a+b\}$. Then the number of all possible subsets $S \subseteq G$ of size $c$ is precisely $\binom{a+b}{c}$. On the other hand, consider the following disjoint partitions of $G$ into $G^{(u)}_1$ and $G^{(u)}_2$ for any index $u \in \{0,\ldots,c\}$:
\begin{itemize}[label={}]
    \item $G^{(0)}_1=\{1,\ldots,a\}, \qquad G^{(0)}_2=\{a+1,\ldots,a+b\}$,
    \item $G^{(1)}_1=\{1,\ldots,a+1\}, \qquad G^{(1)}_2=\{a+2,\ldots,a+b\}$,
    \item $\vdots$
    \item $G^{(u)}_1=\{1,\ldots,a+u\}, \qquad G^{(u)}_2=\{a+u+1,\ldots,a+b\}$,
    \item $\vdots$
    \item $G^{(c)}_1=\{1,\ldots,a+c\}, \qquad G^{(c)}_2=\{a+c+1,\ldots,a+b\}$.
\end{itemize}

A simple induction on $b$ can show that any subset of $G$ of size $c$ can be represented as a set $S$ which shares $u$ elements with $G^{(u)}_1$ and $c-u$ elements with $G^{(u)}_2$, for at least one index $u$. For a given value of $u$, notice that there are:
$$\binom{|G^{(u)}_1|}{u} \cdot \binom{|G^{(u)}_2|}{c-u}=\binom{a+u}{u} \cdot \binom{b-u}{c-u}$$
such sets. One can therefore sum over all indices $u$ and count the number of subsets $S$ that take $u$ elements from $G^{(u)}_1$ and $c-u$ elements with $G^{(u)}_2$. This count would include all of the $\binom{a+b}{c}$ subsets $S \subseteq G$.\\

However, this could count the same subset multiple times, so for the summation to be equal to $\binom{a+b}{c}$ one needs to subtract the subsets counted multiple times. More specifically notice that, due to the sequential nature of the partitions, a set $S$ being decomposed in the aforementioned form also takes a sequential nature: if $S$ may be written as $u'$ elements from $G^{(u')}_1$ and $c-u'$ elements from $G^{(u')}_2$ but also as $u''$ elements from $G^{(u'')}_1$ and $c-u''$ elements from $G^{(l'')}_2$, then this also holds for any other value of $u$ between $u'$ and $u''$. In particular, this means that for any fixed $u$ one may simply count all $\binom{a+u}{u} \cdot \binom{b-u}{c-u}$ subsets and then subtract the ones that will also be counted for the value $u+1$. This way, no multiple copies are included.\\

For a fixed $u$, the sets that will also be counted again at the next iteration with $u+1$, and that therefore need to be subtracted from the count, are the ones that have exactly $u$ elements in $G_1^{(u)}$ but then will also have exactly $u+1$ elements in $G_1^{(u+1)}$. This happens (only) when the element \enquote{$a+u+1$} is in the subset $S$. Therefore, such type of double-counted subset contains:
\begin{itemize}[label={-}]
    \item the element \enquote{$a+u+1$}
    \item any $u$ elements in $G_1^{(u)}$
    \item any $c-u-1$ elements in $G_2^{(u+1)}$
\end{itemize}
Figure~\ref{fig:combisumEx} shows an example to clarify.\\

\begin{figure}
    \centering
    \begin{tikzpicture}
    \foreach \i in {1,...,5}
{
 \filldraw[color=mygrey, fill=myred, very thick] (\i-1,0) rectangle (\i,1);
}
    \foreach \i in {6,...,13}
{
 \filldraw[color=mygrey, fill=myblue, very thick] (\i-1,0) rectangle (\i,1);
}

\node[below] at (9,0) {$\underbrace{\hspace{7.9cm}}$};
\node[below] at (9,-0.3) {Choose 4};
\node[below] at (2.5,0) {$\underbrace{\hspace{4.9cm}}$};
\node[below] at (2.5,-0.3) {Choose 0};

\node[above] at (9,1) {\Large $\textcolor{myblue}{\,\,\, G_2^{(0)}}$};
\node[above] at (2.5,1) {\Large $\textcolor{myred}{\,\,\,G_1^{(0)}}$};

\node[left] at (-1,0.5) {\huge $\substack{u=0}$};

\node[below] at (0,-1.5) {$\,$};
\end{tikzpicture}

    \begin{tikzpicture}
    \foreach \i in {1,...,6}
{
 \filldraw[color=mygrey, fill=myred, very thick] (\i-1,0) rectangle (\i,1);
}
    \foreach \i in {7,...,13}
{
 \filldraw[color=mygrey, fill=myblue, very thick] (\i-1,0) rectangle (\i,1);
}

\node[below] at (9.5,0) {$\underbrace{\hspace{6.9cm}}$};
\node[below] at (9.5,-0.3) {Choose 3};
\node[below] at (3,0) {$\underbrace{\hspace{5.9cm}}$};
\node[below] at (3,-0.3) {Choose 1};

\node[above] at (9.5,1) {\Large $\textcolor{myblue}{\,\,\, G_2^{(1)}}$};
\node[above] at (3,1) {\Large $\textcolor{myred}{\,\,\,G_1^{(1)}}$};

\node[left] at (-1,0.5) {\huge $\substack{u=1}$};

\node[below] at (0,-1.5) {$\,$};
\end{tikzpicture}
    \begin{tikzpicture}
    \foreach \i in {1,...,5}
{
 \filldraw[color=mygrey, fill=myred, very thick] (\i-1,0) rectangle (\i,1);
}
\filldraw[color=mygrey, fill=mypurp, very thick] (5,0) rectangle (6,1);
    \foreach \i in {7,...,13}
{
 \filldraw[color=mygrey, fill=myblue, very thick] (\i-1,0) rectangle (\i,1);
}

\node[below] at (9.5,0) {$\underbrace{\hspace{6.9cm}}$};
\node[below] at (9.5,-0.3) {Choose 3};
\node[below] at (5.5,0) {$\underbrace{\hspace{0.9cm}}$};
\node[below] at (5.5,-0.3) {Choose};
\node[below] at (2.5,0) {$\underbrace{\hspace{4.9cm}}$};
\node[below] at (2.5,-0.3) {Choose 0};

\node[above] at (9,1) {\Large $\textcolor{myblue}{\,\,\, G_2^{(1)}}$};
\node[above] at (2.5,1) {\Large $\textcolor{myred}{\,\,\,G_1^{(0)}}$};

\node[left] at (-1,0.5) {\huge $\substack{u=0\\\text{and}\\u=1}$};
\end{tikzpicture}

    \caption{Example of subsets being counted twice for $a=5,b=8,c=4$ and $u\in\{0,1\}$. A subset here is counted for both $u=0$ and $u=1$ if and only if it has no elements in $G_1^{(0)}$ and one element in $G_1^{(1)}$, four elements in $G_2^{(0)}$ and three elements in $G_2^{(1)}$. This happens when the purple element is in the set, then no elements are on its left and $3=c-1$ elements are on the right.}
    \label{fig:combisumEx}
\end{figure}
This means that for every index $u$ there are $\binom{a+u}{u}\binom{b-u-1}{c-u-1}$ subsets that need to be removed from the count because they will also be counted at the next index, i.e. for $u+1$. This quantity is not subtracted at the last index, i.e. for $u=c$, as there will be no next index to double-count the subsets. This means that we can write:
\begin{equation}
    \label{eq:intermediateComb}
    \binom{a+b}{c}=\sum_{u=0}^{c} \binom{a+u}{u}\binom{b-u}{c-u}-\sum_{u=0}^{c-1} \binom{a+u}{u}\binom{b-u-1}{c-u-1}
\end{equation}
Now that we have this formula, we can prove the Lemma by induction on $c$. For $c=0$, we have:
$$\sum_{u=0}^0\binom{a+u}{u}\binom{b-u}{c-u}=\binom{a}{0}\binom{b}{0}=1=\binom{a+b+1}{0}=\binom{a+b+1}{c}$$
So equality holds. For the induction step, we isolate the first summation in Equation \ref{eq:intermediateComb}:
$$\sum_{u=0}^{c} \binom{a+u}{u}\binom{b-u}{c-u}=\binom{a+b}{c}+\sum_{u=0}^{c-1} \binom{a+u}{u}\binom{b-u-1}{c-u-1}=\binom{a+b}{c}+\binom{a+b}{c-1}$$
The last equality holds by induction, since $\sum_{u=0}^{c-1} \binom{a+u}{u}\binom{b-u-1}{c-u-1}$ is exactly the summation we are inductively calculating, but with $b-1$ instead of $b$ and with $c-1$ instead of $c$.\\
Using Pascal's rule, we conclude that $\sum_{u=0}^{c}\binom{a+u}{u}\binom{b-u}{c-u}=\binom{a+b+1}{c}$
\end{proof}

\begin{lem}
\label{lem:uselessfeatures}
    Consider a function $F: \mathbb{R}^p \rightarrow \mathbb{R}$, and take $q < p$ such that $F(x_1,\ldots,x_p)$ only depends on $q$ of the $p$ total features, say $x_{j_1},\ldots,x_{j_q}$. Then the marginal Shapley values for $F$ may be computed by only focusing on $x_{j_1},\ldots,x_{j_q}$: if $j \in \{j_1,\ldots,j_q\}$, then:
    $$\phi_j(x^*)=\sum_{S \subseteq \{j_1,\ldots,j_q\}\setminus \{j\}}\frac{1}{q\binom{q-1}{|S|}}\Big(\mathbb{E}[F(x_1,\ldots,x_p)|\text{do}(x_j=x_j^*,x_S=x_S^*)] -\mathbb{E}[F(x_1,\ldots,x_p)|\text{do}(x_S=x_S^*)]\Big).$$
    If $j \notin \{j_1,\ldots,j_q\}$, then $\phi_j(x^*)=0$.
\end{lem}
\begin{proof}
Without loss of generality, we may assume that $j_1=1,\ldots,j_q=q$. Furthermore, it suffices to prove this Lemma for $q=p-1$. If we then iterate the argument multiple times and remove all non-contributing features one by one, the argument is generally proven.\\
Let us also use the notation:
    $$\Delta\mathbb{E}_S:=\mathbb{E}[F(x_1,\dots,x_p)|\text{do}(x_j=x_j^*,x_S=x_S^*)]-\mathbb{E}[F(x_1,\dots,x_p)|\text{do}(x_S=x_S^*)].$$
Note that $\phi_p(x^*)=0$ is trivial, since the independence of $F$ from $x_p$ means that $\Delta \mathbb{E}_S=0 \forall S$.\\
For the remaining cases, let us assume that we need to compute the Shapley value for the first feature, for simplicity.\\
Now we can write:
\begin{align*}
    \phi_1(x^*) &=\sum_{S \subseteq \{2,\ldots,p\}}\frac{1}{p\binom{p-1}{|S|}}\Delta\mathbb{E}_S\\
    &=\sum_{\substack{S \subseteq \{2,\ldots,p\}\\ S \not\ni p}}\frac{1}{p\binom{p-1}{|S|}}\Delta\mathbb{E}_S+\sum_{\substack{S \subseteq \{2,\ldots,p\}\\ S \ni p}}\frac{1}{p\binom{p-1}{|S|}}\Delta\mathbb{E}_S
\end{align*}
    Now let us write the subsets $S \subseteq \{2,\ldots,p\}$ containing $p$ as $S=Z \cup \{p\}$. For these subsets, $|S|=1+|Z|$. Furthermore, since $F$ does not depend on $x_p$, we know that $\Delta\mathbb{E}_S=\Delta\mathbb{E}_Z$.
    $$\phi_1(x^*)=\sum_{S \subseteq \{2,\ldots,p-1\}}\frac{1}{p\binom{p-1}{|S|}}\Delta\mathbb{E}_S+\sum_{Z \subseteq \{2,\ldots,p-1\}}\frac{1}{p\binom{p-1}{1+|Z|}}\Delta\mathbb{E}_Z.$$
    Now both summations are over the same subsets, so we can join them:
    \begin{align*}
        \phi_j(x^*)&=\sum_{S \subseteq \{2,\ldots,p-1\}}\frac{1}{p}\Big(\frac{1}{\binom{p-1}{|S|}}+\frac{1}{\binom{p-1}{1+|S|}}\Big)\Delta\mathbb{E}_S\\
        &=\sum_{S \subseteq \{2,\ldots,p-1\}}\frac{1}{p!}\Big(|S|!(p-1-|S|)!+(|S|+1)!(p-2-|S|)!\Big)\Delta\mathbb{E}_S\\
        &=\sum_{S \subseteq \{2,\ldots,p-1\}}\frac{|S|!(p-2-|S|)!}{p!}\Big((p-1-|S|)+(|S|+1)\Big)\Delta\mathbb{E}_S\\
        &=\sum_{S \subseteq \{2,\ldots,p-1\}}\frac{|S|!(p-2-|S|)! \cdot p}{p!}\Delta\mathbb{E}_S\\
        &=\sum_{S \subseteq \{2,\ldots,p-1\}}\frac{|S|!(p-2-|S|)!}{(p-1)!}\Delta\mathbb{E}_S\\
        &=\sum_{S \subseteq \{2,\ldots,p-1\}}\frac{1}{(p-1)\binom{p-2}{|S|}}\Delta\mathbb{E}_S.
    \end{align*}
    This concludes the proof, as the last step is precisely the formula for Shapley values for the features $x_1,\ldots,x_{p-1}$.
\end{proof}

\begin{thm*}
Assume to have a dataset $\mathcal{T}$ of size $n$. Consider a 0-1 coded rule decomposed as the product of single conditions and thus of the form  $r(x_1,\ldots,x_p)=\prod_{k=1}^pR_k(x_k)$, with ${R_k:\mathbb{R} \rightarrow \{0,1\}}$. Consider any real-valued function ${\varphi: \mathbb{R} \rightarrow \mathbb{R}}$. Given $j,\ell \in \{1,\ldots,p\}$ and a datapoint $x^*$, the contribution of the $j$-th feature to the prediction $\varphi(x_\ell^*)r(x^*)$ as defined by marginal Shapley values is unbiasedly estimated by:
$$\widehat{\phi}_j(x^*)=\frac{1}{n}\sum_{\substack{t \in \mathcal{T} \text{s.t.}\\ R_k(t_k)=1 \vee R_k(x_k^*)=1 \,\forall k}}\frac{\big(R_j(x_j^*)-R_j(t_j)\big)\cdot \Big(\frac{\varphi(x_\ell^*)R_\ell(x_\ell^*)}{p-1-q(x^*)+R_\ell(x^*_\ell)+R_j(x^*_j)}+\frac{\varphi(t_\ell)R_l(t_\ell)}{p-1-q(t)+R_\ell(t_\ell)+R_j(t_j)}\Big)}{\binom{2p-q(x^*)-q(t)-2+R_\ell(x^*_\ell)+R_j(x^*_j)+R_\ell(t_\ell)+R_j(t_j)}{p-1-q(x^*)+R_\ell(x^*_\ell)+R_j(x^*_j)}}$$
if $j \neq \ell$ and by:
$$\widehat{\phi}_j(x^*)=\frac{1}{n(p-q(x^*)+R_j(x^*_j))}\sum_{\substack{t \in \mathcal{T} \text{s.t.}\\ R_k(t_k)=1 \vee R_k(x^*_k)=1 \,\forall k}}\frac{\varphi(x_j^*)R_j(x_j^*)-\varphi(t_j)R_j(t_j)}{\binom{2p-q(x^*)-q(t)-1+R_j(x^*_j)+R_j(t_j)}{p-q(x^*)+R_j(x^*_j)}}$$
if $j=\ell$, where $q: \mathbb{R} \rightarrow \mathbb{N}$ is defined as $q(x)=\sum_{j=1}^pR_j(x_j)$ and $\vee$ is the logical \enquote{or} operator.
\end{thm*}
\begin{proof}
For simplicity of notation, let us assume that $\ell=1$.\\

Let us first cover the case $j=\ell=1$ and write $S':=\{2,\ldots,p\}\setminus S$, for any subset $S \subseteq \{2,\ldots,p\}$. For any subset of active players $S$, we have:
\begin{align*}
    \mathbb{E}[\varphi(x_1)r(x)|\text{do}(x_S=x_S^*)]&=\mathbb{E}[\varphi(x_1)\prod_{k=1}^pR_k(x_k)|\text{do}(x_S=x_S^*)]\\
    &=\prod_{k\in S}R_k(x^*_k)\cdot\mathbb{E}[\varphi(x_1)\prod_{k\in \{1,\ldots,p\} \setminus S}R_k(x_k)|\text{do}(x_S=x_S^*)]\\
    &=\prod_{k\in S}R_k(x^*_k)\cdot\mathbb{E}[\varphi(x_1)\prod_{k\in \{1,\ldots,p\} \setminus S}R_k(x_k)].
\end{align*}
Analogously, we have:
$$\mathbb{E}[\varphi(x_1)r(x)|\text{do}(x_S=x_S^*,x_1=x_1^*)]=\varphi(x_1^*)\prod_{k\in S\cup\{1\}}R_k(x^*_k)\cdot \mathbb{E}[\prod_{k\in S'}R_k(x_k)]$$
We write these products more compactly by grouping up the indices:
define $R_S(x_S):=\prod_{k\in S}R_k(x_k)$ and, consequently, $R_{S'}(x_{S'}):=\prod_{k\in S'}R_k(x_k)$. Then $\phi_1(x^*)$ may be re-written as:
\begin{align*}
\phi_1(x^*)&=\sum_{S \subseteq \{2,\ldots,p\}}w(S)\Big(\mathbb{E}[\varphi(x_1)r(x)|\text{do}(x_S=x_S^*,x_1=x_1^*)]-\mathbb{E}[\varphi(x_1)r(x)|\text{do}(x_S=x_S^*)]\Big)\\
& \!\begin{multlined}[t]
     = \sum_{S \subseteq \{2,\ldots,p\}}w(S)\mathbb{E}[\varphi(x_1)r(x)|\text{do}(x_S=x_S^*,x_1=x_1^*)]\\
     -\sum_{S \subseteq \{2,\ldots,p\}}w(S)\mathbb{E}[\varphi(x_1)r(x)|\text{do}(x_S=x_S^*)]
     \end{multlined}\\
& \!\begin{multlined}[t]
     = \sum_{S \subseteq \{2,\ldots,p\}}w(S) \varphi(x_1^*)R_1(x_1^*)R_S(x_S^*)\mathbb{E}[R_{S'}(x_{S'})]\\
     -\sum_{S \subseteq \{2,\ldots,p\}}w(S) R_S(x_S^*)\mathbb{E}[\varphi(x_1)R_{S'\cup\{1\}}(x_{S'\cup\{1\}})].
     \end{multlined}
\end{align*}
Let us treat the two summations separately; define the following:
$$\mathcal{A}=\sum_{S \subseteq \{2,\ldots,p\}}w(S) \varphi(x_1^*)R_1(x_1^*)R_S(x_S^*)\mathbb{E}[R_{S'}(x_{S'})],$$
$$\mathcal{B}=\sum_{S \subseteq \{2,\ldots,p\}}w(S) R_S(x_S^*)\mathbb{E}[\varphi(x_1)R_{S'\cup\{1\}}(x_{S'\cup\{1\}})].$$
Note that the expectations that appear in these formulas may be unbiasedly estimated by their sample means over the set $\mathcal{T}$. Let us focus on the first summation $\mathcal{A}$, which is therefore approximated by:
\begin{align*}
\widehat{\mathcal{A}}&=\sum_{S \not \ni 1}w(S)\varphi(x_1^*) R_1(x_1^*)R_S(x_S^*)\widehat{\mathbb{E}}[R_{S'}(x_{S'})]\\
&=\sum_{S \not \ni 1}w(S) \varphi(x_1^*) R_1(x_1^*)R_S(x_S^*)\frac{1}{n}\sum_{t \in \mathcal{T}}R_{S'}(t_{S'})\\
&=\varphi(x_1^*) R_1(x_1^*) \cdot \frac{1}{n}\sum_{t \in \mathcal{T}}\sum_{S \not \ni 1}w(S) R_S(x_S^*)R_{S'}(t_{S'}).
\end{align*}
Now let us focus on $\mathcal{C}:=\sum_{S \not \ni 1}w(S) R_{S}(x_{S}^*)R_{S'}(t_{S'})$, for a fixed datapoint $t$. For any datapoint x, define the following:
$$\Omega_1(x):=\{k \in \{2,\ldots,p\} \big| R_k(x_k)=1\}, \qquad \qquad q_1(x):=|\Omega_1(x)|=\sum_{k=2}^pR_k(x_k)=q(x)-R_1(x_1).$$
By definition of $\Omega_1$, we have:
\begin{align*}
    R_S(x_S^*)R_{S'}(t_{S'})\neq 0 &\iff \begin{cases}
    R_S(x_S^*)=1\\
    R_{S'}(t_{S'})=1
\end{cases} \\
&\iff
\begin{cases}
    S \subseteq \Omega_1(x^*)\\
    S' \subseteq \Omega_1(t)
\end{cases}\\
&\iff
\begin{cases}
    S \subseteq \Omega_1(x^*)\\
    \Omega_1(t)' \subseteq S
\end{cases},
\end{align*}
where $\Omega_1(t)'$ is meant as the complementary set of $\Omega_1(t)$ with respect to $\{2,\ldots,p\}$. This means that $w(S) R_{S}(x_{S}^*)R_{S'}(z_{S'})$ only gives a non-zero effect for the datapoints $t$ such that $\Omega_1(t)' \subseteq \Omega_1(x^*)$, in which case the (only) subsets $S$ that contribute are the ones such that $\Omega_1(t)' \subseteq S \subseteq \Omega_1(x^*)$. Since all such sets $S$ contain $\Omega_1(t)'$, they can all be uniquely identified by the indices that they have \textit{besides} those in $\Omega_1(t)'$. In other words, each $S$ may be (uniquely) re-written as $S=\Omega_1(t)' \cup Z$, with $Z \subseteq \Omega_1(x^*) \setminus \Omega_1(t)'$. For every size $|Z|=m$, there are exactly $\binom{|\Omega_1(x^*)\setminus\Omega_1(t)'|}{m}=\binom{q_1(x^*)+q_1(t)-p+1}{m}$ possible choices of $Z$, and they all have the same contribution $w(S)=\frac{1}{p\binom{p-1}{|S|}}=\frac{1}{p\binom{p-1}{p-1-q_1(t)+l}}$.\\
This means that we can write:
\begin{align*}
\mathcal{C}&=\sum_{S \not \ni 1}w(S) R_S(x_S^*)R_{S'}(t_{S'})\\
&=\sum_{m=0}^{|\Omega_1(x^*)\setminus\Omega_1(t)|}\frac{\binom{q_1(x^*)+q_1(t)-p+1}{m}}{p\binom{p-1}{m+p-1-q_1(t)}}\\
&=\sum_{m=0}^{q_1(x^*)+q_1(t)-p+1}\frac{\binom{q_1(x^*)+q_1(t)-p+1}{m}}{p\binom{p-1}{m+p-1-q_1(t)}}\\
& \!\begin{multlined}[t]
     =\frac{(q_1(x^*)+q_1(t)-p+1)!}{p!}\\
     \cdot \sum_{m=0}^{q_1(x^*)+q_1(t)-p+1}\frac{(q_1(t)-m)!(p-1-q_1(t)+m)!}{(q_1(x^*)+q_1(t)-p+1-m)!m!}
     \end{multlined}\\
& \!\begin{multlined}[t]
     =\frac{(q_1(x^*)+q_1(t)-p+1)!}{p!}\\
     \cdot \sum_{m=0}^{q_1(x^*)+q_1(t)-p+1}\Bigg[\binom{q_1(t)-m}{q_1(x^*)+q_1(t)-p+1-m}(p-1-q_1(x^*))!\\
     \cdot \binom{p-1-q_1(t)+m}{m}(p-1-q_1(t))!\Bigg]
     \end{multlined}\\
& \!\begin{multlined}[t]
     =\frac{(q_1(x^*)+q_1(t)-p+1)!(p-1-q_1(t))!(p-1-q_1(x^*))!}{p!}\\
     \cdot \sum_{m=0}^{q_1(x^*)+q_1(t)-p+1}\binom{q_1(t)-m}{q_1(x^*)+q_1(t)-p+1-m}\binom{p-1-q_1(t)+m}{m}.
     \end{multlined}
\end{align*}

Using Lemma~\ref{lem:binomialsum} with $a=p-1-q_1(t),b=q_1(t),c=q_1(x^*)+q_1(t)-p+1$, we conclude:

\begin{align*}
\mathcal{C}&=\sum_{S \not \ni 1}w(S) R_S(x_S^*)R_{S'}(t_{S'})\\
& \!\begin{multlined}[t]
     =\frac{(q_1(x^*)+q_1(t)-p+1)!(p-1-q_1(t))!(p-1-q_1(x^*))!}{p!}\\
     \cdot \binom{p}{q_1(x^*)+q_1(t)-p+1}
     \end{multlined}\\
&=\frac{(p-1-q_1(t))!(p-1-q_1(x^*))!}{(2p-q_1(x^*)-q_1(t)-1)!}\\
&=\frac{1}{(p-q_1(x^*))}\cdot \frac{1}{\binom{2p-q_1(x^*)-q_1(t)-1}{p-q_1(x^*)}}.
\end{align*}

This allows us to conclude that:
$$\widehat{\mathcal{A}}=\varphi(x_1^*)R_1(x_1^*) \cdot \frac{1}{n(p-q_1(x^*))}\sum_{\substack{t \in \mathcal{T} \text{ s.t.}\\ \Omega_1(t)' \subseteq \Omega_1(x^*)}}\frac{1}{\binom{2p-q_1(x^*)-q_1(t)-1}{p-q_1(x^*)}}.$$
The sum is only over the datapoints with $\Omega_1(t)' \subseteq \Omega_1(x^*)$, as the other datapoints have a null effect. In a similar fashion, let us compute $\widehat{\mathcal{B}}$:
\begin{align*}
    \widehat{\mathcal{B}}&=\sum_{S \not \ni 1}w(S) R_S(x_S^*)\mathbb{E}[\varphi(x_1)R_{S'\cup\{p\}}(x_{S'\cup\{p\}})]\\
    &=\sum_{S \not \ni 1}w(S) R_S(x_S^*)\frac{1}{n}\sum_{t \in \mathcal{T}}\varphi(t_1)R_{S'\cup\{p\}}(t_{S'\cup\{p\}})\\
    &=\sum_{S \not \ni 1}w(S) R_S(x_S^*)\frac{1}{n}\sum_{t \in \mathcal{T}}R_S(x_S^*)\varphi(t_1)R_{S'}(t_{S'})R_1(t_1)\\
    &=\frac{1}{n}\sum_{t \in \mathcal{T}}\varphi(t_1)R_1(t_1)\sum_{S \not \ni 1}w(S) R_S(x_S^*)R_{S'}(t_{S'}).
\end{align*}
As we can see, the expression now contains the exact same summation as before, which immediately allows us to conclude that:
$$\widehat{\mathcal{B}}=\frac{1}{n(p-q_1(x^*))}\sum_{\substack{t \in \mathcal{T} \text{ s.t.}\\ \Omega_1(t)' \subseteq \Omega_1(x^*)}}\frac{\varphi(t_1)R_1(t_1)}{\binom{2p-q_1(x^*)-q_1(t)-1}{p-q_1(x^*)}}.$$
Combining the two expressions together gives us the formula:
$$\widehat{\phi_1}(x^*)=\widehat{\mathcal{A}}-\widehat{\mathcal{B}}=\frac{1}{n(p-q_1(x^*))}\sum_{\substack{t \in \mathcal{T} \text{ s.t.}\\ \Omega_1(t)' \subseteq \Omega_1(x^*)}}\frac{\varphi(x_1^*)R_1(x_1^*)-\varphi(t_1)R_1(t_1)}{\binom{2p-q_1(x^*)-q_1(t)-1}{p-q_1(x^*)}}.$$
To obtain the formula stated in the Theorem, notice that the condition $\Omega_1(t)' \subseteq \Omega_1(x^*)$ in the summation may be replaced with the condition $R_k(t_k)=1 \vee R_k(x_k^*)=1 \,\forall k$: the only datapoints for which the two conditions differ are the points for which $R_1(t_1)=R_1(x^*_1)=0$. For such datapoints, the summand is null anyway and therefore does not affect the formula.\\

Now let us consider the case $j \neq \ell = 1$. For simplicity, let's assume $j=p$. For any subset $V \subseteq \{2,\ldots,p-1\}$, define $V''$ as $\{2,\ldots,p-1\} \setminus V$. By separating the cases $1 \notin S$ from the cases $1 \in S$ (which therefore is of the form $\{1\} \cup V$ for $V \not\ni 1$), Shapley values may be re-written as:
\begin{align*}
\hspace*{-4cm}\phi_p(x^*)&=\sum_{S \subseteq \{1,\ldots,p-1\}}w(S)\Big(\mathbb{E}[\varphi(x_1)r(x)|\text{do}(x_S=x_S^*,x_p=x_p^*)]-\mathbb{E}[\varphi(x_1)r(x)|\text{do}(x_S=x_S^*)]\Big)\\
& = \sum_{V \subseteq \{2,\ldots,p-1\}}w(V\cup\{1\})\Big(\mathbb{E}[\varphi(x_1)r(x)|\text{do}(x_V=x_V^*,x_1=x_1^*,x_p=x_p^*)]-\mathbb{E}[\varphi(x_1)r(x)|\text{do}(x_V=x_V^*,x_1=x_1^*)]\Big)\\
     &\qquad+\sum_{V \subseteq \{2,\ldots,p-1\}}w(V)\Big(\mathbb{E}[\varphi(x_1)r(x)|\text{do}(x_V=x_V^*,x_p=x_p^*)]-\mathbb{E}[\varphi(x_1)r(x)|\text{do}(x_V=x_V^*)]\Big) \\
& \!\begin{multlined}[t]
     = \sum_{V \subseteq \{2,\ldots,p-1\}}w(V\cup\{1\})\mathbb{E}[\varphi(x_1)r(x)|\text{do}(x_V=x_V^*,x_1=x_1^*,x_p=x_p^*)]\\
     - \sum_{V \subseteq \{2,\ldots,p-1\}}w(V\cup\{1\})\mathbb{E}[\varphi(x_1)r(x)|\text{do}(x_V=x_V^*,x_1=x_1^*)]\\
     + \sum_{V \subseteq \{2,\ldots,p-1\}}w(V)\mathbb{E}[\varphi(x_1)r(x)|\text{do}(x_V=x_V^*,x_p=x_p^*)]\\
     - \sum_{V \subseteq \{2,\ldots,p-1\}}w(V)\mathbb{E}[\varphi(x_1)r(x)|\text{do}(x_V=x_V^*)]\\
     \end{multlined}\\
& \!\begin{multlined}[t]
     = \sum_{V \subseteq \{2,\ldots,p-1\}}w(V\cup\{1\})\varphi(x_1^*)R_1(x_1^*)R_V(x_V^*)R_p(x_p^*)\mathbb{E}[R_{V''}(x_{V''})] \\
     \qquad - \sum_{V \subseteq \{2,\ldots,p-1\}}w(V\cup\{1\})\varphi(x_1^*)R_1(x_1^*)R_V(x_V^*)\mathbb{E}[R_{V''\cup\{p\}}(x_{V''\cup\{p\}})]\\
     + \sum_{V \subseteq \{2,\ldots,p-1\}}w(V)R_V(x_V^*)R_p(x_p^*)\mathbb{E}[\varphi(x_1)R_{V''\cup\{1\}}(x_{V''\cup\{1\}})] \\
     - \sum_{V \subseteq \{2,\ldots,p-1\}}w(V)R_V(x_V^*)\mathbb{E}[\varphi(x_1)R_{V''\cup\{1,p\}}(x_{V''\cup\{1,p\}})]\\
     \end{multlined}\\
\end{align*}
Let us treat the four summations separately; define the following:
$$\mathcal{D}=\sum_{V \subseteq \{2,\ldots,p-1\}}w(V\cup\{1\})\varphi(x_1^*)R_1(x_1^*)R_V(x_V^*)R_p(x_p^*)\mathbb{E}[R_{V''}(x_{V''})],$$
$$\mathcal{E}=\sum_{V \subseteq \{2,\ldots,p-1\}}w(V\cup\{1\})\varphi(x_1^*)R_1(x_1^*)R_V(x_V^*)\mathbb{E}[R_{V''\cup\{p\}}(x_{V''\cup\{p\}})],$$
$$\mathcal{F}=\sum_{V \subseteq \{2,\ldots,p-1\}}w(V)R_V(x_V^*)R_p(x_p^*)\mathbb{E}[\varphi(x_1)R_{V''\cup\{1\}}(x_{V''\cup\{1\}})],$$
$$\mathcal{G}=\sum_{V \subseteq \{2,\ldots,p-1\}}w(V)R_V(x_V^*)\mathbb{E}[\varphi(x_1)R_{V''\cup\{1,p\}}(x_{V''\cup\{1,p\}})].$$

With a similar argument as above, they can be approximated without bias by:
\begin{align*}
    \widehat{\mathcal{D}} &= \sum_{V \subseteq \{2,\ldots,p-1\}}w(V\cup\{1\})\varphi(x_1^*)R_1(x_1^*)R_V(x_V^*)R_p(x_p^*)\frac{1}{n}\sum_{t \in \mathcal{T}}R_{V''}(t_{V''})\\
    &= \frac{1}{n}\sum_{t \in \mathcal{T}}x_1^*R_1(x_1^*)R_p(x_p^*)\sum_{V \subseteq \{2,\ldots,p-1\}}w(V\cup\{1\})R_V(x_V^*)R_{V''}(t_{V''})\\
\end{align*}
\begin{align*}
    \widehat{\mathcal{E}} &= \sum_{V \subseteq \{2,\ldots,p-1\}}w(V\cup\{1\})\varphi(x_1^*)R_1(x_1^*)R_V(x_V^*)\frac{1}{n}\sum_{t \in \mathcal{T}}R_{V''\cup\{p\}}(t_{V''\cup\{p\}})\\
    &= \frac{1}{n}\sum_{t \in \mathcal{T}}x_1^*R_1(x_1^*)R_p(t_p)\sum_{V \subseteq \{2,\ldots,p-1\}}w(V\cup\{1\})R_V(x_V^*)R_{V''}(t_{V''})\\
\end{align*}
\begin{align*}
    \widehat{\mathcal{F}} &= \sum_{V \subseteq \{2,\ldots,p-1\}}w(V)R_V(x_V^*)R_p(x_p^*)\frac{1}{n}\sum_{t \in \mathcal{T}}\varphi(t_1)R_{V''\cup\{1\}}(t_{V''\cup\{1\}})\\
    &= \frac{1}{n}\sum_{t \in \mathcal{T}}\varphi(t_1)R_1(t_1)R_p(x_p^*)\sum_{V \subseteq \{2,\ldots,p-1\}}w(V)R_V(x_V^*)R_{V''}(t_{V''})\\
\end{align*}
\begin{align*}
    \widehat{\mathcal{G}} &= \sum_{V \subseteq \{2,\ldots,p-1\}}w(V)R_V(x_V^*)\frac{1}{n}\sum_{t \in \mathcal{T}}\varphi(t_1)R_{V''\cup\{1,p\}}(t_{V''\cup\{1,p\}})\\
    &= \frac{1}{n}\sum_{t \in \mathcal{T}}\varphi(t_1)R_1(t_1)R_p(t_p)\sum_{V \subseteq \{2,\ldots,p-1\}}w(V)R_V(x_V^*)R_{V''}(t_{V''})\\
\end{align*}

In order to compute these, we re-write the two summations over the sets $V$ that appear above: for any datapoint $x$, define the following:
$$\Omega_{1,p}(x):=\{j \in \{2,\ldots,p-1\} \big| R_j(x)=1\}, \qquad \qquad q_{1,p}(x)=|\Omega_{1,p}(x)|=\sum_{k=2}^{p-1}R_k(x_k).$$
Then the terms $R_V(x_V^*)R_{V''}(t_{V''})$ are again non-zero exactly when $\Omega_{1,p}(x^*)$ contains $\Omega_{1,p}(t)''$, and only for the sets in between these two, so:
\begin{align*}
    \mathcal{H}&:=\sum_{V \subseteq \{2,\ldots,p-1\}}w(V\cup\{1\})R_V(x_V^*)R_{V''}(t_{V''})\\
    &=\sum_{m=0}^{|\Omega_{1,p}(x^*)\setminus\Omega_{1,p}(t)''|}\frac{\binom{q_{1,p}(x^*)+q_{1,p}(t)-p+2}{m}}{p\binom{p-1}{m+p-1-q_{1,p}(t)}}\\
&=\sum_{m=0}^{q_{1,p}(x^*)+q_{1,p}(t)-p+2}\frac{\binom{q_{1,p}(x^*)+q_{1,p}(t)-p+2}{m}}{p\binom{p-1}{m+p-1-q_{1,p}(t)}}\\
& \!\begin{multlined}[t]
     =\frac{(q_{1,p}(x^*)+q_{1,p}(t)-p+2)!}{p!}\\
     \cdot \sum_{m=0}^{q_{1,p}(x^*)+q_{1,p}(t)-p+2}\frac{(q_{1,p}(t)-m)!(p-1-q_{1,p}(t)+m)!}{(q_{1,p}(x^*)+q_{1,p}(t)-p+2-m)!m!}
     \end{multlined}\\
& \!\begin{multlined}[t]
     =\frac{(q_{1,p}(x^*)+q_{1,p}(t)-p+2)!}{p!}\\
     \cdot \sum_{m=0}^{q_{1,p}(x^*)+q_{1,p}(t)-p+2}\Bigg[\binom{q_{1,p}(t)-m}{q_{1,p}(x^*)+q_{1,p}(t)-p+2-m}(p-2-q_{1,p}(x^*))!\\
     \cdot \binom{p-1-q_{1,p}(t)+m}{m}(p-1-q_{1,p}(t))!\Bigg]
     \end{multlined}\\
& \!\begin{multlined}[t]
     =\frac{(q_{1,p}(x^*)+q_{1,p}(t)-p+2)!(p-1-q_{1,p}(t))!(p-2-q_{1,p}(x^*))!}{p!}\\
     \cdot \sum_{m=0}^{q_{1,p}(x^*)+q_{1,p}(t)-p+2}\binom{q_{1,p}(t)-m}{q_{1,p}(x^*)+q_{1,p}(t)-p+2-m}\binom{p-1-q_{1,p}(t)+m}{m}.
     \end{multlined}
\end{align*}

Using Lemma~\ref{lem:binomialsum} with $a=p-1-q_{1,p}(t),b=q_{1,p}(t)$ and $c=q_{1,p}(x^*)+q_{1,p}(t)-p+2$, we conclude:

\begin{align*}
    \mathcal{H}&=\sum_{V \subseteq \{2,\ldots,p-1\}}w(V\cup\{1\})R_V(x_V^*)R_{V''}(t_{V''})\\
& \!\begin{multlined}[t]
     =\frac{(q_{1,p}(x^*)+q_{1,p}(t)-p+2)!(p-1-q_{1,p}(t))!(p-2-q_{1,p}(x^*))!}{p!}\\
     \cdot \binom{p}{q_{1,p}(x^*)+q_{1,p}(t)-p+2}
     \end{multlined}\\
&=\frac{(p-1-q_{1,p}(t))!(p-2-q_{1,p}(x^*))!}{(2p-q_{1,p}(x^*)-q_{1,p}(t)-2)!}\\
&=\frac{1}{(p-1-q_{1,p}(x^*))}\cdot \frac{1}{\binom{2p-q_{1,p}(x^*)-q_{1,p}(t)-2}{p-1-q_{1,p}(x^*)}}.
\end{align*}
Similarly, we can write:
\begin{align*}
    \mathcal{I}&:=\sum_{V \subseteq \{2,\ldots,p-1\}}w(V)R_V(x_V^*)R_{V''}(t_{V''})\\
    &=\sum_{m=0}^{|\Omega_{1,p}(x^*)\setminus\Omega_{1,p}(t)|}\frac{\binom{q_{1,p}(x^*)+q_{1,p}(t)-p+2}{m}}{p\binom{p-1}{m+p-2-q_{1,p}(t)}}\\
&=\sum_{m=0}^{q_{1,p}(x^*)+q_{1,p}(t)-p+2}\frac{\binom{q_{1,p}(x^*)+q_{1,p}(t)-p+2}{m}}{p\binom{p-1}{m+p-2-q_{1,p}(t)}}\\
& \!\begin{multlined}[t]
     =\frac{(q_{1,p}(x^*)+q_{1,p}(t)-p+2)!}{p!}\\
     \cdot \sum_{m=0}^{q_{1,p}(x^*)+q_{1,p}(t)-p+2}\frac{(q_{1,p}(t)+1-m)!(p-2-q_{1,p}(t)+m)!}{(q_{1,p}(x^*)+q_{1,p}(t)-p+2-m)!m!}
     \end{multlined}\\
& \!\begin{multlined}[t]
     =\frac{(q_{1,p}(x^*)+q_{1,p}(t)-p+2)!}{p!}\\
     \cdot \sum_{m=0}^{q_{1,p}(x^*)+q_{1,p}(t)-p+2}\Bigg[\binom{q_{1,p}(t)+1-m}{q_{1,p}(x^*)+q_{1,p}(t)-p+2-m}(p-1-q_{1,p}(x^*))!\\
     \cdot \binom{p-2-q_{1,p}(t)+m}{m}(p-2-q_{1,p}(t))!\Bigg]
     \end{multlined}\\
& \!\begin{multlined}[t]
     =\frac{(q_{1,p}(x^*)+q_{1,p}(t)-p+2)!(p-2-q_{1,p}(t))!(p-1-q_{1,p}(x^*))!}{p!}\\
     \cdot \sum_{m=0}^{q_{1,p}(x^*)+q_{1,p}(t)-p+2}\binom{q_{1,p}(t)+1-m}{q_{1,p}(x^*)+q_{1,p}(t)-p+2-l}\binom{p-2-q_{1,p}(t)+mssssssssw}{m}.
     \end{multlined}
\end{align*}

Using Lemma~\ref{lem:binomialsum} with $a=p-2-q_{1,p}(t),b=q_{1,p}(t)+1$ and $c=q_{1,p}(x^*)+q_{1,p}(t)-p+2$, we conclude:

\begin{align*}
\mathcal{I}&=\sum_{V \subseteq \{2,\ldots,p-1\}}w(V)R_V(x_V^*)R_{V''}(t_{V''})\\
& \!\begin{multlined}[t]
     =\frac{(q_{1,p}(x^*)+q_{1,p}(t)-p+2)!(p-2-q_{1,p}(t))!(p-1-q_{1,p}(x^*))!}{p!}\\
     \cdot \binom{p}{q_{1,p}(x^*)+q_{1,p}(t)-p+2}
     \end{multlined}\\
&=\frac{(p-2-q_{1,p}(t))!(p-1-q_{1,p}(x^*))!}{(2p-q_{1,p}(x^*)-q_{1,p}(t)-2)!}\\
&=\frac{1}{(p-1-q_{1,p}(t))}\cdot \frac{1}{\binom{2p-q_{1,p}(x^*)-q_{1,p}(t)-2}{p-1-q_{1,p}(x^*)}}.
\end{align*}

These allow us to conclude that:
$$\widehat{\mathcal{D}}=\frac{1}{n}\sum_{\substack{t \in \mathcal{T} \text{s.t.}\\ \Omega_{1,p}(t)' \subseteq \Omega_{1,p}(x^*)}}\frac{1}{\binom{2p-q_{1,p}(x^*)-q_{1,p}(t)-2}{p-1-q_{1,p}(x^*)}}\cdot \frac{\varphi(x_1^*)R_1(x_1^*)R_p(x_p^*)}{p-1-q_{1,p}(x^*)}$$
$$\widehat{\mathcal{E}}=\frac{1}{n}\sum_{\substack{t \in \mathcal{T} \text{s.t.}\\ \Omega_{1,p}(t)' \subseteq \Omega_{1,p}(x^*)}}\frac{1}{\binom{2p-q_{1,p}(x^*)-q_{1,p}(t)-2}{p-1-q_{1,p}(x^*)}}\cdot \frac{\varphi(x_1^*)R_1(x_1^*)R_p(t_p)}{p-1-q_{1,p}(x^*)}$$
$$\widehat{\mathcal{F}}=\frac{1}{n}\sum_{\substack{t \in \mathcal{T} \text{s.t.}\\ \Omega_{1,p}(t)' \subseteq \Omega_{1,p}(x^*)}}\frac{1}{\binom{2p-q_{1,p}(x^*)-q_{1,p}(t)-2}{p-1-q_{1,p}(x^*)}}\cdot \frac{\varphi(t_1)R_1(t_1)R_p(x_p^*)}{p-1-q_{1,p}(t)}$$
$$\widehat{\mathcal{G}}=\frac{1}{n}\sum_{\substack{t \in \mathcal{T} \text{s.t.}\\ \Omega_{1,p}(t)' \subseteq \Omega_{1,p}(x^*)}}\frac{1}{\binom{2p-q_{1,p}(x^*)-q_{1,p}(t)-2}{p-1-q_{1,p}(x^*)}}\cdot \frac{\varphi(t_1)R_1(t_1)R_p(t_p)}{p-1-q_{1,p}(t)}$$

Putting the four pieces together, we obtain:
\begin{align*}
    \widehat{\phi_p}(x^*)&=\widehat{\mathcal{D}}-\widehat{\mathcal{E}}+\widehat{\mathcal{F}}-\widehat{\mathcal{G}}\\
    &=\frac{1}{n}\sum_{\substack{t \in \mathcal{T} \text{s.t.}\\ \Omega_{1,p}(t)' \subseteq \Omega_{1,p}(x^*)}}\frac{\varphi(x_1^*)R_1(x_1^*)\frac{R_p(x_p^*)-R_p(t_p)}{p-1-q_{1,p}(x^*)}+\varphi(t_1)R_1(t_1)\frac{R_p(x_p^*)-R_p(t_p)}{p-1-q_{1,p}(t)}}{\binom{2p-q_{1,p}(x^*)-q_{1,p}(t)-2}{p-1-q_{1,p}(x^*)}}.
\end{align*}

To obtain the formula stated in the Theorem, notice that $q(x)=q_1(x)+R_1(x)+R_p(x_p)$ and that the condition $\Omega_{1,p}(t)' \subseteq \Omega_{1,p}(x^*)$ in the summation may be replaced with the condition $R_k(t_k)=1 \vee R_k(x_k^*)=1 \,\forall k$: the only datapoints for which the two conditions differ are the points for which $R_1(t_1)=R_1(x^*_1)=0$ or $R_p(t_p)=R_p(x^*_p)=0$. For such datapoints, the effect is null and thus does not affect the formula.\\
\end{proof}

The theorem discussed in the paper is a specific case of the theorem above:

\begin{cor}
\label{cor:ourFormulaCor}
Assume to have a dataset $\mathcal{T}$ of size $n$. Consider a 0-1 coded rule decomposed as the product of single conditions and thus of the form $r(x_1,\ldots,x_p)=\prod_{k=1}^pR_k(x_k)$, with ${R_k:\mathbb{R} \rightarrow \{0,1\}}$. Given $j \in \{1,\ldots,p\}$ and a datapoint $x^*$, the contribution of the $j$-th feature to the prediction $\hat{a} \cdot r(x^*)$ as defined by marginal Shapley values is unbiasedly estimated by:
$$\widehat{\phi}_j(x^*)=\hat{a} \cdot \Bigg(\frac{1}{n(p-q(x^*)+R_j(x^*_j))}\sum_{\substack{t \in \mathcal{T} \text{s.t.}\\ R_k(t_k)=1 \vee R_k(x^*_k)=1 \,\forall k}}\frac{R_j(x_j^*)-R_j(t_j)}{\binom{2p-q(x^*)-q(t)-1+R_j(x^*_j)+R_j(t_j)}{p-q(x^*)+R_j(x^*_j)}}\Bigg),$$
where $q: \mathbb{R}^p \rightarrow \mathbb{N}$ is defined as $q(x)=\sum_{j=1}^pR_l(x_l)$ and $\vee$ is the logical \enquote{or} operator.
\end{cor}
\begin{proof}
We may use the theorem above for $\varphi \equiv \hat{a}$. Then $r(x)=\varphi(x_j)r(x)$ has Shapley values:
\begin{align*}
\widehat{\phi}_j(x^*)&=\frac{1}{n(p-q(x^*)+R_j(x^*_j))}\sum_{\substack{t \in \mathcal{T} \text{s.t.}\\ R_k(t_k)=1 \vee R_k(x^*_k)=1 \,\forall k}}\frac{\varphi(x_j^*)R_l(x_j^*)-\varphi(t_j)R_l(t_j)}{\binom{2p-q(x^*)-q(t)-1+R_j(x^*_j)+R_j(t_j)}{p-q(x^*)+R_j(x^*_j)}}\\
&=\frac{\hat{a}}{n(p-q(x^*)+R_j(x^*_j))}\sum_{\substack{t \in \mathcal{T} \text{s.t.}\\ R_k(t_k)=1 \vee R_k(x^*_k)=1 \,\forall k}}\frac{R_l(x_j^*)-R_l(t_j)}{\binom{2p-q(x^*)-q(t)-1+R_j(x^*_j)+R_j(t_j)}{p-q(x^*)+R_j(x^*_j)}}
\end{align*}
\end{proof}

\begin{thm*}
Assume to have a dataset $\mathcal{T}$ of size $n$. Consider a 0-1 coded rule decomposed as the product of single conditions and thus of the form  $r(x_1,\ldots,x_p)=\prod_{k=1}^pR_k(x_k)$, with ${R_k:\mathbb{R} \rightarrow \{0,1\}}$. Consider any real-valued function $\varphi: \mathbb{R} \rightarrow \mathbb{R}$. Given $\ell,j,j' \in \{1,\ldots,p\}$ with $j \neq j',j \neq \ell$ and a datapoint $x^*$, the interaction of the $j$-th and the $j'$-th features within the prediction $\varphi(x^*_\ell)r(x^*)$ as defined by marginal interaction Shapley values is unbiasedly estimated by:
\begin{align*}
\widehat{\phi}_{j,j'}(x^*) \!\begin{multlined}[t]
    =\frac{1}{n}\sum_{\substack{t \in \mathcal{T} \text{s.t.}\\ R_k(t_k)=1 \vee R_k(x_k^*)=1 \,\forall k}}\frac{\frac{\varphi(x_\ell^*)R_\ell(x_\ell^*)}{p-2-q(x^*)+R_\ell(x_\ell^*)+R_{j'}(x_{j'}^*)+R_j(x_j^*)}+\frac{\varphi(t_\ell)R_\ell(t_\ell)}{p-2-q(t)+R_\ell(t_\ell)+R_{j'}(t_{j'})+R_j(t_j)}}{\binom{2p-q(x^*)+R_\ell(x_\ell^*)+R_{j'}(x_{j'}^*)+R_j(x_j^*)-q(t)+R_\ell(t_\ell)+R_{j'}(t_{j'})+R_j(t_j)-4}{p-2-q(x^*)+R_\ell(x_\ell^*)+R_{j'}(x_{j'}^*)+R_j(x_j^*)}}\\
    \cdot \Big(R_{j'}(x_{j'}^*)R_j(x_j^*)-R_{j'}(x_{j'}^*)R_j(t_j)-R_{j'}(t_{j'})R_j(x_j^*)+R_{j'}(t_{j'})R_j(t_j)\Big).
    \end{multlined}
\end{align*}
if $j' \neq \ell$ and by:
\begin{align*}
\widehat{\phi}_{j,j'}(x^*) \!\begin{multlined}[t]
    =\frac{1}{n(p-1-q(x^*)+R_j(x^*_j)+R_{j'}(x^*_{j'}))}\\
    \cdot \sum_{\substack{t \in \mathcal{T} \text{s.t.}\\ R_k(t_k)=1 \vee R_k(x^*_k)=1 \,\forall k}}\frac{\varphi(x_{j'}^*)\big(R_{j'}(x_{j'}^*)R_j(x_j^*)-R_{j'}(x_{j'}^*)R_j(t_j)\big)-\varphi(t_{j'})\big(R_{j'}(t_{j'})R_j(x_j^*)-R_{j'}(t_{j'})R_j(t_j)\big)}{\binom{2p-q(x^*)-q(t)+R_{j'}(x^*_{j'})+R_j(x^*_j)+R_{j'}(t_{j'})+R_j(t_j)-3}{p-q(x^*)+R_{j'}(x^*_{j'})+R_j(x^*_j)}}
    \end{multlined}
\end{align*}
if $j'=\ell$, where $q: \mathbb{R} \rightarrow \mathbb{N}$ is defined as $q(x)=\sum_{k=1}^pR_k(x_k)$ and $\vee$ is the logical \enquote{or} operator.
\end{thm*}
\begin{proof}
For simplicity of notation, let us assume that $\ell=1$ and $j=p$.\\

Let us first cover the case $j'=\ell=1$ and write $S':=\{2,\ldots,p-1\}\setminus S$, for any subset $S \subseteq \{2,\ldots,p-1\}$. For any subset of active players $S$, we have:
\begin{align*}
    \mathbb{E}[\varphi(x_1)r(x)|\text{do}(x_S=x_S^*)]&=\mathbb{E}[\varphi(x_1)\prod_{k=1}^pR_k(x_k)|\text{do}(x_S=x_S^*)]\\
    &=\prod_{k\in S}R_k(x^*_k)\mathbb{E}[\varphi(x_1)\prod_{k\in \{1,\ldots,p\} \setminus S}R_k(x_k)|\text{do}(x_S=x_S^*)]\\
    &=\prod_{k\in S}R_k(x^*_k)\mathbb{E}[\varphi(x_1)\prod_{k\in \{1,\ldots,p\} \setminus S}R_k(x_k)].
\end{align*}
Analogously, we have:
$$\mathbb{E}[\varphi(x_1)r(x)|\text{do}(x_S=x_S^*,x_1=x_1^*,x_p=x_p^*)]=\varphi(x_1^*)\prod_{k\in S\cup\{1,p\}}R_k(x^*_k)\mathbb{E}[\prod_{k\in S'}R_k(x_k)]$$
$$\mathbb{E}[\varphi(x_1)r(x)|\text{do}(x_S=x_S^*,x_1=x_1^*)]=\varphi(x_1^*)\prod_{k\in S\cup\{1\}}R_k(x^*_k)\mathbb{E}[\prod_{k\in S' \cup \{p\}}R_k(x_k)]$$
$$\mathbb{E}[\varphi(x_1)r(x)|\text{do}(x_S=x_S^*,x_p=x_p^*)]=\prod_{k\in S\cup\{p\}}R_k(x^*_k)\mathbb{E}[\varphi(x_1)\prod_{k\in S' \cup \{1\}}R_k(x_k)]$$
We write these products more compactly by grouping up the indices:
define $R_S(x_S):=\prod_{k\in S}R_k(x_k)$ and, consequently, $R_{S'}(x_{S'}):=\prod_{k\in S'}R_k(x_k)$. Then $\phi_{1,p}(x^*)$ may be re-written as:
\begin{align*}
\phi_{1,p}(x^*)&=\!\begin{multlined}[t]
     \sum_{S \subseteq \{2,\ldots,p-1\}}w(S)\Big(\mathbb{E}[\varphi(x_1)r(x)|\text{do}(x_S=x_S^*,x_1=x_1^*,x_p=x_p^*)]\\
     -\mathbb{E}[\varphi(x_1)r(x)|\text{do}(x_S=x_S^*,x_1=x_1^*)]-
     \mathbb{E}[\varphi(x_1)r(x)|\text{do}(x_S=x_S^*,x_p=x_p^*)]\\
     +\mathbb{E}[\varphi(x_1)r(x)|\text{do}(x_S=x_S^*)]\Big)
     \end{multlined}\\
& \!\begin{multlined}[t]
     = \sum_{S \subseteq \{2,\ldots,p-2\}}w(S)\mathbb{E}[\varphi(x_1)r(x)|\text{do}(x_S=x_S^*,x_1=x_1^*,x_p=x_p^*)]\\
     -\sum_{S \subseteq \{2,\ldots,p-2\}}w(S)\mathbb{E}[\varphi(x_1)r(x)|\text{do}(x_S=x_S^*,x_1=x_1^*)]\\
     -\sum_{S \subseteq \{2,\ldots,p-2\}}w(S)\mathbb{E}[\varphi(x_1)r(x)|\text{do}(x_S=x_S^*,x_p=x_p^*)]\\
     +\sum_{S \subseteq \{2,\ldots,p-2\}}w(S)\mathbb{E}[\varphi(x_1)r(x)|\text{do}(x_S=x_S^*)]
     \end{multlined}\\
& \!\begin{multlined}[t]
     = \sum_{S \subseteq \{2,\ldots,p-2\}}w(S) \varphi(x_1^*)R_1(x_1^*)R_p(x_p^*)R_S(x_S^*)\mathbb{E}[R_{S'}(x_{S'})]\\
     -\sum_{S \subseteq \{2,\ldots,p-2\}}w(S) \varphi(x^*_1)R_1(x^*_1)R_S(x_S^*)\mathbb{E}[R_{S'\cup\{p\}}(x_{S'\cup\{p\}})]\\
     -\sum_{S \subseteq \{2,\ldots,p-2\}}w(S) R_p(x^*_p)R_S(x_S^*)\mathbb{E}[\varphi(x_1)R_{S'\cup\{1\}}(x_{S'\cup\{1\}})]\\
     +\sum_{S \subseteq \{2,\ldots,p-2\}}w(S) R_S(x_S^*)\mathbb{E}[\varphi(x_1)R_{S'\cup\{1,p\}}(x_{S'\cup\{1,p\}})].
     \end{multlined}
\end{align*}
Let us treat the four summations separately; define the following:
$$\mathcal{A}=\sum_{S \subseteq \{2,\ldots,p-2\}}w(S) \varphi(x_1^*)R_1(x_1^*)R_p(x_p^*)R_S(x_S^*)\mathbb{E}[R_{S'}(x_{S'})],$$
$$\mathcal{B}=\sum_{S \subseteq \{2,\ldots,p-2\}}w(S) \varphi(x^*_1)R_1(x^*_1)R_S(x_S^*)\mathbb{E}[R_{S'\cup\{p\}}(x_{S'\cup\{p\}})],$$
$$\mathcal{C}=\sum_{S \subseteq \{2,\ldots,p-2\}}w(S) R_p(x^*_p)R_S(x_S^*)\mathbb{E}[\varphi(x_1)R_{S'\cup\{1\}}(x_{S'\cup\{1\}})],$$
$$\mathcal{D}=\sum_{S \subseteq \{2,\ldots,p-2\}}w(S) R_S(x_S^*)\mathbb{E}[\varphi(x_1)R_{S'\cup\{1,p\}}(x_{S'\cup\{1,p\}})].$$
Note that the expectations that appear in these formulas may be unbiasedly estimated by their sample means over the set $\mathcal{T}$. Let us focus on the first summation $\mathcal{A}$, which is therefore approximated by:
\begin{align*}
\widehat{\mathcal{A}}&=\sum_{S \subseteq \{2,\ldots,p-2\}}w(S)\varphi(x_1^*) R_1(x_1^*)R_p(x_p^*)R_S(x_S^*)\widehat{\mathbb{E}}[R_{S'}(x_{S'})]\\
&=\sum_{S \subseteq \{2,\ldots,p-2\}}w(S) \varphi(x_1^*) R_1(x_1^*)R_p(x_p^*)R_S(x_S^*)\frac{1}{n}\sum_{t \in \mathcal{T}}R_{S'}(t_{S'})\\
&=\varphi(x_1^*) R_1(x_1^*)R_p(x_p^*) \cdot \frac{1}{n}\sum_{t \in \mathcal{T}}\sum_{S \subseteq \{2,\ldots,p-2\}}w(S) R_S(x_S^*)R_{S'}(t_{S'}).
\end{align*}
Now let us focus on $\mathcal{E}:=\sum_{S \subseteq \{2,\ldots,p-2\}}w(S) R_{S}(x_{S}^*)R_{S'}(t_{S'})$, for a fixed datapoint $t$. For any datapoint x, define the following:
$$\Omega_{1,p}(x):=\{j \in \{2,\ldots,p-2\} \big| R_j(x)=1\}, \qquad \qquad q_{1,p}(x)=|\Omega_{1,p}(x)|.$$
By definition of $\Omega_{1,p}$, we have:
\begin{align*}
    R_S(x_S^*)R_{S'}(t_{S'})\neq 0 &\iff \begin{cases}
    R_S(x_S^*)=1\\
    R_{S'}(t_{S'})=1
\end{cases} \\
&\iff
\begin{cases}
    S \subseteq \Omega_{1,p}(x^*)\\
    S' \subseteq \Omega_{1,p}(t)
\end{cases}\\
&\iff
\begin{cases}
    S \subseteq \Omega_{1,p}(x^*)\\
    \Omega_{1,p}(t)' \subseteq S
\end{cases},
\end{align*}
where $\Omega_{1,p}(t)'$ is meant as the complementary set of $\Omega_{1,p}(t)$ with respect to $\{2,\ldots,p-2\}$. This means that $w(S) R_{S}(x_{S}^*)R_{S'}(z_{S'})$ only gives a non-zero effect for the datapoints $t$ such that $\Omega_{1,p}(t)' \subseteq \Omega_{1,p}(x^*)$, in which case the (only) subsets $S$ that contribute are the ones such that $\Omega_{1,p}(t)' \subseteq S \subseteq \Omega_{1,p}(x^*)$. Since all such sets $S$ contain $\Omega_{1,p}(t)'$, they can all be uniquely identified by the indices that they have \textit{besides} those in $\Omega_{1,p}(t)'$. In other words, each $S$ may be (uniquely) re-written as $S=\Omega_{1,p}(t)' \cup Z$, with $Z \subseteq \Omega_{1,p}(x^*) \setminus \Omega_{1,p}(t)'$. For every size $|Z|=m$, there are exactly $\binom{|\Omega_{1,p}(x^*)\setminus\Omega_{1,p}(t)'|}{m}=\binom{q_{1,p}(x^*)+q_{1,p}(t)-p+2}{m}$ possible choices of $Z$, and they all have the same effect $w(S)=\frac{1}{(p-1)\binom{p-2}{|S|}}=\frac{1}{(p-1)\binom{p-2}{p-2-q_{1,p}(t)+m}}$.\\
This means that we can write:
\begin{align*}
\mathcal{E}&=\sum_{S \subseteq \{2,\ldots,p-2\}}w(S) R_S(x_S^*)R_{S'}(t_{S'})\\
&=\sum_{m=0}^{|\Omega_{1,p}(x^*)\setminus\Omega_{1,p}(t)|}\frac{\binom{q_{1,p}(x^*)+q_{1,p}(t)-p+2}{m}}{(p-1)\binom{p-2}{p-2-q_{1,p}(t)+m}}\\
&=\sum_{m=0}^{q_{1,p}(x^*)+q_{1,p}(t)-p+2}\frac{\binom{q_{1,p}(x^*)+q_{1,p}(t)-p+2}{m}}{(p-1)\binom{p-2}{p-2-q_{1,p}(t)+m}}\\
& \!\begin{multlined}[t]
     =\frac{(q_{1,p}(x^*)+q_{1,p}(t)-p+2)!}{(p-1)!}\\
     \cdot \sum_{m=0}^{q_{1,p}(x^*)+q_{1,p}(t)-p+2}\frac{(q_{1,p}(t)-m)!(p-2-q_{1,p}(t)+m)!}{(q_{1,p}(x^*)+q_{1,p}(t)-p+2-m)!m!}
     \end{multlined}\\
& \!\begin{multlined}[t]
     =\frac{(q_{1,p}(x^*)+q_{1,p}(t)-p+2)!}{(p-1)!}\\
     \cdot \sum_{m=0}^{q_{1,p}(x^*)+q_{1,p}(t)-p+2}\Bigg[\binom{q_{1,p}(t)-m}{q_{1,p}(x^*)+q_{1,p}(t)-p+2-m}(p-2-q_{1,p}(x^*))!\\
     \cdot \binom{p-2-q_{1,p}(t)+m}{m}(p-2-q_{1,p}(t))!\Bigg]
     \end{multlined}\\
& \!\begin{multlined}[t]
     =\frac{(q_{1,p}(x^*)+q_{1,p}(t)-p+2)!(p-2-q_{1,p}(x^*))!(p-2-q_{1,p}(t))!}{(p-1)!}\\
     \cdot \sum_{l=0}^{q_{1,p}(x^*)+q_{1,p}(t)-p+2}\binom{q_{1,p}(t)-l}{q_{1,p}(x^*)+q_{1,p}(t)-p+2-l}\binom{p-2-q_{1,p}(t)+l}{l}.
     \end{multlined}
\end{align*}

Using Lemma~\ref{lem:binomialsum} with $a=p-2-q_{1,p}(t),b=q_{1,p}(t),c=q_{1,p}(x^*)+q_{1,p}(t)-p+2$, we conclude:

\begin{align*}
\mathcal{E}&=\sum_{S \subseteq \{2,\ldots,p-2\}}w(S) R_S(x_S^*)R_{S'}(t_{S'})\\
& \!\begin{multlined}[t]
     =\frac{(q_{1,p}(x^*)+q_{1,p}(t)-p+2)!(p-2-q_{1,p}(x^*))!(p-2-q_{1,p}(t))!}{(p-1)!}\\
     \cdot \binom{p-1}{q_{1,p}(x^*)+q_{1,p}(t)-p+2}
     \end{multlined}\\
&=\frac{(p-2-q_{1,p}(x^*))!(p-2-q_{1,p}(t))!}{(2p-q_{1,p}(x^*)-q_{1,p}(t)-3)!}\\
&=\frac{1}{(p-1-q_{1,p}(x^*))}\cdot \frac{1}{\binom{2p-q_{1,p}(x^*)-q_{1,p}(t)-3}{p-1-q_{1,p}(x^*)}}.
\end{align*}
This allows us to conclude that:
$$\widehat{\mathcal{A}}=\frac{1}{n(p-1-q_{1,p}(x^*))}\sum_{\substack{t \in \mathcal{T} \text{ s.t.}\\ \Omega_{1,p}(t)' \subseteq \Omega_{1,p}(x^*)}}\frac{\varphi(x_1^*)R_1(x_1^*)R_p(x^*_p)}{\binom{2p-q_{1,p}(x^*)-q_{1,p}(t)-1}{p-q_{1,p}(x^*)}}.$$
The sum is only over the datapoints with $\Omega_{1,p}(t)' \subseteq \Omega_{1,p}(x^*)$, as the other datapoints have a null effect. In a similar fashion, let us compute $\widehat{\mathcal{B}},\widehat{\mathcal{C}},\widehat{\mathcal{D}}$:
\begin{align*}
    \widehat{\mathcal{B}}&=\sum_{S \subseteq \{2,\ldots,p-2\}}w(S) \varphi(x_1^*)R_1(x_1^*)R_S(x_S^*)\frac{1}{n}\sum_{t \in \mathcal{T}}R_{S'\cup\{p\}}(t_{S'\cup\{p\}})\\
    &=\varphi(x_1^*)R_1(x_1^*)\frac{1}{n}\sum_{t \in \mathcal{T}}R_p(t_p)\sum_{S \subseteq \{2,\ldots,p-2\}}w(S)R_S(x_S^*)R_{S'}(t_{S'})\\
    &=\frac{1}{n(p-1-q_{1,p}(x^*))}\sum_{\substack{t \in \mathcal{T} \text{ s.t.}\\ \Omega_{1,p}(t)' \subseteq \Omega_{1,p}(x^*)}}\frac{\varphi(x_1^*)R_1(x_1^*)R_p(t_p)}{\binom{2p-q_{1,p}(x^*)-q_{1,p}(t)-1}{p-q_{1,p}(x^*)}},
\end{align*}
\begin{align*}
    \widehat{\mathcal{C}}&=\sum_{S \subseteq \{2,\ldots,p-2\}}w(S) R_p(x_p^*)R_S(x_S^*)\frac{1}{n}\sum_{t \in \mathcal{T}}\varphi(t_1)R_{S'\cup\{1\}}(t_{S'\cup\{1\}})\\
    &=R_p(x_p^*)\frac{1}{n}\sum_{t \in \mathcal{T}}\varphi(t_1)R_1(t_1)\sum_{S \subseteq \{2,\ldots,p-2\}}w(S)R_S(x_S^*)R_{S'}(t_{S'})\\
    &=\frac{1}{n(p-1-q_{1,p}(x^*))}\sum_{\substack{t \in \mathcal{T} \text{ s.t.}\\ \Omega_{1,p}(t)' \subseteq \Omega_{1,p}(x^*)}}\frac{\varphi(t_1)R_1(t_1)R_p(x^*_p)}{\binom{2p-q_{1,p}(x^*)-q_{1,p}(t)-1}{p-q_{1,p}(x^*)}},
\end{align*}
\begin{align*}
    \widehat{\mathcal{D}}&=\sum_{S \subseteq \{2,\ldots,p-2\}}w(S) R_S(x_S^*)\frac{1}{n}\sum_{t \in \mathcal{T}}\varphi(t_1)R_{S'\cup\{1,p\}}(t_{S'\cup\{1,p\}})\\
    &=\frac{1}{n}\sum_{t \in \mathcal{T}}\varphi(t_1)R_1(t_1)R_p(t_p)\sum_{S \subseteq \{2,\ldots,p-2\}}w(S)R_S(x_S^*)R_{S'}(t_{S'})\\
    &=\frac{1}{n(p-1-q_{1,p}(x^*))}\sum_{\substack{t \in \mathcal{T} \text{ s.t.}\\ \Omega_{1,p}(t)' \subseteq \Omega_{1,p}(x^*)}}\frac{\varphi(t_1)R_1(t_1)R_p(t_p)}{\binom{2p-q_{1,p}(x^*)-q_{1,p}(t)-1}{p-q_{1,p}(x^*)}},
\end{align*}
Combining the expressions together gives us the formula:
\begin{align*}
    \widehat{\phi}_{1,p}(x^*)&=\widehat{\mathcal{A}}-\widehat{\mathcal{B}}-\widehat{\mathcal{C}}+\widehat{\mathcal{D}}\\
    & \!\begin{multlined}[t]
     =\frac{1}{n(p-1-q_{1,p}(x^*))}\\
     \cdot \sum_{\substack{t \in \mathcal{T} \text{ s.t.}\\ \Omega_{1,p}(t)' \subseteq \Omega_{1,p}(x^*)}}\frac{\varphi(x^*_1)R_1(x^*_1)R_p(x^*_p)-\varphi(x^*_1)R_1(x^*_1)R_p(t_p)-\varphi(t_1)R_1(t_1)R_p(x^*_p)+\varphi(t_1)R_1(t_1)R_p(t_p)}{\binom{2p-q_{1,p}(x^*)-q_{1,p}(t)-1}{p-q_{1,p}(x^*)}}.
     \end{multlined}\\
\end{align*}
To obtain the formula stated in the Theorem, notice that $q(x)=q_{1,p}(x)+R_1(x)+R_p(x)$ and that the condition $\Omega_{1,p}(t)' \subseteq \Omega_{1,p}(x^*)$ in the summation may be replaced with the condition $R_k(t_k)=1 \vee R_k(x_k^*)=1 \,\forall k$: the only datapoints for which the two conditions differ are the points for which $R_1(t_1)=R_1(x^*_1)=0$ or $R_p(t_p)=R_p(x^*_p)=0$. For such datapoints, the effect is null and thus does not affect the formula.\\

Now let us consider the case $j' \neq \ell = 1$. For simplicity, let's assume $j'=p-1$. For any subset $V \subseteq \{2,\ldots,p-2\}$, define $V''$ as $\{2,\ldots,p-2\} \setminus V$. By separating the cases $1 \notin S$ from the cases $1 \in S$ (which therefore is of the form $\{1\} \cup V$ for $V \not\ni 1$), interaction Shapley values may be re-written as:
\begin{align*}
\phi_{p,p-1}(x^*) & \!\begin{multlined}[t]
     =\sum_{S \subseteq \{1,\ldots,p-2\}}w(S)\Big(\mathbb{E}[\varphi(x_1)r(x)|\text{do}(x_S=x_S^*,x_{p-1}=x_{p-1}^*,x_p=x_p^*)]\\
     -\mathbb{E}[\varphi(x_1)r(x)|\text{do}(x_S=x_S^*,x_{p-1}=x_{p-1}^*)] -\mathbb{E}[\varphi(x_1)r(x)|\text{do}(x_S=x_S^*,x_p=x_p^*)]\\
     +\mathbb{E}[\varphi(x_1)r(x)|\text{do}(x_S=x_S^*)]\Big)
     \end{multlined}\\
 & \!\begin{multlined}[t]
     =\sum_{V \subseteq \{2,\ldots,p-2\}}w(V\cup\{1\})\Big(\mathbb{E}[\varphi(x_1)r(x)|\text{do}(x_V=x_V^*,x_1=x_1^*,x_{p-1}=x_{p-1}^*,x_p=x_p^*)]\\
     -\mathbb{E}[\varphi(x_1)r(x)|\text{do}(x_V=x_V^*,x_1=x_1^*,x_{p-1}=x_{p-1}^*)] -\mathbb{E}[\varphi(x_1)r(x)|\text{do}(x_V=x_V^*,x_1=x_1^*,x_p=x_p^*)]\\
     +\mathbb{E}[\varphi(x_1)r(x)|\text{do}(x_V=x_V^*,x_1=x_1^*)]\Big)\\
     +\sum_{V \subseteq \{2,\ldots,p-2\}}w(V)\Big(\mathbb{E}[\varphi(x_1)r(x)|\text{do}(x_V=x_V^*,x_{p-1}=x_{p-1}^*,x_p=x_p^*)]\\
     -\mathbb{E}[\varphi(x_1)r(x)|\text{do}(x_V=x_V^*,x_{p-1}=x_{p-1}^*)] -\mathbb{E}[\varphi(x_1)r(x)|\text{do}(x_V=x_V^*,x_p=x_p^*)]\\
     +\mathbb{E}[\varphi(x_1)r(x)|\text{do}(x_V=x_V^*)]\Big)\\
     \end{multlined}\\
 & \!\begin{multlined}[t]
     =\sum_{V \subseteq \{2,\ldots,p-2\}}w(V\cup\{1\})\varphi(x^*_1)R_1(x^*_1)R_{p-1}(x^*_{p-1})R_p(x^*_p)R_V(x^*_V)\mathbb{E}[R_{V''}(x_{V''})]\\
     -\sum_{V \subseteq \{2,\ldots,p-2\}}w(V\cup\{1\})\varphi(x^*_1)R_1(x^*_1)R_{p-1}(x^*_{p-1})R_V(x^*_V)\mathbb{E}[R_{V'' \cup \{p\}}(x_{V'' \cup \{p\}})]\\
     -\sum_{V \subseteq \{2,\ldots,p-2\}}w(V\cup\{1\})\varphi(x^*_1)R_1(x^*_1)R_p(x^*_p)R_V(x^*_V)\mathbb{E}[R_{V'' \cup \{p-1\}}(x_{V'' \cup \{p-1\}})]\\
     +\sum_{V \subseteq \{2,\ldots,p-2\}}w(V\cup\{1\})\varphi(x^*_1)R_1(x^*_1)R_V(x^*_V)\mathbb{E}[R_{V'' \cup \{p-1,p\}}(x_{V'' \cup \{p-1,p\}})]\\
     +\sum_{V \subseteq \{2,\ldots,p-2\}}w(V)R_{p-1}(x^*_{p-1})R_p(x^*_p)R_V(x^*_V)\mathbb{E}[\varphi(x_1)R_{V'' \cup \{1\}}(x_{V'' \cup \{1\}})]\\
     -\sum_{V \subseteq \{2,\ldots,p-2\}}w(V)R_{p-1}(x^*_{p-1})R_V(x^*_V)\mathbb{E}[\varphi(x_1)R_{V'' \cup \{1,p\}}(x_{V'' \cup \{1,p\}})]\\
     -\sum_{V \subseteq \{2,\ldots,p-2\}}w(V)R_p(x^*_p)R_V(x^*_V)\mathbb{E}[\varphi(x_1)R_{V'' \cup \{1,p-1\}}(x_{V'' \cup \{1,p-1\}})]\\
     +\sum_{V \subseteq \{2,\ldots,p-2\}}w(V)R_V(x^*_V)\mathbb{E}[\varphi(x_1)R_{V'' \cup \{1,p-1,p\}}(x_{V'' \cup \{1,p-1,p\}})]\\
     \end{multlined}\\
\end{align*}
Let us treat the eight summations separately; define the following:
$$\mathcal{F}=\sum_{V \subseteq \{2,\ldots,p-2\}}w(V\cup\{1\})\varphi(x^*_1)R_1(x^*_1)R_{p-1}(x^*_{p-1})R_p(x^*_p)R_V(x^*_V)\mathbb{E}[R_{V''}(x_{V''})],$$
$$\mathcal{G}=\sum_{V \subseteq \{2,\ldots,p-2\}}w(V\cup\{1\})\varphi(x^*_1)R_1(x^*_1)R_{p-1}(x^*_{p-1})R_V(x^*_V)\mathbb{E}[R_{V'' \cup \{p\}}(x_{V'' \cup \{p\}})],$$
$$\mathcal{H}=\sum_{V \subseteq \{2,\ldots,p-2\}}w(V\cup\{1\})\varphi(x^*_1)R_1(x^*_1)R_p(x^*_p)R_V(x^*_V)\mathbb{E}[R_{V'' \cup \{p-1\}}(x_{V'' \cup \{p-1\}})],$$
$$\mathcal{I}=\sum_{V \subseteq \{2,\ldots,p-2\}}w(V\cup\{1\})\varphi(x^*_1)R_1(x^*_1)R_V(x^*_V)\mathbb{E}[R_{V'' \cup \{p-1,p\}}(x_{V'' \cup \{p-1,p\}})],$$
$$\mathcal{J}=\sum_{V \subseteq \{2,\ldots,p-2\}}w(V)R_{p-1}(x^*_{p-1})R_p(x^*_p)R_V(x^*_V)\mathbb{E}[\varphi(x_1)R_{V'' \cup \{1\}}(x_{V'' \cup \{1\}})],$$
$$\mathcal{K}=\sum_{V \subseteq \{2,\ldots,p-2\}}w(V)R_{p-1}(x^*_{p-1})R_V(x^*_V)\mathbb{E}[\varphi(x_1)R_{V'' \cup \{1,p\}}(x_{V'' \cup \{1,p\}})],$$
$$\mathcal{L}=\sum_{V \subseteq \{2,\ldots,p-2\}}w(V)R_p(x^*_p)R_V(x^*_V)\mathbb{E}[\varphi(x_1)R_{V'' \cup \{1,p-1\}}(x_{V'' \cup \{1,p-1\}})],$$
$$\mathcal{M}=\sum_{V \subseteq \{2,\ldots,p-2\}}w(V)R_V(x^*_V)\mathbb{E}[\varphi(x_1)R_{V'' \cup \{1,p-1,p\}}(x_{V'' \cup \{1,p-1,p\}})].$$

With a similar argument as above, they can be approximated without bias by:
\begin{align*}
    \widehat{\mathcal{F}} &= \sum_{V \subseteq \{2,\ldots,p-2\}}w(V\cup\{1\})\varphi(x^*_1)R_1(x^*_1)R_{p-1}(x^*_{p-1})R_p(x^*_p)R_V(x^*_V) \frac{1}{n}\sum_{t \in \mathcal{T}}R_{V''}(t_{V''})\\
    &= \frac{1}{n}\sum_{t \in \mathcal{T}}\varphi(x^*_1)R_1(x^*_1)R_{p-1}(x^*_{p-1})R_p(x^*_p)\sum_{V \subseteq \{2,\ldots,p-2\}}w(V\cup\{1\})R_V(x_V^*)R_{V''}(t_{V''})\\
\end{align*}
\begin{align*}
    \widehat{\mathcal{G}} &= \sum_{V \subseteq \{2,\ldots,p-2\}}w(V\cup\{1\})\varphi(x^*_1)R_1(x^*_1)R_{p-1}(x^*_{p-1})R_V(x^*_V)\frac{1}{n}\sum_{t \in \mathcal{T}}R_{V''\cup\{p\}}(t_{V''\cup\{p\}})\\
    &= \frac{1}{n}\sum_{t \in \mathcal{T}}\varphi(x^*_1)R_1(x^*_1)R_{p-1}(x^*_{p-1})R_p(t_p)\sum_{V \subseteq \{2,\ldots,p-2\}}w(V\cup\{1\})R_V(x_V^*)R_{V''}(t_{V''})\\
\end{align*}
\begin{align*}
    \widehat{\mathcal{H}} &= \sum_{V \subseteq \{2,\ldots,p-2\}}w(V\cup\{1\})\varphi(x^*_1)R_1(x^*_1)R_{p}(x^*_{p})R_V(x^*_V)\frac{1}{n}\sum_{t \in \mathcal{T}}R_{V''\cup\{p-1\}}(t_{V''\cup\{p-1\}})\\
    &= \frac{1}{n}\sum_{t \in \mathcal{T}}\varphi(x^*_1)R_1(x^*_1)R_{p-1}(t_{p-1})R_p(x^*_p)\sum_{V \subseteq \{2,\ldots,p-2\}}w(V\cup\{1\})R_V(x_V^*)R_{V''}(t_{V''})\\
\end{align*}
\begin{align*}
    \widehat{\mathcal{I}} &= \sum_{V \subseteq \{2,\ldots,p-2\}}w(V\cup\{1\})\varphi(x^*_1)R_1(x^*_1)R_V(x^*_V)\frac{1}{n}\sum_{t \in \mathcal{T}}R_{V''\cup\{p-1,p\}}(t_{V''\cup\{p-1,p\}})\\
    &=\frac{1}{n}\sum_{t \in \mathcal{T}}\varphi(x^*_1)R_1(x^*_1)R_{p-1}(t_{p-1})R_p(t_p)\sum_{V \subseteq \{2,\ldots,p-2\}}w(V\cup\{1\})R_V(x_V^*)R_{V''}(t_{V''})\\
\end{align*}
\begin{align*}
    \widehat{\mathcal{J}} &= \sum_{V \subseteq \{2,\ldots,p-2\}}w(V)R_{p-1}(x^*_{p-1})R_p(x^*_p)R_V(x^*_V)\frac{1}{n}\sum_{t \in \mathcal{T}}\varphi(t_1)R_{V''\cup\{1\}}(t_{V''\cup\{1\}})\\
    &= \frac{1}{n}\sum_{t \in \mathcal{T}}\varphi(t_1)R_1(t_1)R_{p-1}(x^*_{p-1})R_p(x^*_p)\sum_{V \subseteq \{2,\ldots,p-2\}}w(V)R_V(x_V^*)R_{V''}(t_{V''})\\
\end{align*}
\begin{align*}
    \widehat{\mathcal{K}} &= \sum_{V \subseteq \{2,\ldots,p-2\}}w(V)R_{p-1}(x^*_{p-1})R_V(x^*_V)\frac{1}{n}\sum_{t \in \mathcal{T}}\varphi(t_1)R_{V'' \cup \{1,p\}}(t_{V'' \cup \{1,p\}})\\
    &= \frac{1}{n}\sum_{t \in \mathcal{T}}\varphi(t_1)R_1(t_1)R_{p-1}(x_{p-1}^*)R_p(t_p)\sum_{V \subseteq \{2,\ldots,p-2\}}w(V)R_V(x_V^*)R_{V''}(t_{V''})\\
\end{align*}
\begin{align*}
    \widehat{\mathcal{L}} &= \sum_{V \subseteq \{2,\ldots,p-2\}}w(V)R_p(x^*_p)R_V(x^*_V)\frac{1}{n}\sum_{t \in \mathcal{T}}R_{V''\cup\{1,p-1\}}(t_{V''\cup\{1,p-1\}})\\
    &= \frac{1}{n}\sum_{t \in \mathcal{T}}\varphi(t_1)R_1(t_1)R_{p-1}(t_{p-1})R_p(x^*_p)\sum_{V \subseteq \{2,\ldots,p-2\}}w(V)R_V(x_V^*)R_{V''}(t_{V''})\\
\end{align*}
\begin{align*}
    \widehat{\mathcal{M}} &= \sum_{V \subseteq \{2,\ldots,p-2\}}w(V)R_V(x^*_V)\frac{1}{n}\sum_{t \in \mathcal{T}}\varphi(t_1)R_{V''\cup\{1,p-1,p\}}(t_{V''\cup\{1,p-1,p\}})\\
    &= \frac{1}{n}\sum_{t \in \mathcal{T}}\varphi(t_1)R_1(t_1)R_{p-1}(t_{p-1})R_p(t_p)\sum_{V \subseteq \{2,\ldots,p-2\}}w(V)R_V(x_V^*)R_{V''}(t_{V''})\\
\end{align*}

Notice that the way that the weights $w(V)$ are defined for interactions Shapley values is so that the summations over $V$ that appear in these expressions are the same as those computed for the proof of Theorem~\ref{thm:ourFormula}, except that here they use a set of $p-1$ features instead of a set of $p$ features. Using that same calculation, we thus deduce:
$$\widehat{\mathcal{F}}=\frac{1}{n(p-2-q_{1,p-1,p}(x^*))}\sum_{\substack{t \in \mathcal{T} \text{s.t.}\\ \Omega_{1,p-1,p}(t)' \subseteq \Omega_{1,p-1,p}(x^*)}}\frac{\varphi(x_1^*)R_1(x_1^*)R_{p-1}(x_{p-1}^*)R_p(x_p^*)}{\binom{2p-q_{1,p-1,p}(x^*)-q_{1,p-1,p}(t)-4}{p-2-q_{1,p-1,p}(x^*)}}$$
$$\widehat{\mathcal{G}}=\frac{1}{n(p-2-q_{1,p-1,p}(x^*))}\sum_{\substack{t \in \mathcal{T} \text{s.t.}\\ \Omega_{1,p-1,p}(t)' \subseteq \Omega_{1,p-1,p}(x^*)}}\frac{\varphi(x_1^*)R_1(x_1^*)R_{p-1}(x_{p-1}^*)R_p(t_p)}{\binom{2p-q_{1,p-1,p}(x^*)-q_{1,p-1,p}(t)-4}{p-2-q_{1,p-1,p}(x^*)}}$$
$$\widehat{\mathcal{H}}=\frac{1}{n(p-2-q_{1,p-1,p}(x^*))}\sum_{\substack{t \in \mathcal{T} \text{s.t.}\\ \Omega_{1,p-1,p}(t)' \subseteq \Omega_{1,p-1,p}(x^*)}}\frac{\varphi(x_1^*)R_1(x_1^*)R_{p-1}(t_{p-1})R_p(x_p^*)}{\binom{2p-q_{1,p-1,p}(x^*)-q_{1,p-1,p}(t)-4}{p-2-q_{1,p-1,p}(x^*)}}$$
$$\widehat{\mathcal{I}}=\frac{1}{n(p-2-q_{1,p-1,p}(x^*))}\sum_{\substack{t \in \mathcal{T} \text{s.t.}\\ \Omega_{1,p-1,p}(t)' \subseteq \Omega_{1,p-1,p}(x^*)}}\frac{\varphi(x_1^*)R_1(x_1^*)R_{p-1}(t_{p-1})R_p(t_p)}{\binom{2p-q_{1,p-1,p}(x^*)-q_{1,p-1,p}(t)-4}{p-2-q_{1,p-1,p}(x^*)}}$$
$$\widehat{\mathcal{J}}=\sum_{\substack{t \in \mathcal{T} \text{s.t.}\\ \Omega_{1,p-1,p}(t)' \subseteq \Omega_{1,p-1,p}(x^*)}}\frac{1}{n(p-2-q_{1,p-1,p}(t))}\cdot\frac{\varphi(t_1)R_1(t_1)R_{p-1}(x_{p-1}^*)R_p(x_p^*)}{\binom{2p-q_{1,p-1,p}(x^*)-q_{1,p-1,p}(t)-4}{p-2-q_{1,p-1,p}(x^*)}}$$
$$\widehat{\mathcal{K}}=\sum_{\substack{t \in \mathcal{T} \text{s.t.}\\ \Omega_{1,p-1,p}(t)' \subseteq \Omega_{1,p-1,p}(x^*)}}\frac{1}{n(p-2-q_{1,p-1,p}(t))}\cdot\frac{\varphi(t_1)R_1(t_1)R_{p-1}(x_{p-1}^*)R_p(t_p)}{\binom{2p-q_{1,p-1,p}(x^*)-q_{1,p-1,p}(t)-4}{p-2-q_{1,p-1,p}(x^*)}}$$
$$\widehat{\mathcal{L}}=\sum_{\substack{t \in \mathcal{T} \text{s.t.}\\ \Omega_{1,p-1,p}(t)' \subseteq \Omega_{1,p-1,p}(x^*)}}\frac{1}{n(p-2-q_{1,p-1,p}(t))}\cdot\frac{\varphi(t_1)R_1(t_1)R_{p-1}(t_{p-1})R_p(x_p^*)}{\binom{2p-q_{1,p-1,p}(x^*)-q_{1,p-1,p}(t)-4}{p-2-q_{1,p-1,p}(x^*)}}$$
$$\widehat{\mathcal{F}}=\sum_{\substack{t \in \mathcal{T} \text{s.t.}\\ \Omega_{1,p-1,p}(t)' \subseteq \Omega_{1,p-1,p}(x^*)}}\frac{1}{n(p-2-q_{1,p-1,p}(t))}\cdot \frac{\varphi(t_1)R_1(t_1)R_{p-1}(t_{p-1})R_p(t_p)}{\binom{2p-q_{1,p-1,p}(x^*)-q_{1,p-1,p}(t)-4}{p-2-q_{1,p-1,p}(x^*)}}$$
Putting the eight pieces together, we obtain:
\begin{align*}
    \widehat{\phi_p}(x^*)&=\widehat{\mathcal{F}}-\widehat{\mathcal{G}}-\widehat{\mathcal{H}}+\widehat{\mathcal{I}}+\widehat{\mathcal{J}}-\widehat{\mathcal{K}}-\widehat{\mathcal{L}}+\widehat{\mathcal{M}}\\
 & \!\begin{multlined}[t]
    =\frac{1}{n}\sum_{\substack{t \in \mathcal{T} \text{s.t.}\\ \Omega_{1,p-1,p}(t)' \subseteq \Omega_{1,p-1,p}(x^*)}}\frac{\frac{\varphi(x_1^*)R_1(x_1^*)}{p-2-q_{1,p-1,p}(x^*)}+\frac{\varphi(t_1)R_1(t_1)}{p-2-q_{1,p-1,p}(t)}}{\binom{2p-q_{1,p-1,p}(x^*)-q_{1,p-1,p}(t)-4}{p-2-q_{1,p-1,p}(x^*)}}\\
    \cdot \Big(R_{p-1}(x_{p-1}^*)R_p(x_p^*)-R_{p-1}(x_{p-1}^*)R_p(t_p)-R_{p-1}(t_{p-1})R_p(x_p^*)+R_{p-1}(t_{p-1})R_p(t_p)\Big).
    \end{multlined}
\end{align*}

To obtain the formula stated in the Theorem, notice that $q(x)=q_{1,p-1,p}(x)+R_1(x)+R_{p-1}(x)+R_p(x)$ and that the condition $\Omega_{1,p-1,p}(t)' \subseteq \Omega_{1,p-1,p}(x^*)$ in the summation may be replaced with the condition $R_k(t_k)=1 \vee R_k(x_k^*)=1 \,\forall k$: the only datapoints for which the two conditions differ are the points for which $R_k(t_k)=R_k(x^*_k)=0$ for either $k=1$, $k=p-1$ or $k=p$. For such datapoints, the effect is null and thus does not affect the formula.\\
\end{proof}

The theorem discussed in the paper is a specific case of the theorem above:

\begin{cor}
\label{cor:ourFormulaIntCor}
Assume to have a dataset $\mathcal{T}$ of size $n$. Consider a 0-1 coded rule decomposed as the product of single conditions and thus of the form $r(x_1,\ldots,x_p)=\prod_{k=1}^pR_k(x_k)$, with ${R_k:\mathbb{R} \rightarrow \{0,1\}}$. Given two different indices $j,j' \in \{1,\ldots,p\}$ and a datapoint $x^*$, the interaction of the $j$-th and the $j'$-th features within the prediction $\hat{a}\cdot r(x^*)$ as defined by marginal interaction Shapley values is unbiasedly estimated by:
\begin{align*}
\widehat{\phi}_{j,j'}(x^*) \!\begin{multlined}[t]
    =\frac{\hat{a}}{n(p-1-q(x^*)+R_j(x^*_j)+R_{j'}(x^*_{j'}))}\\
    \cdot \sum_{\substack{t \in \mathcal{T} \text{s.t.}\\ R_k(t_k)=1 \vee R_k(x^*_k)=1 \,\forall k}}\frac{R_{j'}(x_{j'}^*)R_j(x_j^*)-R_{j'}(x_{j'}^*)R_j(t_j)-R_{j'}(t_{j'})R_j(x_j^*)+R_{j'}(t_{j'})R_j(t_j)}{\binom{2p-q(x^*)-q(t)+R_{j'}(x^*_{j'})+R_j(x^*_j)+R_{j'}(t_{j'})+R_j(t_j)-3}{p-q(x^*)+R_{j'}(x^*_{j'})+R_j(x^*_j)}},
    \end{multlined}
\end{align*}
where $q: \mathbb{R}^p \rightarrow \mathbb{N}$ is defined as $q(x)=\sum_{j=1}^pR_l(x_l)$ and $\vee$ is the logical \enquote{or} operator.
\end{cor}
\begin{proof}
We may use the theorem above for $\varphi \equiv \hat{a}$. Then $r(x)=\varphi(x_j)r(x)$ has Shapley values estimated by:

\begin{align*}
\widehat{\phi}_{j,j'}(x^*)
    &=\frac{1}{n(p-1-q(x^*)+R_j(x^*_j)+R_{j'}(x^*_{j'}))}\\
    & \qquad \cdot \sum_{\substack{t \in \mathcal{T} \text{s.t.}\\ R_k(t_k)=1 \vee R_k(x^*_k)=1 \,\forall k}}\frac{\varphi(x_{j'}^*)R_{j'}(x_{j'}^*)\big(R_j(x_j^*)-R_j(t_j)\big)-\varphi(t_{j'})R_{j'}(t_{j'})\big(R_j(x_j^*)-R_j(t_j)\big)}{\binom{2p-q(x^*)-q(t)+R_{j'}(x^*_{j'})+R_j(x^*_j)+R_{j'}(t_{j'})+R_j(t_j)-3}{p-q(x^*)+R_{j'}(x^*_{j'})+R_j(x^*_j)}}\\
    &=\frac{\hat{a}}{n(p-1-q(x^*)+R_j(x^*_j)+R_{j'}(x^*_{j'}))}\\
    & \qquad \cdot \sum_{\substack{t \in \mathcal{T} \text{s.t.}\\ R_k(t_k)=1 \vee R_k(x^*_k)=1 \,\forall k}}\frac{R_{j'}(x_{j'}^*)R_j(x_j^*)-R_{j'}(x_{j'}^*)R_j(t_j)-R_{j'}(t_{j'})R_j(x_j^*)+R_{j'}(t_{j'})R_j(t_j)}{\binom{2p-q(x^*)-q(t)+R_{j'}(x^*_{j'})+R_j(x^*_j)+R_{j'}(t_{j'})+R_j(t_j)-3}{p-q(x^*)+R_{j'}(x^*_{j'})+R_j(x^*_j)}}.
\end{align*}
\end{proof}

\end{document}